\documentclass[lettersize,onecolumn]{IEEEtran}
\usepackage{amsmath,amsfonts}
\usepackage{algorithmic}
\usepackage{algorithm}
\usepackage{array}
\usepackage[caption=false,font=normalsize,labelfont=sf,textfont=sf]{subfig}
\usepackage{textcomp}
\usepackage{stfloats}
\usepackage{url}
\usepackage{verbatim}
\usepackage{graphicx}
\usepackage{cite}
\usepackage[colorlinks=true, linkcolor=blue, citecolor=blue]{hyperref}

\usepackage{amsmath,amsfonts,bbm}
\usepackage{enumitem}

\def\eqref#1{equation~\ref{#1}}

\def\1{\bm{1}}

\DeclareMathAlphabet{\mathsfit}{\encodingdefault}{\sfdefault}{m}{sl}
\SetMathAlphabet{\mathsfit}{bold}{\encodingdefault}{\sfdefault}{bx}{n}

\newcommand{\cP}{\mathcal{P}}
\newcommand{\cQ}{\mathcal{Q}}

\newcommand{\bI}{\mathbb{I}}

\newcommand{\bR}{\mathbb{R}}

\usepackage{amssymb}
\usepackage{mathtools}
\usepackage{amsthm}
\usepackage{multirow}

\usepackage{titletoc}
\usepackage{booktabs}
\usepackage{setspace}
\usepackage{hyperref}
\usepackage{thmtools}
\usepackage{thm-restate}
\usepackage{longtable}
\usepackage{wrapfig}

\newtheorem{corollary}{Corollary}
\newtheorem{lemma}{Lemma}
\newtheorem{definition}{Definition}
\newtheorem{assumption}{Assumption}
\newtheorem{proposition}{Proposition}

\usepackage[most]{tcolorbox}

\begin{document}

\title{Disentangling Feature Structure: \\A Mathematically Provable Two-Stage Training Dynamics in Transformers}

\author{Zixuan Gong, Shijia Li, Yong Liu, Jiaye Teng
\thanks{Core contributor: Zixuan Gong. Corresponding author: Yong Liu. Authors are listed in alphabetical order.}%
\thanks{Zixuan Gong and Yong Liu are with Gaoling School of Artificial Intelligence, Renmin University of China, Beijing 100872, China (e-mail: zxgong@ruc.edu.cn; liuyonggsai@ruc.edu.cn). Shijia Li is with School of Mechanical-Electronic and Vehicle Engineering, Beijing University of Civil Engineering and Architecture, Beijing, 102616, China (e-mail: shijiali@stu.bucea.edu.cn). Jiaye Teng is with School of Statistics and Management, Shanghai University of Finance and Economics, Shanghai, 200433, China (e-mail: tengjiaye@sufe.edu.cn).}%
}

\markboth{Journal of \LaTeX\ Class Files,~Vol.~1, No.~2, December~2023}%
{Shell \MakeLowercase{\textit{et al.}}: A Sample Article Using IEEEtran.cls for IEEE Journals}

\IEEEpubid{0000--0000~\copyright~2023 IEEE}

\maketitle

\begin{abstract}
Transformers may exhibit two-stage training dynamics during the real-world training process.
For instance, when training \mbox{GPT-2} on the Counterfact dataset, the answers progress from syntactically incorrect to syntactically correct to semantically correct.
However, existing theoretical analyses hardly account for this feature-level two-stage phenomenon, which could be conceptually attributed to disentangled two-type features like syntax and semantics.
In this paper, we theoretically demonstrate how the two-stage training dynamics potentially occur in transformers.
Specifically, we analyze the feature learning dynamics induced by the aforementioned disentangled two-type feature structure, grounding our analysis in a simplified yet illustrative setting that comprises normalized ReLU self-attention and structured data.
Such disentanglement of feature structure is general in practice, \emph{e.g.}, natural languages contain syntax and semantics, and proteins contain primary and secondary structures. 
To our best knowledge, this is the first rigorous result regarding a \textbf{\emph{feature-level}} two-stage optimization process in transformers within this theoretical framework.
A corollary further indicates that such a two-stage process is closely related to the spectral properties of attention weights.
\end{abstract}

\begin{IEEEkeywords}
Feature-level two-stage learning, Optimization dynamics, Finite-time convergence, Feature learning, Feature disentanglement.
\end{IEEEkeywords}

\section{Introduction}\label{sec:introduction}
\IEEEPARstart{T}ransformers~\cite{vaswani2017attention} have emerged as foundational architectures with broad applications across multiple research domains, such as natural language processing~\cite{kenton2019bert, radford2019language}, computer vision~\cite{he2022masked, liu2021swin}, \emph{etc}. 
Recently, large language models (LLM) based on decoder-only transformer architectures further demonstrate impressive capabilities, excelling in various downstream tasks~\cite{brown2020language, openai2023gpt}.
However, it remains an essential issue to delve into why LLMs exhibit such remarkable performance.
Fortunately, exploring the optimization dynamics in transformers presents a promising approach for investigating the possible factors that contribute to this behavior.

Empirically, it is widely known that transformers exhibit staged learning behaviors. For instance, when fine-tuning GPT-2 on the Counterfact dataset in Figure \ref{fig:exp-two-stage} (details in Appendix \ref{sec:app-exp-counterfact}), we observe the following phenomenon: at initial (epoch $1$), most predictions are both syntactically and semantically incorrect. At midpoint (epoch $5$), we observe a significant decrease in training loss; all predictions meet syntactic requirements, but most remain semantically incorrect and inconsistent with the true answers. At convergence (epoch $100$), all predictions are syntactically correct, with most being semantically correct and achieving a small training loss. 
Overall, the model's answers progress \emph{\textbf{from syntactically incorrect to syntactically correct to semantically correct}}, exhibiting two-stage training dynamics for syntactic and semantic information.

Motivated by this phenomenon, for various tasks like language tasks, protein structure prediction tasks, or classic supervised learning tasks, we can disentangle feature structure into two types: \emph{\textbf{elementary knowledge}} (\emph{like syntactic information}), and \emph{\textbf{specialized knowledge}} (\emph{like semantic information}).
Such disentanglement is empirically general in both NLP and biological research ~\cite{alquraishi2019alphafold, bao2019generating,chen2019multi,huang2021disentangling, jumper2021highly}. Additionally, the corresponding two-stage learning process has been revealed in vision field~\cite{caron2021emerging}.
Based on the above discussion, it is natural to infer that knowledge may be acquired following an \emph{\textbf{elementary-then-specialized}} principle. 
However, this leaves the following critical theoretical question:

\begin{tcolorbox}[
    colback = white,
    colframe = black,
    boxrule = 0.8pt,
    arc = 0pt,
    outer arc = 0pt,
    enhanced,
    overlay={
        \node[
            fill=white,
            text=black,
            font=\bfseries,
            inner xsep=10pt,
            anchor=center
        ] at (frame.north) {\textbf{The key question:}};
    },
    bottom = 2pt
] 
\begin{center}
    How does the disentangled two-type feature structure theoretically induce the \\feature-level two-stage training dynamics in transformers?
\end{center}
\end{tcolorbox}

\begin{figure}%
    \centering
    \includegraphics[width=0.9\linewidth]{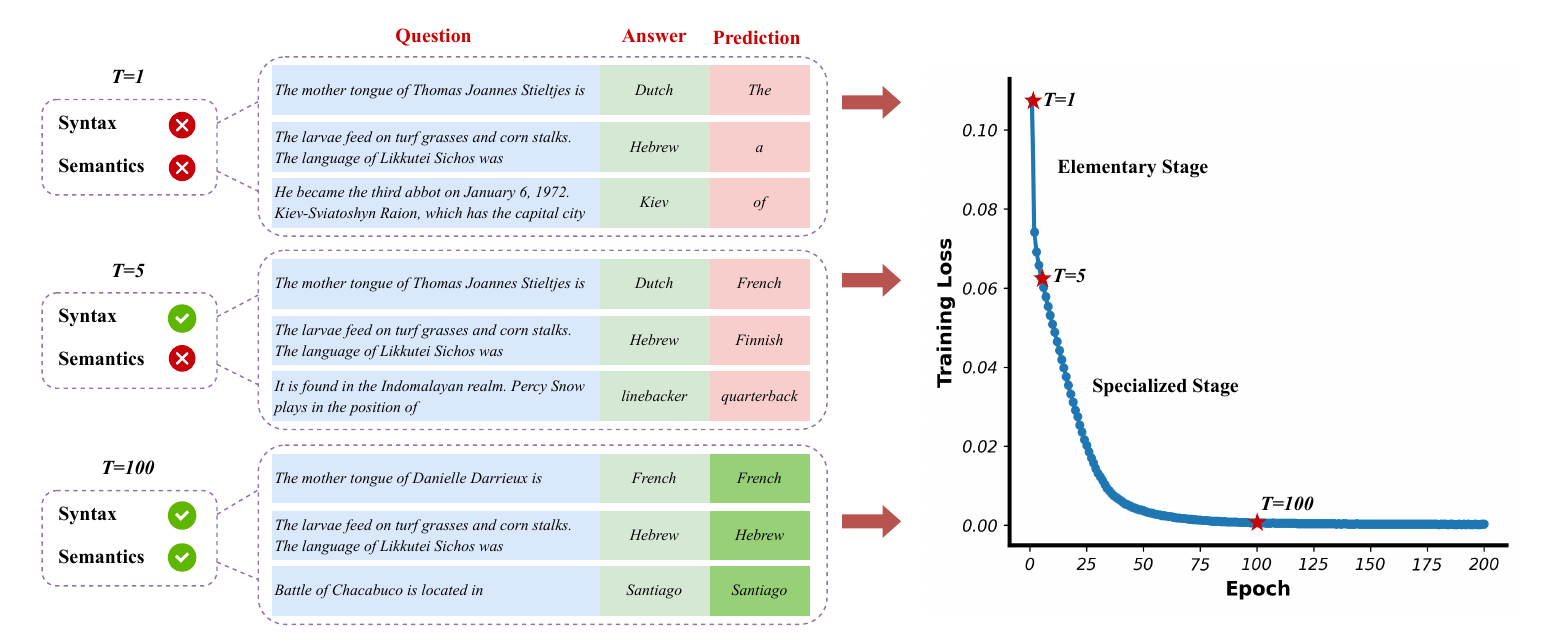}
    \caption{Two-stage learning of syntactic and semantic information on Counterfact dataset.}
    \label{fig:exp-two-stage}
\end{figure}

Many scholars have theoretically delved into the optimization dynamics in supervised learning or language tasks by studying gradient flow ~\cite{chen2024training, cheng2023transformers, huang2023context, zhang2023trained}, or convergence ~\cite{deora2023optimization, li2023transformers}.
Moreover, additional works also involve multi-stage discussions, \emph{e.g.}, theoretical insights on induction head mechanisms~\cite{edelman2024evolution}, token-level stages~\cite{nichani2024understanding}, model layer-level stages~\cite{chen2024unveiling}, attention weight-level stages~\cite{jiang2024unveil, tian2023scan, tian2023joma}, loss-level stages~\cite{ren2024learning}, and frequency-level stages~\cite{pan2025understanding} (More explanations in Section \ref{sec:related-work}). 
However, there is less consideration of feature structure, which might be crucial to inducing a realistic optimization process, and is empirically general in wide NLP and biological research as introduced before. 
Thus, distinct from the above, 
we rigorously present feature-level stages within a simplified theoretical framework that makes the phenomenon tractable.

To demystify the feature-level stages of transformers under disentangled feature structure, we adopt in-context learning (ICL) regimes, constructing training prompts with independently and identically distributed (\emph{i.i.d.}) in-context samples to study supervised classification tasks.
As is well-known, ICL ~\cite{brown2020language} has emerged as a remarkable ability in LLMs, where the model solves new tasks based on prompts without further parameter fine-tuning~\cite{black2022gpt,rae2021scaling}. This ability has served as the foundation for developing more advanced prompting techniques to tackle complex problems~\cite{huang2022towards}. 
Recent theoretical studies mainly focus on the setting where the training and test prompts are embedded as sequences of labeled training samples and an unlabeled query, where transformers can mimic the behavior of \emph{supervised learning} algorithms~\cite{akyurek2022learning, chen2024training, cheng2023transformers, huang2023context, zhang2023trained}.
This prompt-embedding method, the so-called ICL regime, enables theoretical analysis of attention mechanisms in supervised learning tasks.

In this paper, we derive a feature-level two-stage optimization process within a specific theoretical framework, where the model first masters elementary knowledge and then unlocks specialized knowledge.
To formalize this, we disentangle the feature structure into two key types: \emph{elementary knowledge}, modeled as a linear separable component~$\mathcal{P}$, and \emph{specialized knowledge}, modeled as a nonlinear separable component~$\mathcal{Q}$. 
We employ this theoretical abstraction to analyze the dynamics within a simplified one-layer transformer that uses normalized ReLU self-attention mechanism (Section~\ref{sec:problem-setup}), potentially offering a foundation for future explorations of transformer-based learning paradigms. 

Based on this feature disentanglement framework, to our best knowledge, this is the first work to provide rigorous theoretical results for the feature-level two-stage learning process. 
We present the optimization trajectory and finite-time convergence analysis with feature learning and signal-noise decomposition techniques, offering deeper insights into the two-stage learning phenomenon (Theorem~\ref{the:stage1-Q}$\sim$\ref{the:stage2-P} in Section \ref{sec:main-theorems}).
Specifically, our theorems reveal the precise dynamics of this two-stage phenomenon. In the elementary stage, the model fails to learn the nonlinear separable component $\mathcal{Q}$, as the norm of its signal weight $\overline{V}$ stays close to its initial state and the surrogate loss for this component remains high (Theorem~\ref{the:stage1-Q}). Meanwhile, the model learns the linear separable component $\mathcal{P}$, where the norm of its signal weight $\overline{W}$ grows significantly and its loss drops to a small value (Theorem~\ref{the:stage1-P}). 
In the specialized stage, the model begins to learn the specialized knowledge $\mathcal{Q}$. The norm of $\overline{V}$ grows substantially, and its loss converges to a minimal value (Theorem~\ref{the:stage2-Q}). Crucially, the model preserves the previously acquired elementary knowledge during this stage, with the weight $\overline{W}$ reinforcing the target signal direction and its associated loss remaining small (Theorem~\ref{the:stage2-P}).

Finally, we discuss the spectral characteristics of attention weights in Section \ref{sec:coro}, highlighting the close relationship with the two-stage process. Specifically, in Corollary \ref{sec:coro}, we formalize this connection by examining the trace (the sum of eigenvalues) of the weights for elementary and specialized knowledge. This theoretical finding demonstrates that smaller eigenvalues preserve elementary knowledge, while larger ones allow the model to progressively acquire specialized knowledge.

\textit{Organization of Our Paper.}\quad
The remainder of this paper is organized as follows. 
We begin with an overview of the literature related to our work in Section \ref{sec:related-work}. 
In Section \ref{sec:problem-setup}, we then introduce our problem setup, presenting the details of the disentangled feature structure, one-layer transformer architecture, and the training procedure. 
In Section \ref{sec:main-theorems}, we present the main theoretical results of optimization trajectory and finite-time convergence analysis for the two-stage process. We extend this analysis in Section \ref{sec:coro} to discuss the spectral characteristics of attention weights. 
Finally, we conclude the paper in Section \ref{sec:conclusions}. 
Complete proofs of all theoretical results alongside empirical observations are provided in the Appendix.

\section{Related work}\label{sec:related-work}
This section reviews prior literature from four key perspectives. We first survey the optimization analysis of transformers, both under ICL regimes and more broadly. We then discuss the generalization analysis of transformers under ICL. Finally, we introduce feature learning theory, distinguishing it from the lazy training regime.

\textit{Optimization Analysis under ICL Regimes.}\quad 
The optimization analysis under ICL regimes can be roughly split into two branches. 
The first branch investigates whether gradient-based optimization can converge to a global minimum of the ICL objective function~\cite{cheng2023transformers, shen2024training, zhang2023trained, zheng2024mesa}.
These studies focus on optimizing transformers using training prompts structured with input-label pairs, showing that the global minimum of ICL loss is reachable through gradient flow across various models and tasks (\emph{e.g.}, models with linear or softmax modules, and tasks like linear regression or nonlinear function learning). 
For instance, Cheng et al.~\cite{cheng2023transformers} demonstrate that transformers can implicitly implement functional gradient descent to learn non-linear functions. 
However, this line of research focuses less on model weight optimization throughout training and hardly considers finite-time convergence or distinct stages of various information. 
Complementing this, the second branch further analyzes the optimization properties during training~\cite{chen2024training, huang2023context, kim2024transformers}.
Of particular relevance is Huang et al.~\cite{huang2023context}, which theoretically derives a stage-wise learning phenomenon in attention maps for linear regression tasks and reveals that stages emerge from the model's progressive learning of imbalanced features.
Our work differs from Huang et al.~\cite{huang2023context} in two fundamental aspects: (a) we identify a stage-wise phenomenon originating not from feature imbalance but from a disentangled feature structure, where knowledge is separated into elementary and specialized components; and (b) we focus on nonlinear classification tasks. 
In summary, the finite-time training dynamics of transformers remains relatively unexplored, especially when attempting to illustrate the optimization process induced by the disentangled two-type feature. 
Our work proposes a framework to bridge this gap, offering a rigorous feature-level explanation within a tractable setting.

\textit{Optimization Analysis of Transformers without ICL Regimes.}\quad 
A line of work analyzes the training dynamics of transformers without ICL regimes~\cite{deora2023optimization, edelman2024evolution, li2023theoretical}, identifying different levels of stage-wise transitions~\cite{chen2024unveiling, jiang2024unveil, nichani2024understanding, pan2025understanding, ren2024learning, tian2023scan, tian2023joma}. 
For instance, Nichani et al~\cite{nichani2024understanding} categorize all tokens into relation-only and subject-only types, revealing a learning order where relational aspects are mastered before subject-specific information, the so-called \textbf{\emph{token-level stages}}.
Chen et al.~\cite{chen2024unveiling} demonstrate a three-stage process where layers progressively specialize with lower layers first capturing local patterns before upper layers integrating global information, the so-called \textbf{\emph{model layer-level stages}}. Tian et al.~\cite{tian2023scan, tian2023joma} focus on how self-attention aggregates token information via attention maps, and Jiang et al.~\cite{jiang2024unveil} demonstrate a three-stage evolution of attention patterns from simple to complex, the so-called \textbf{\emph{attention weight-level stages}}.
Ren et al.~\cite{ren2024learning} partition the training process based on changes in the loss curve that correspond to distinct behaviors, the so-called \textbf{\emph{loss-level stages}}. 
Pan et al.~\cite{pan2025understanding} argue from a data compression perspective that models learn high-frequency patterns before rare ones, the so-called \textbf{\emph{frequency-level stages}}.
Greatly different from above, we disentangle individual token features into elementary knowledge and specialized knowledge, inducing two-stage learning of the two-type feature, the so-called \textbf{\emph{feature-level stages}}.

\textit{Feature Learning.}\quad 
A significant line of work in understanding neural network convergence is built upon the Neural Tangent Kernel (NTK) technique~\cite{allen2019convergence,chen2019much,du2019gradient, jacot2018neural, li2018learning}.
The NTK framework posits that for highly over-parameterized networks, the training dynamics can be approximated by a deterministic kernel method.
This approximation holds within the \textbf{\textit{lazy training regime}} where network parameters oscillate within a small neighborhood of their random initialization.
While powerful, this perspective does not fully capture the behavior of many practical networks where parameters move significantly, indicating that the network is actively learning new data representations rather than operating with a fixed kernel.
To address this divergence between theory and practice, feature learning theory (often referred to as the \textit{\textbf{rich training regime}}) has emerged as a prominent alternative for analyzing deep learning optimization~\cite{allen2020towards,allen2022feature,li2023transformers,li2019towards,wen2021toward}. 
This approach moves beyond the lazy training, examining how neural networks dynamically learn. To make theoretically tractable, it often relies on specific, structured data generation models, such as Gaussian mixtures or signal-noise models. 
A fundamental work in this direction is Li et al.~\cite{li2019towards}, which provides insights into the optimization of two-layer networks. 
Drawing inspiration from their approach, we extend this to the more complex Transformer-based models. Specifically, we design structured ICL prompts and tokens, providing a fine-grained characterization of optimization dynamics (Detailed technical contributions in Appendix~\ref{sec:thecontributions}).
Ultimately, feature learning technique allows for a precise characterization of how network dynamics facilitate the learning of intrinsic data structures, offering a deeper understanding of the optimization mechanisms.

\section{Problem Setup}\label{sec:problem-setup}
This section presents details of data, model and training procedure. Concretely, Section~\ref{sec:data} designs the individual token feature structure and constructs training prompts following ICL regimes. 
Section~\ref{sec:model} introduces a one-layer attention-based model and two virtual networks. Finally, Section~\ref{sec:training-procedure} describes the corresponding loss function and optimization algorithm used for classification tasks.

\textit{Notations.}\quad Let $\|A\|_F$ be the Frobenius norm for matrix $A$ and $\|x\|_2$ be the 2-norm for vector $x$. 
For vector $x$, $\text{ReLU}(x)= \max\{x,0\}$ denotes the standard ReLU activation function, and $\mathbbm{1}(x)$ denotes a binary vector that takes entries $1$ when $x_i \geq 0$.
The indicator function $\bI(\cdot) \in \{-1,1\}$ is defined such that it takes value $1$ if the condition is satisfied, and $-1$ otherwise.
For order analysis, $\text{Poly}(\cdot)$ represents polynomial order, $f(n) = \mathcal{O}(g(n))$ indicates that $f(n)$ is asymptotically bounded above by $g(n)$, and $f(n) = \Theta(g(n))$ means that $f(n)$ and $g(n)$ are of the same asymptotic order.
Additionally, throughout the paper, let $U \in \mathbb{R}^{2d \times 2d}$ denote a weight matrix, and $W \in \mathbb{R}^{d \times d}, V \in \mathbb{R}^{d \times d}$ denote the principal sub matrices of $U$ defined later.

\subsection{Disentangled Feature Structure}\label{sec:data}
In this section, we detail the construction of the disentangled token feature structure, building upon the established In-Context Learning (ICL) regimes~\cite{garg2022can}. The conceptual motivation for this design is deferred to Appendix~\ref{app:discussion}, where we discuss the intuition behind the feature disentanglement framework and provide justifications for our theoretical abstraction.

\begin{figure}%
    \centering
    \includegraphics[width=0.9\linewidth]{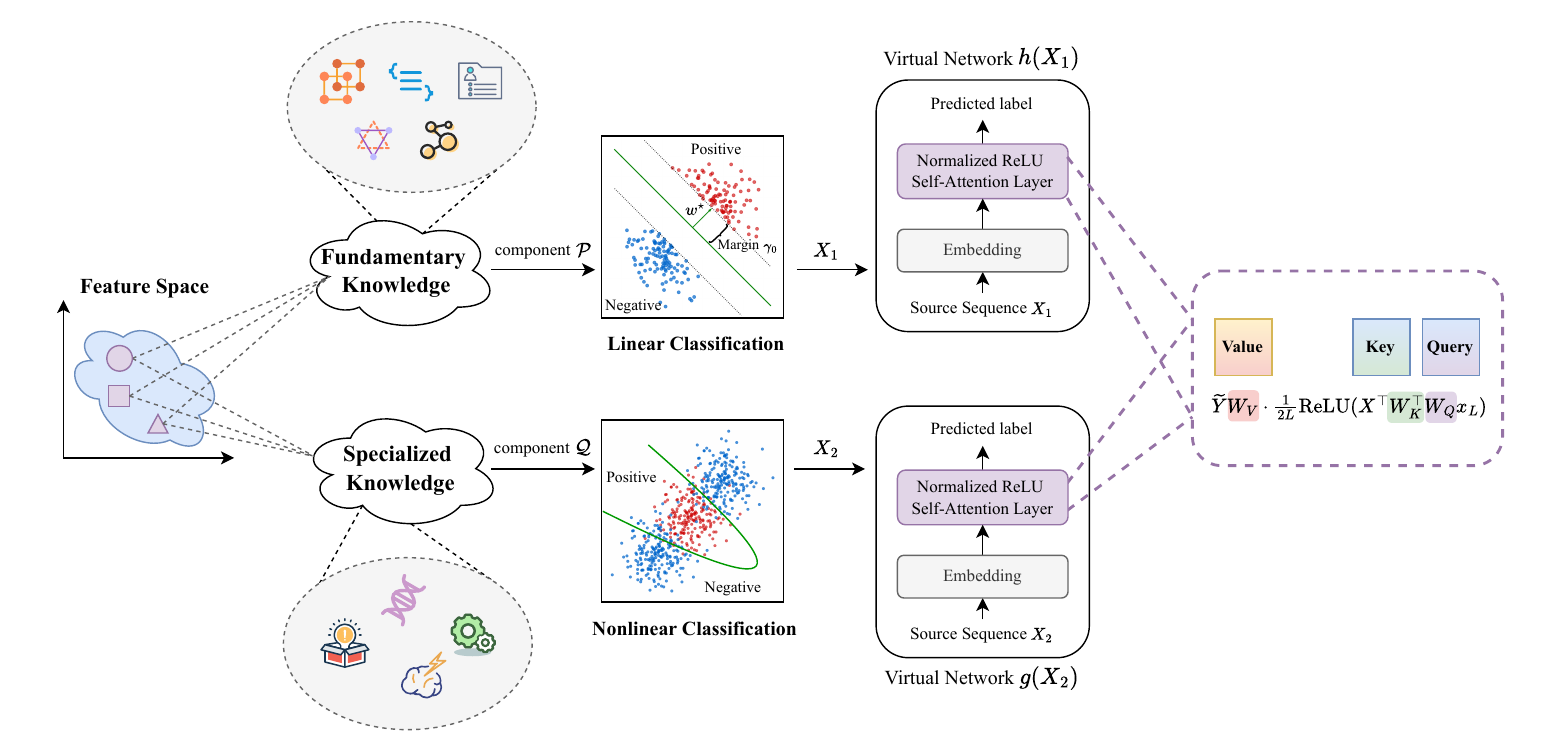}
    \caption{Overview of disentangled feature structure.}
    \label{fig:overview}
\end{figure}

\textit{Training Prompt Structure:}\  Following the regimes in Grag et al.~\cite{garg2022can}, a collection of samples and their corresponding labels are organized in a sequence, commonly referred to as a prompt. ICL is trained on $N$ random training prompts, denoted by $\{ P^n \}_{n \in [N]}$.
The $n$-th training prompt is constructed as $P^n=\left(x_1^n, y_1^n, \cdots, x_{L-1}^n, y_{L-1}^n, x_L^n\right)$ with prompt length $L$, where $x_i^n, i \in [L-1]$ denotes the input samples, $y_i^n, i \in [L-1]$ denotes the corresponding labels, and $x_L^n$ denotes the query. 
Assume that $x_i^n, i\in[L-1]$ are \emph{i.i.d.} drawn, and consider a binary classification setting with $y_i^n = y(x_i^n) \in \{-1, 1\}$. 
The goal of the ICL learner is to train a model $f (\cdot)$, such that the output approximates the label of the query $x_L^n$, namely, $f(P^n) \approx y_L^n = y(x_L^n)$.

\textit{Individual Token Feature Structure:}\ 
In Figure~\ref{fig:overview}, each individual token $x_i^n$ in the prompt $P^n$ is disentangled into two components: $\mathcal{P}$ component represents elementary knowledge (\emph{e.g.}, syntactic information in natural languages, primary structure in protein), and $\mathcal{Q}$ component represents specialized knowledge (\emph{e.g.}, semantic information in natural languages, secondary structure in protein). Specifically, consider a disentangled feature structure  $x_i^n=[x_{i,1}^n, x_{i,2}^n]^\top \in \bR^{2d}$, where $x_{i,1}^n \in \mathbb{R}^d$ denotes the elementary knowledge drawn from distribution $\cP$ and $x_{i,2}^n \in \mathbb{R}^d$ denotes the specialized knowledge drawn from distribution $\cQ$.
We construct the distributions $\cP$ and $\cQ$ as follows, drawing inspirations from Li et al.~\cite{li2019towards}:
\begin{itemize}[itemsep=0em, parsep=0em, topsep=0em, partopsep=0em, leftmargin=2em]
    \item \emph{For distribution $\mathcal{P}$}, given a fixed vector $w^\star \in \mathbb{R}^d$ and a random vector $e_i \sim\mathcal{N}\left(0,I_{d\times d}/d\right)$, the data $(x^n_{i,1}, y^n_{i,1})$ is constructed by
    \begin{align*}
        y^n_{i,1} = \bI(\langle w^{\star},e_i\rangle \geq 0) \in \{-1, 1\};\ \ 
        x^n_{i,1} = y^n_{i,1} \gamma_{0}w^{\star} + e_i.
    \end{align*}
    Such construction guarantees its linear separability by the classifier $w^\star$ with a margin of $\gamma_0 \| w^\star \|^2$.
    We assume $\gamma_0 = 1/\sqrt{d}$, and we set $\|w^\star\|_2=1$ without loss of generality.
    \item \emph{For distribution $\mathcal{Q}$}, given the label $y^n_{i,1} \in \{-1, 1\}$ a scalar $\alpha \in \bR$, and two vectors $\zeta, z \in \bR^d$, the data $(x^n_{i,2}, y^n_{i,2})$ is constructed by
    \begin{equation*}
        \begin{split}
            y^n_{i,2} = y^n_{i,1}; \ \ 
            x^n_{i,2} = \alpha z \text{\ if \ } y^n_{i,2} = 1; \ \ 
            x^n_{i,2} \sim \text{Unif}\left(\{\alpha(z-\zeta), \alpha(z+\zeta)\}\right) \text{\ if \ } y^n_{i,2} = -1. 
        \end{split}
    \end{equation*}
    Different from distribution $\cP$, this distribution is not linear separable due to the construction of $x^n_{i,2}$. 
    Assume that $\alpha=1$, $\|z\|_2=u=\Theta(1)$, $\|\zeta\|_2=r=\Theta(1/\text{Poly}(d))$, and $\langle z, \zeta \rangle = 0$. 
\end{itemize}

Overall, distributions $\cP$ and $\cQ$ represent two types of components. $\cP$ represents the elementary knowledge and $\cQ$ represents the specialized knowledge. The above construction implies that fitting the distribution $\cP$ (linear separable) is easier than fitting the distribution $\cQ$ (nonlinear separable).

Figure \ref{fig:composite-data} provides a two-dimensional intuition for the roles of two components $\mathcal{P}$ and $\mathcal{Q}$ in learning both linear and nonlinear classifiers. 
As shown on the right, their concatenation creates a significantly more complex, composite nonlinear classification task. This demonstrates that despite the simple concatenation, the resulting data structure represents a challenging problem.
\begin{figure}
    \centering
    \includegraphics[width=0.9\linewidth]{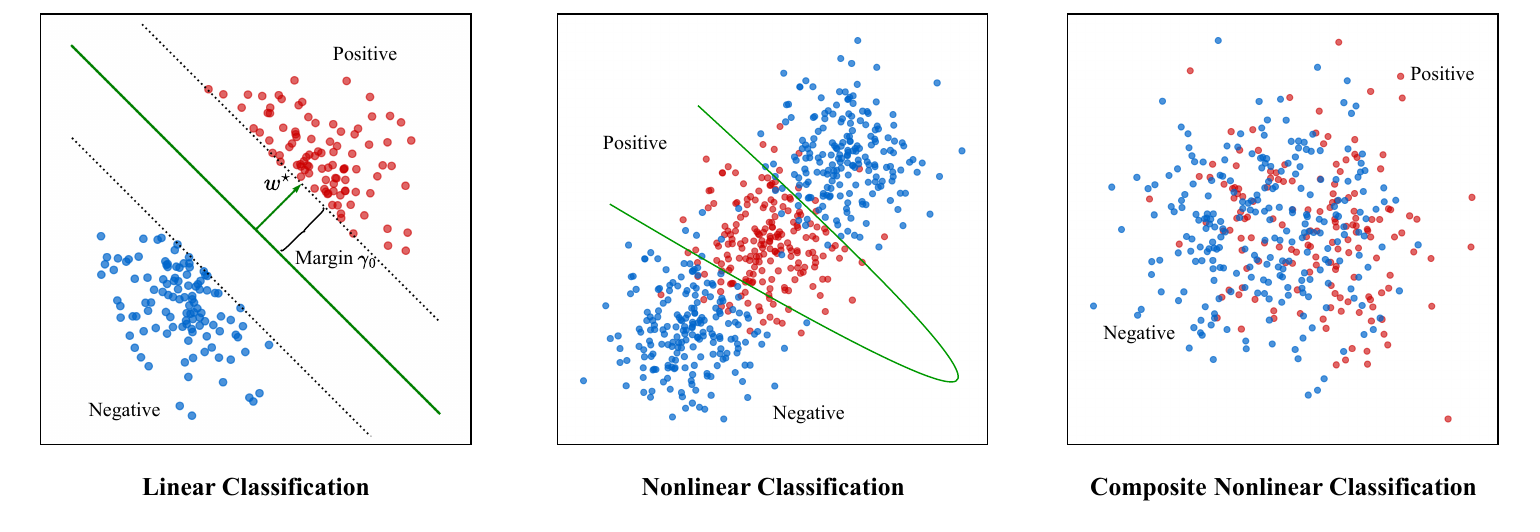}
    \caption{Composite nonlinear classification.}
    \label{fig:composite-data}
\end{figure}

\textit{Embeddings:}\ To simplify the presentation, we denote the embedding matrix by stacking $x^n_i$ or $y^n_i$. 
Specifically, for the feature embedding, denote
\begin{align*}
    X^n_1 = \begin{bmatrix}
		x^n_{1,1} & x^n_{2,1} & \cdots & x^n_{L,1}
	\end{bmatrix} \in \mathbb{R}^{d \times L},\ \ 
	X^n_2 = \begin{bmatrix}
		x^n_{1,2} & x^n_{2,2} & \cdots & x^n_{L,2}
	\end{bmatrix} \in \mathbb{R}^{d \times L}.
\end{align*}
Besides, to ensure the model output is linearly decomposable (See Equation~\ref{eq:f-decomp}), we combine $X_1$ and $X_2$ to form the complete feature embedding matrix as
$X^n = \begin{bmatrix}X^n_1 & 0 \\0 & X^n_2\end{bmatrix} \in \mathbb{R}^{2d \times 2L}$.
Similarly, define the label embedding as
\begin{align*}
    Y^n_1 = Y^n_2 \triangleq Y^n =
    \begin{bmatrix}
    y^n_1 & y^n_2 & \cdots & 0 
    \end{bmatrix} \in \mathbb{R}^{1 \times L},
\end{align*}
and the complete label embedding as $\widetilde{Y}^n = \begin{bmatrix}Y^n  &  Y^n
\end{bmatrix} \in \mathbb{R}^{1 \times 2L}$.

The preceding formulation provides a mathematical basis for our work. In Appendix~\ref{app:discussion}, we discuss the intuition and justification for this design, focusing on feature disentangling framework and theoretical abstraction.

\subsection{One-Layer Transformer Architecture}\label{sec:model}
This section introduces the notations of the one-layer transformer, including the normalized ReLU self-attention layer and transformer weight structure.

\textit{Normalized ReLU Self-Attention Layer:}\ A self-attention layer ~\cite{vaswani2017attention} in the single-head case includes parameters: key, query, and value matrices $W_K,W_Q \in \mathbb{R}^{2d \times 2d}$, $W_V \in \mathbb{R}^{2L \times 2L}$. Given the feature embedding matrix $X \in \mathbb{R}^{2d \times 2L}$, we use a normalized ReLU activation in place of standard softmax activation, following the theoretical work of Bai et al.~\cite{bai2024transformers} for theoretical tractability. Then the prediction for query $x_L$ using a one-layer transformer is given by
\begin{align}
	f(U; X, \widetilde{Y}) &= \widetilde{Y}W_V \cdot \frac{1}{2L}\text{ReLU}\left(X^\top W_K^\top W_Q x_L\right)
    = \widetilde{Y}/2L \cdot  \text{ReLU}\left(X^\top U x_L\right), \label{eq:f-original}
\end{align}
where $1/2L$ is the normalization factor. To simplify, we reparameterize $W_K^\top W_Q \triangleq U \in \mathbb{R}^{2d \times 2d}$ and assume the value matrix is the identity transformation, \emph{i.e.}, $W_V = I$. 
Although practical implementations often normalize by the sum of ReLU outputs~\cite{shen2023study, wortsman2023replacing}, our use of a fixed $1/2L$ normalization factor ensures that the attention weights remain non-negative and sum to $\mathcal{O}(1)$ in typical scenarios~\cite{bai2024transformers}. Crucially, these ReLU-based attention mechanisms empirically achieve faster speed and comparable performance to standard softmax in many vision and NLP tasks.

\textit{Transformer Weight Structure:}\ Given that individual samples $x^n_i$ can be characterized by two specific types of features, we abstract the real training network into two virtual networks, with the weight matrix composed of two distinct parts. To simplify our analysis, we consider the simplest structure of weight $U$ as a block diagonal matrix:
\begin{align*}
U = \begin{bmatrix}
	W & 0 \\
	0 & V
\end{bmatrix} \in \mathbb{R}^{2d \times 2d},
\end{align*}
where the sub-matrix $W$ operates only on $X_1$ and $V$ operates only on $X_2$.

We adopt this block diagonal structure as a key simplifying assumption for theoretical analysis. 
It is crucial to note that the model decomposability alone does not guarantee the emergence of two-stage behavior. For instance, a standard linear model is naturally decomposable across orthogonal features, yet it learns them simultaneously unless driven by highly specific feature variances. Therefore, our theoretical derivation with this architectural assumption remains meaningful, as it demonstrates that the two-stage behavior is driven by the properties of features, rather than the architectural decoupling. 

This structure exhibits a strong property of linear decomposability over the model output, \emph{i.e.}, by disentangling, the two new predictions with features $X_1$ and $X_2$ maintain a similar formulation to the original ones: 
\begin{align}
	\underbrace{f(U; X, \widetilde{Y})}_{N_U(U; X, \widetilde{Y})} = 1/2 \cdot \underbrace{Y/L \cdot \text{ReLU}\left(X^\top_1 W x_{L,1}\right)}_{N_W(W; X_1, Y)\  \text{or}\  h(X_1)} + 1/2 \cdot \underbrace{Y/L \cdot \text{ReLU} \left(X^\top_2 V x_{L,2}\right)}_{N_V(V; X_2, Y) \  \text{or}\  g(X_2)} \label{eq:f-decomp}.
\end{align}
In summary, we naturally abstract two virtual networks under a joint empirical loss: network $h(X_1)$ with parameter $W$ operates on $X_1$ part to learn component $\mathcal{P}$, and network $g(X_2)$ with parameter $V$ operates on $X_2$ part to learn component $\mathcal{Q}$. The overview is illustrated in Figure \ref{fig:overview}.

\subsection{Training Procedure}\label{sec:training-procedure}
This section introduces the training procedure for the model, including loss function, optimization algorithm, and the signal-noise decomposition technique employed for analysis.

\textit{Loss Function:}\ To train the transformer model on binary classification tasks, we consider the regularized empirical loss over $N$ training prompts. Denote the logistic loss for each prompt as $l(f(U; X^n, \widetilde{Y}^n))=\log (1+e^{-y^n_L f(U; X^n, \widetilde{Y}^n)})$, then
\begin{align}
	\widehat{L}(U) = \frac{1}{N}\sum_{n=1}^N l\left( f(U; X^n, \widetilde{Y}^n)\right), \label{eq:L-loss}
\end{align}
and the regularized loss is denoted as $\widehat{L}_\lambda(U) = \widehat{L}(U) + \frac{\lambda}{2}\|U\|^2_F$, where $\lambda$ denotes the $L_2$ regularization coefficient.

\textit{Optimization Algorithm:}\ Consider stochastic gradient descent with spherical Gaussian noise, which is a simplification of minibatch SGD. Taking initial weight $[U_0]_{ij} \sim \mathcal{N}\left(0,\tau_0^2\right)$ and noise $[\xi_t]_{ij} \sim \mathcal{N}\left(0,\tau_\xi^2\right)$, then the update of $U$ with time is represented as
\begin{align}
	U_{t+1} = U_{t} - \gamma_t \nabla_U (\widehat{L}_\lambda(U_t) + \xi_t)
    = (1-\gamma_t \lambda)U_t - \gamma_t\xi_t-\gamma_t\nabla_U \widehat{L}(U_t). \label{eq:U-update-rule}
\end{align}

\textit{Signal-noise Decomposition:}\ With noise in SGD optimization, we take signal-noise decomposition for weight $U=\overline{U}+\widetilde{U}$ ~\cite{allen2019convergence, li2019towards}.
The signal weight is driven by the deterministic gradients, \emph{i.e.},$\overline{U}_{t+1} \triangleq (1-\gamma_t \lambda)\overline{U}_{t} -\gamma_t\nabla_U \widehat{L}(U_t)$;
and the noise weight is driven by the injected noise, \emph{i.e.}, $\widetilde{U}_{t+1} \triangleq (1-\gamma_t \lambda)\widetilde{U}_{t} - \gamma_t\xi_t.$
Note that due to Equation~\ref{eq:U-update-rule}, such decomposition is always valid.

Notably, the noise component $\widetilde{U}$ follows a Gaussian distribution since it is a linear combination of Gaussian random variables.
By setting a relatively small variance $\tau_\xi^2$, the signal component dominates the weight trajectory~\cite{li2019towards}. 
Therefore, one can always rewrite the total weight $U_t=\overline{U}_t+\widetilde{U}_t$ as a deterministic signal part $\overline{U}_t$ with a small Gaussian random noise $\widetilde{U}_t$.
Based on this observation, we shift our focus from total weight dynamics to signal dynamics.
For any optimization iteration $t$, we define a surrogate loss $K_t(\overline{U})$ as a function of the signal weight:
\begin{align}
	K_t(\overline{U}) = \frac{1}{N}\sum_{n=1}^N l\left(N_{U_t}(\overline{U}+\widetilde{U}_t; X^n, \widetilde{Y}^n)\right). \label{eq:K-loss}
\end{align}
where the non-linear activation state (i.e., ReLU activation) is deterministically fixed by $U_t$, and the linear attention score is computed by varying the signal $\overline{U}$ coupled with the current noise $\widetilde{U}_t$.

This definition ensures three crucial properties: (a) for iteration $t$, $K_t(\overline{U})$ is a differentiable convex function; (b) $K_t(\overline{U}_t)$ is numerically identical to the original loss $\widehat{L}(U_t)$; (c) critically for our analysis, their gradients are equivalent due to the chain rule $\nabla_{\overline{U}}K_t(\overline{U}_t) \equiv \nabla_{U} \widehat{L}({U}_t)$.
This allows the evolution of signal weight to be represented as $\overline{U}_{t+1} \triangleq (1-\gamma_t \lambda)\overline{U}_{t} -\gamma_t \nabla_{\overline{U}}K_t(\overline{U}_t)$, which completely isolates the optimization dynamics of the signal weight for direct analysis.
Similarly, at iteration $t$, we take signal-noise decomposition for $W_t=\overline{W}_t+\widetilde{W}_t$ and $V_t=\overline{V}_t+\widetilde{V}_t$, then define the surrogate loss of linear separable component $\mathcal{P}$ over signal weight as $K^1_t(\overline{W})$, and the surrogate loss of nonlinear separable component $\mathcal{Q}$ over signal weight as $K^2_t(\overline{V})$: 
\begin{align}
	K^1_t(\overline{W}) = \frac{1}{N}\sum_{n=1}^N l\left(\frac{1}{2}N_{W_t}(\overline{W}+\widetilde{W}_t; X_1^n, Y^n)\right),
    K^2_t(\overline{V}) = \frac{1}{N}\sum_{n=1}^N l\left(\frac{1}{2}N_{V_t}(\overline{V}+\widetilde{V}_t; X_2^n, Y^n)\right). \label{eq:K12-loss}
\end{align}
Notice that while the parameter trajectory is driven by the global proxy $K_t(\overline{U}_t)$, the component-specific surrogate losses $K_t^1(\overline{W})$ and $K_t^2(\overline{V})$ serve as analytical metrics and allow us to independently evaluate the learning progress of $\mathcal{P}$ and $\mathcal{Q}$.
We make the standard assumption that the surrogate losses are $L_K$-Lipschitz continuous and $L_M$-smooth, which ensure uniformly bounded gradients and the stability of the discrete training dynamics.

\section{Two-stage Optimization of Transformers}\label{sec:main-theorems}
Based on the data characteristics and the different learning complexity of component $\mathcal{P}$ and $\mathcal{Q}$, we split the entire training process into two stages: the Elementary Stage (Theorem \ref{the:stage1-Q} and Theorem \ref{the:stage1-P} in Section~\ref{sec:elementary stage}), and the Specialized Stage (Theorem~\ref{the:stage2-Q} and Theorem \ref{the:stage2-P} in Section~\ref{sec:specialized stage}). We establish the weight trajectory and analyze the finite-time convergence in the two stages. 
The main theorems are summarized in Figure \ref{fig:the-summary}.
Before diving into the details, we introduce the fundamental settings of two stages, including the learning rate and training iterations. Specifically, 
\begin{itemize}[itemsep=2pt, topsep=0pt, parsep=0pt,leftmargin=2em]
    \item \textbf{Elementary Stage.} Learning rate $\eta_1 = \Theta(1/\sqrt{d})$; Containing $0\leq t \leq t_1 \triangleq \Theta\left(\frac{1}{\eta_1 \lambda}\right)$ where $\lambda$ denotes the $L_2$ regularization coefficient.
    \item \textbf{Specialized Stage.} Annealing learning rate $\eta_2 = \eta_1\lambda^2 r^2$ where $ r \triangleq \|\zeta\|_2 $ represents the hardness of specialized knowledge (See Section~\ref{sec:data}); Containing $t_1 \leq t \leq t_1 + t_2$, $t_2 \triangleq \Theta\left(\frac{1}{\eta_2 \lambda}\right)$.
\end{itemize}
The annealing learning rate is widely adopted in practical training procedures.
To position the role of learning rate schedule, we consider a scenario where all features exhibit a similar level of learning difficulty (\emph{e.g.}, all are simple). In such a case, the two-stage phenomenon would not occur, even with an annealed schedule. The model would concurrently learn all simple features during the initial phase with a large learning rate. 
Therefore, disparate feature complexity is the essential condition for the two-stage behavior, whereas the annealing schedule serves to make this underlying phenomenon cleanly observable.

In the following, we present the same choices of hyperparameters for two stages in Assumption \ref{ass:choice-hyperparam}.

\begin{restatable}{assumption}{assum}
\label{ass:choice-hyperparam}
Throughout the Theorems, we set the variance of initialization parameter $\tau_0 = \Theta(1/\sqrt{d})$, regularization coefficient $\lambda = \Theta(1/d^{5/2})$, prompt length $L=\Theta(\text{Poly}(d))$ and number of training prompts $N = \Theta\left(\text{Poly}(d)\right)$, where $d$ denotes the input dimension. 
\end{restatable}
We next validate the hyperparameter orders in  Assumption~\ref{ass:choice-hyperparam}.
\begin{enumerate}[label=\textit{(\alph*)},itemsep=0em, parsep=0em, topsep=0em, partopsep=0em, leftmargin=2em]
    \item \textit{$\tau_0$ denotes the standard deviation of the initialization parameter.}
    
    The requirement $\tau_0 = \Theta(1/\sqrt{d})$ suggests that as dimension $d$ increases and the data complexity grows, the variance should be adaptively decreased. 
    Notably, this scaling aligns with the standard Maximal Update Parametrization ($\mu$P)~\cite{yang2021tensor}. 
    In high-dimensional spaces, a higher variance would cause the forward activations and backward gradients to explode. Thus, this scaling acts as a critical mechanism to ensure stable training dynamics at initialization.

    \item \textit{$\lambda$ denotes the $L_2$ regularization coefficient in the loss function.}

    The requirement $\lambda=\Theta(1/d^{5/2})$ suggests that, as dimension $d$ increases, $\lambda$ should be adjusted to be correspondingly smaller. 
    This prevents a large $\lambda$ from overly constrain the model in high-dimensional scenarios, which would otherwise lead to underfitting. Furthermore, $t_1 \triangleq \Theta(\frac{1}{\eta_1 \lambda})$ implies that there might be a longer period during which the model struggles to effectively learn from more complex feature $\mathcal{Q}$, which accords with the empirical intuition.

    \item \textit{$L$ denotes the prompt length and $N$ denotes the number of training prompts.}

    The requirements $L  = \Theta(\text{Poly}(d))$ and $N  = \Theta(\text{Poly}(d))$ establish the necessary sample complexity to ensure strict statistical concentration. These suggest that the model anticipates longer and more input sequences for learning high-dimensional data, which accords with reality.
\end{enumerate}

\begin{figure}%
    \centering
    \includegraphics[width=0.9\linewidth]{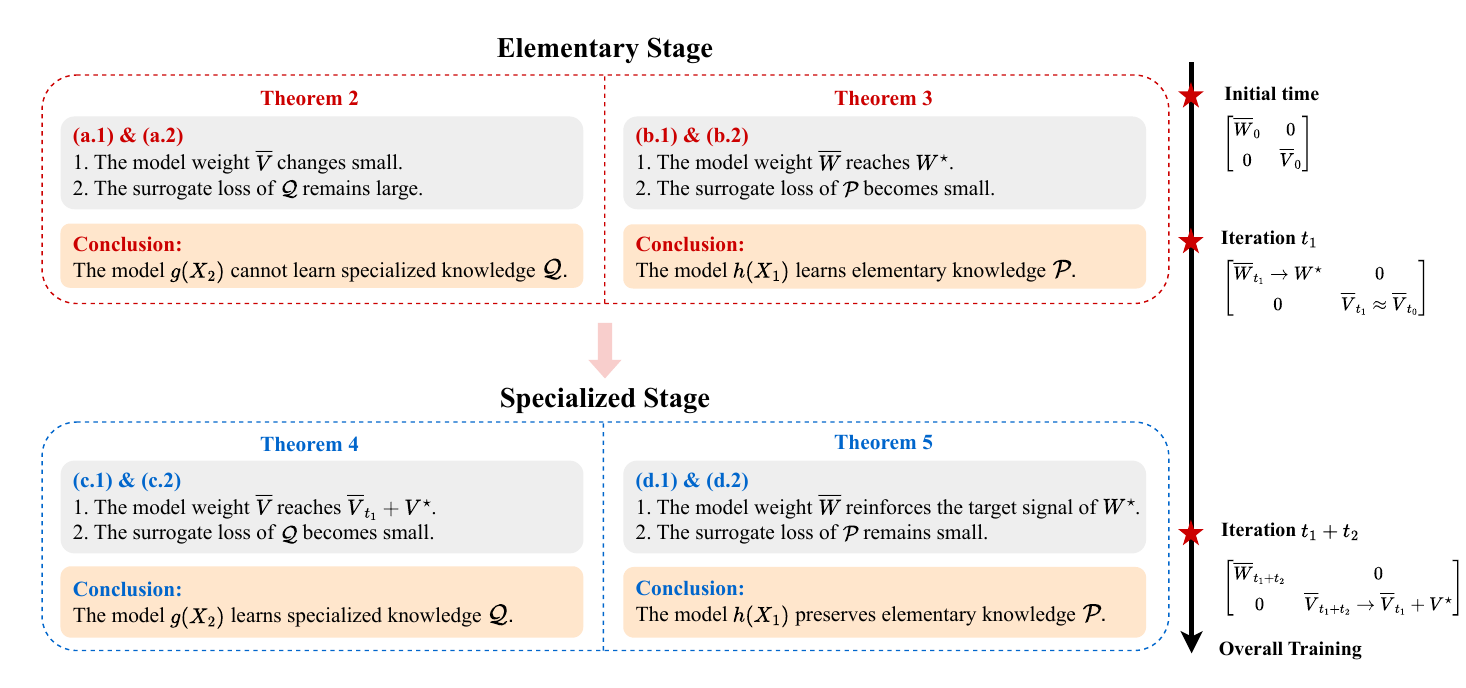}
    \caption{Summary of Two-stage Learning.}
    \label{fig:the-summary}
\end{figure}

\subsection{Elementary Stage}
\label{sec:elementary stage}
This section aims to analyze the regime with $\eta_1 = \Theta(1/\sqrt{d})$ and $t \leq t_1 \triangleq \Theta\left(\frac{1}{\eta_1 \lambda}\right)$. Our goal is to prove that the weights are optimized from $\overline{U}_0 = \begin{bmatrix}\overline{W}_0 & 0 \\ 0 & \overline{V}_0\end{bmatrix}$ to $\overline{U}_{t_1} = \begin{bmatrix}\overline{W}_{t_1}\xrightarrow{}W^\star & 0 \\ 0 & \overline{V}_{t_1}\approx \overline{V}_{0}\end{bmatrix}$. This means that $\overline{W}_{t_1}$ approaches the optimal weight $W^\star$, while $\overline{V}_{t_1}$ remains close to $\overline{V}_{0}$. 
We split the derivation into two theorems: Theorem~\ref{the:stage1-Q} demonstrates that the component $\mathcal{Q}$ (specialized knowledge) is not effectively learned by network $g$, and Theorem~\ref{the:stage1-P} demonstrates that the network $h$ successfully learns the component $\mathcal{P}$ (elementary knowledge). We begin the analysis with Theorem~\ref{the:stage1-Q}.

\begin{restatable}{theorem}{theStageIQ}
\label{the:stage1-Q}
    In the elementary stage with $\eta_1 =\Theta(1/\sqrt{d})$ and $t_1 = \Theta\left(\frac{1}{\eta_1 \lambda}\right)$ where $\lambda$ denotes the regularization coefficient. With Assumption \ref{ass:choice-hyperparam} and initial weight $V_0 \sim \mathcal{N}(0,\tau_0^2 I)$ (signal weight $\overline{V}_0=\mathbf{0}_{d\times d}$), it holds that
    
    \noindent \textbf{(a.1)} For the model parameter $V$ of network $g$, through gradient descent,
    its signal weight norm $\|\overline{V}_{t_1}\|_F$ satisfies 
    \begin{align*}
        \|\overline{V}_{t_1}\|_F \lesssim \frac{1}{\text{Poly}(d)}.
    \end{align*}
    
    \noindent \textbf{(a.2)} With random and small noise weight, the surrogate loss of nonlinear separable component $\mathcal{Q}$ over signal weight (Definition in Equation \ref{eq:K12-loss}) at iteration $t_1$ satisfies
    \begin{align*}
        K^2_{t_1}\left(\overline{V}_{t_1}\right) \gtrsim \log 2 -\frac{1}{\text{Poly}(d)} - \sqrt{\frac{\log d}{N}}.
    \end{align*}
    Namely, the network ${g}$ fails to learn the nonlinear separable component $\mathcal{Q}$ within $t_1$ iterations.
\end{restatable}
\textit{Messages Behind Theorem~\ref{the:stage1-Q}.} \quad Theorem ~\ref{the:stage1-Q} demonstrates that the component $\mathcal{Q}$ cannot be effectively learned by the corresponding network $g$ defined in Equation \ref{eq:f-decomp}. \textbf{In (a.1)}, within $t_1$ iterations, the weight $\|\overline{V}_{t_1}\|_F$ is upped-bounded by order $\frac{1}{\text{Poly}(d)}$, which implies that the signal weight accumulates negligibly from its initial state. \textbf{In (a.2)}, we provide the lower bound for the training loss of component $\mathcal{Q}$. The value is close to $\log 2$ with a large dimension $d$ and training prompts $N$, which corresponds to the loss of a random guess.  
Overall, the above discussions exhibit that the network $g$ fails to effectively learn the specialized knowledge like $\mathcal{Q}$ during the elementary stage.

\textit{Proof Sketch.}\quad
The proof sketch for Theorem \ref{the:stage1-Q} is summarized below, and we defer the detailed proof to Appendix \ref{proof:stage1-Q}.
\textit{The proof begins by employing a signal-noise decomposition for the network weights.}
We first show that the network output $g_t$ can be approximated by an auxiliary function $\widetilde{g}_t$, which uses the noise weight for activation state calculations and the signal weight for attention score. With a large prompt length $L=\Theta(\text{Poly}(d))$, the approximation residual is bounded by an asymptotically small quantity $\mathcal{O}(1/\text{Poly}(d))$ (using Corollary \ref{coro:lemma-ac-V} and Corollary \ref{coro:N-V-tV}).

\textit{The core of the proof analyzes the capacity of $\widetilde{g}_t$ to differentiate between positive and negative samples.} This is achieved by examining the upper bound on $\left|\widetilde{g}_t (X_2, z-\zeta) + \widetilde{g}_t (X_2, z+\zeta) - 2 \widetilde{g}_t (X_2, z)\right|$, the difference in network outputs for these classes.
By decomposing this term and applying concentration inequalities, we establish that this difference remains $\mathcal{O}(1/\text{Poly}(d))$, which in turn implies that the network cannot concurrently achieve high accuracy in both positive and negative samples.

\textit{Finally, we connect the network's output behavior to the weight norm and surrogate loss to establish statements (a.1) and (a.2).} We express the total empirical loss as a function of the network outputs for different sample types, weighted by their probabilistic occurrences. Leveraging the Lipschitz continuity and convexity of the logistic loss, we establish a lower bound for the empirical loss tightly near $\log 2$, proving (a.2).
By coupling this lower bound with an upper bound from the minimizing nature of gradient descent, we get the magnitude of the network's output, which states $|g_{t_1}| \lesssim (\log d/N)^{1/4}$. 
Subsequently, we analyze the empirical gradient and prove that the gradient norm is dominated by a vanishingly small order. Integrating this over the elementary stage duration establishes that the weight norm $\|\overline{V}_{t_1}\|_F$ remains bounded by $\mathcal{O}(1/\text{Poly}(d))$, proving (a.1).

\begin{restatable}{theorem}{theStageIP}
\label{the:stage1-P}
    In the elementary stage with $\eta_1 =\Theta(1/\sqrt{d})$ and $ t_1 = \Theta\left(\frac{1}{\eta_1 \lambda}\right)$ where $\lambda$ denotes the regularization coefficient. With Assumption~\ref{ass:choice-hyperparam} and initial weight $W_0 \sim \mathcal{N}(0,\tau_0^2 I)$ (signal weight $\overline{W}_0=\mathbf{0}_{d\times d}$), it holds that
 
    \noindent \textbf{(b.1)} For the model parameter $W$ of network $h$, through gradient descent, its signal weight norm $\|\overline{W}_{t_1}\|_F$ satisfies
	\begin{align*}
		\|\overline{W}_{t_1}\|_F \gtrsim d\log d-d\log \log d -\sqrt{d}\log d  \gg \|\overline{W}_{0}\|_F.
	\end{align*}    
    
	\noindent \textbf{(b.2)} With random and small noise weight, the surrogate loss of linear separable component $\mathcal{P}$ over signal weight (Definition in Equation \ref{eq:K12-loss}) at iteration $t_1$ satisfies
	\begin{align*}
		K^1_{t_1}(\overline{W}_{t_1}) \lesssim \frac{1}{\text{Poly}(d)}+\frac{(\log d)^2}{\sqrt{d}}+\left(\frac{\log d}{N}\right)^{1/4}.
	\end{align*}
Namely, the network $h$ learns the linear separable component $\mathcal{P}$ within $t_1$ iterations.
\end{restatable}

\textit{Messages Behind Theorem~\ref{the:stage1-P}.} \quad Theorem~\ref{the:stage1-P} describes how the linear separable component $\mathcal{P}$ is learned by the network $h$ defined in Equation~\ref{eq:f-decomp}.
\textbf{In (b.1)}, within $t_1$ iterations, the signal weight norm $\|\overline{W}_{t_1}\|_F$ significantly grows from zero ($\overline{W}_0 = \mathbf{0}$) to an asymptotic order of $\Omega(d \log d)$, indicating that the knowledge might be learned. 
In comparison, $\overline{V}_{t_1}$ for the component $\mathcal{Q}$ changes small, since $\|\overline{V}_{t_1}\|_F \lesssim \frac{1}{\text{Poly}(d)}$, which implies $\overline{V}_{t_1} \approx \overline{V}_0$ (See Theorem \ref{the:stage1-Q} (a.1)).
\textbf{In (b.2)}, it shows that the surrogate loss of linear separable component $\mathcal{P}$ is upper bounded by an $o(1)$ term which converges to zero as the dimension $d$ goes to infinity. 
In comparison, the loss of component $\mathcal{Q}$ is lower-bounded by a constant close to $\log 2$ (See Theorem \ref{the:stage1-Q} (a.2)).
In summary, the above discussions imply that \textbf{the network $h$ learns elementary knowledge $\mathcal{P}$, marking the so-called elementary stage.}

\textit{Proof Sketch.}\quad
The proof sketch for Theorem \ref{the:stage1-P} is summarized below, and we defer the detailed proof to Appendix \ref{proof:stage1-P}.
\textit{The analysis begins by examining the output of network $h$ under an optimal weight $W^\star$}, using a signal-noise decomposition. This separates the target output into two components: one determined by the optimal signal weight and another by the random noise weight.
The bound for the noise-driven component is established by leveraging Propositions \ref{prop:N-tU-tU}$\sim$\ref{prop:N-U-tU} and Corollary \ref{coro:N-W-tW} that detail the calculation of activations and attention scores.
The bound for the signal-driven component is determined by the properties of the optimal weight matrix $W^\star$, and the defined data structure of component $\mathcal{P}$.
Combining these bounds proves that the surrogate loss evaluated at this optimal target state is bounded by a vanishingly small quantity $\epsilon_{W,1}=\Theta(1/\text{Poly}(d))$ (See Definition in Equation \ref{eq:def-epsilonw1}). 

\textit{Following this, a gradient descent analysis is conducted} to track the weight optimization trajectory. A proof by contradiction establishes that the global surrogate loss $K_t(\overline{U})$ (i.e., $K_t(\overline{W},\overline{V})$) descends into the target neighborhood of the optimal loss within $t_1$ iterations. Together with Theorem~\ref{the:stage1-Q}, we decouple this global loss to bound the surrogate loss $K^1_{t_1}(\overline{W}_{t_1})$, isolating the pure learning progress on component $\mathcal{P}$ and proving (b.2).

\textit{Finally, we structurally analyze the learned signal weight to establish statement (b.1).} By expanding the gradient accumulation, we prove that $\overline{W}_{t_1}$ is dominated by a rank-1 matrix aligned with the target feature $w^\star$. This forces that $\|\overline{W}_{t_1}\|_F$ grows to $\Omega(d \log(1/\mathcal{E})-\sqrt{d}\log d)$ when $K^1_{t_1}(\overline{W}_{t_1})\lesssim \mathcal{E}$, asymptotically matching the target signal $W^\star$ and confirming the successful learning of component $\mathcal{P}$.

\subsection{Specialized Stage}\label{sec:specialized stage}
This section aims to analyze the regime with $\eta_2 = \eta_1 \lambda^2 r^2$ and $t_1 \leq t \leq t_1+t_2$, $t_1 \triangleq \Theta\left(\frac{1}{\eta_1 \lambda}\right)$, $t_2 \triangleq \Theta\left(\frac{1}{\eta_2 \lambda}\right)$. Our goal is to prove that the weights are optimized from $\overline{U}_{t_1} = \begin{bmatrix}\overline{W}_{t_1} & 0 \\ 0 & \overline{V}_{t_1}\end{bmatrix}$ to $\overline{U}_{t_1+t_2} = \begin{bmatrix}\overline{W}_{t_1+t_2} & 0 \\ 0 & \overline{V}_{t_1+t_2} \xrightarrow{}\overline{V}_{t_1}+V^\star\end{bmatrix}$. 
In total, we split the derivation into two theorems: Theorem~\ref{the:stage2-Q} demonstrates that the network $g$ learns specialized knowledge like component $\mathcal{Q}$, and Theorem~\ref{the:stage2-P} demonstrates that the network $h$ continues to preserve elementary knowledge like component $\mathcal{P}$.
We start from Theorem~\ref{the:stage2-Q}.

\begin{restatable}{theorem}{theStageIIQ}
\label{the:stage2-Q}
In the specialized stage with annealing learning rate $\eta_2 = \eta_1 \lambda^2 r^2$ and $t_1 \leq t \leq t_1+t_2$, $t_1 = \Theta\left(\frac{1}{\eta_1 \lambda}\right)$, $t_2 = \Theta\left(\frac{1}{\eta_2 \lambda}\right)$, $\lambda$ denotes the regularization coefficient and data noise $\|\zeta\|_2 = r$ (See Section \ref{sec:data}). With Assumption \ref{ass:choice-hyperparam}, it holds that

    \noindent \textbf{(c.1)} For the model parameter $V$ of network $g$, through gradient descent,
    its signal weight norm $\|\overline{V}_{t_1+t_2}\|_F$ satisfies 
    \begin{align*}
        \|\overline{V}_{t_1+t_2}\|_F \gtrsim d^{9/8}\log d - \frac{1}{\text{Poly}(d)}\gg \|\overline{V}_{t_1}\|_F.
    \end{align*}
    
    \noindent  \textbf{(c.2)} With random and small noise weight, the surrogate loss of nonlinear separable component $\mathcal{Q}$ over signal weight (Definition in Equation \ref{eq:K12-loss}) at iteration $t_1+t_2$ satisfies
	\begin{align*}
		K_{t_1+t_2}^2(\overline{V}_{t_1+t_2}) \lesssim \exp\left(-d^{9/8}\log d\right).
	\end{align*}	
Namely, the network $g$ learns nonlinear separable component $\mathcal{Q}$ within $t_2$ iterations.
\end{restatable}

\textit{Messages Behind Theorem~\ref{the:stage2-Q}.} \quad Theorem~\ref{the:stage2-Q} illustrates the optimization on the nonlinear separable part $\mathcal{Q}$ in the specialized stage, annealing the learning rate from $\eta_1$ to $\eta_2$.
\textbf{Statement (c.1)} implies that within $t_2$ iterations, the signal weight significantly grows from the order $\|\overline{V}_{t_1}\|_F \lesssim \frac{1}{\text{Poly}(d)}$ to an asymptotic order $\Omega(d^{9/8}\log d)$.
\textbf{Statement (c.2)} implies that the loss is upper bounded by $o(1)$ which converges to zero as $d$ goes to infinity. Compared to Theorem~\ref{the:stage1-Q} with constant lower bound, we conclude that \textbf{with a small learning rate, the network $g$ learns specialized knowledge $\mathcal{Q}$, marking the so-called specialized stage.} 

\textit{Proof Sketch.}\quad
The proof sketch for Theorem \ref{the:stage2-Q} is summarized below, and we defer the detailed proof to Appendix \ref{proof:stage2-Q}.
\textit{The analysis begins by examining the output of network $g$ under our constructed target signal weight $\overline{V}_{t_1} + V^\star$ at iteration $t_1+t_2$.} 
Using a triangle inequality, the output is decomposed into three distinct parts. We establish the upper bound by analyzing the structural properties of the target signal, the network output at time $t_1$, and the activation drift between iterations $t_1$ and $t_1+t_2$ (Lemma \ref{lemma:activation-patterns-t1+t-t1} and Corollary \ref{coro:lemma-ac-V-t1+t-t}). Combining these, we prove that the surrogate loss at this target state is bounded by a vanishingly small quantity, $K^2_{t_1+t_2}(\overline{V}_{t_1}+V^\star) \lesssim \epsilon_{V,1}$, i.e., $\mathcal{O}(1/\text{Poly}(d))$. This validates the constructed weight as a valid goal for the subsequent optimization.

\textit{Following this, a gradient descent analysis is conducted} to track the optimization trajectory of the signal weight. A proof by contradiction establishes that the distance to the target signal strictly contracts within $t_2$ iterations. It guarantees that the signal weight norm grows to an asymptotic order $\|\overline{V}_{t_1+t_2}\|_F = \Omega(d^{9/8}\log d)$, proving statement (c.1).

\textit{Finally, we structurally analyze the learned signal weight $\overline{V}_{t_1+t_2}=(1-\eta_2\lambda)^{t_2}\overline{V}_{t_1}+\Delta \overline{V}_{t_2}$ to establish the surrogate loss statement (c.2).} 
By expanding the trajectory of gradient updates, we decompose the accumulated parameter $\Delta \overline{V}_{t_2}$ into an expectation term, a zero-mean noise term, and an activation drift term. Applying concentration inequalities, we bound the latter two terms by $\mathcal{O}(1/\text{Poly}(d))$. 
Subsequently, by expanding the network output margin under $V_{t_1+t_2}$ and leveraging the structural decomposition of $\Delta V_{t_2}$ alongside $\|\Delta \overline{V}_{t_2}\|_F = \Theta(d^{9/8}\log d)$, we determine the dominant margin. This guarantees that the surrogate loss evaluated at $\overline{V}_{t_1+t_2}$ yields an upper bound of $\exp(-d^{9/8}\log d)$, confirming the successful learning of component $\mathcal{Q}$.

\begin{restatable}{theorem}{theStageIIP}
\label{the:stage2-P}
	In the specialized stage with annealing learning rate $\eta_2 = \eta_1 \lambda^2 r^2$ and $t_1 \leq t \leq t_1+t_2$, $t_1 = \Theta\left(\frac{1}{\eta_1 \lambda}\right)$, $t_2=\Theta\left(\frac{1}{\eta_2 \lambda}\right)$, $\lambda$ denotes the regularization coefficient and data noise $\|\zeta\|_2 = r$ (See Section \ref{sec:data}). With the target linear classifier $w^\star$ (See Section \ref{sec:data}), Assumption \ref{ass:choice-hyperparam} and a decay factor $\alpha = (1-\eta_2\lambda)^{t_2} \in (0,1)$, it holds that
 
    \noindent \textbf{(d.1)} For the model parameter $W$ of network $h$, through gradient descent, its signal weight $\overline{W}_{t_1+t_2}$ and the shifted norm satisfy
	\begin{align*}
		\overline{W}_{t_1 + t_2} = \alpha \overline{W}_{t_1} + \mu^\prime w^\star (w^\star)^\top + E^\prime, \quad \text{and} \quad
        \|\overline{W}_{t_1 + t_2} - \alpha \overline{W}_{t_1}\|_F \lesssim d\log^2 d,
	\end{align*}
    with a scalar multiplier $0 \leq \mu^\prime \lesssim d\log^2 d$ and a noise matrix bounded by $\|E^\prime\|_F \lesssim \frac{1}{\text{Poly}(d)}$.    
    
	\noindent \textbf{(d.2)} With random and small noise weight, given $K_{t_1}^1(\overline{W}_{t_1}) \lesssim \mathcal{E}$ where $\mathcal{E}$ is the upper bound defined in Theorem~\ref{the:stage1-P} (b.2), the surrogate loss of linear separable component $\mathcal{P}$ over signal weight (Definition in Equation \ref{eq:K12-loss}) at iteration $t_1+t_2$ satisfies
	\begin{align*}
		K_{t_1 + t_2}^1(\overline{W}_{t_1 + t_2}) \lesssim \mathcal{E}^{\alpha/2}.
	\end{align*}
    Namely, the network $h$ continues to preserve the knowledge $\mathcal{P}$ within $t_2$ iterations.
\end{restatable}

\textit{Messages Behind Theorem~\ref{the:stage2-P}.} \quad Theorem~\ref{the:stage2-P} demonstrates the optimization process on the linear separable part $\mathcal{P}$ in the specialized stage, annealing the learning rate from $\eta_1$ to $\eta_2$.
\textbf{Statement (d.1)} demonstrates that the signal weight $\overline{W}$ experiences a structured evolution rather than catastrophic drift. This dynamic is governed by a decay factor $\alpha$, a target reinforcement multiplier $\mu^\prime$, and a negligible noise perturbation $\|E^\prime\|_F \lesssim \frac{1}{\text{Poly}(d)}$. Concretely, while the shifted weight magnitude can reach up to $\mathcal{O}(d\log^2 d)$, it is confined to reinforcing the target signal direction $w^\star (w^\star)^{\top}$. 
\textbf{Statement (d.2)} demonstrates that the model maintains low surrogate loss on the linear separable component $\mathcal{P}$ from iteration $t_1$ to $t_1+t_2$. In detail, the loss bound evolves from $\mathcal{E}$ to $\mathcal{E}^{\alpha/2}$, scaling as $\Theta\left((\frac{\log^2 d}{\sqrt{d}})^{\alpha/2}\right)$, which preserves the asymptotically small error.
In summary, given that the signal weight $\overline{W}$ reinforces its target direction and the loss remains minimal, \textbf{in the specialized stage, the network $h$ continues to preserve the knowledge $\mathcal{P}$ acquired during the elementary stage.}

\textit{Proof Sketch.}\quad
The proof sketch for Theorem \ref{the:stage2-P} is summarized below, and the detailed proof is deferred to Appendix \ref{proof:stage2-P}.
\textit{The analysis begins by conducting a gradient descent analysis on the signal weight.} Similar to the analysis of $\overline{W}_{t_1}$ in Theorem~\ref{the:stage1-P}, we expand the accumulated gradient to establish its structural decomposition and the upper bound of shifted weight difference $\|\overline{W}_{t_1+t_2} - \alpha \overline{W}_{t_1}\|_F$, proving statement (d.1).
\textit{Following this, we connect the signal weight dynamic to the surrogate loss.} By expanding the network output margin, we apply concentration inequalities to bound the margin fluctuations. Statement (d.2) follows directly from synthesizing these margin bounds. In total, the proof demonstrates that both the weight structure and surrogate loss for component $\mathcal{P}$ remain stable during the specialized stage.

\subsection{Extensions in Spectral Characteristics of Attention Weights}\label{sec:coro}
This section extends our dynamical analysis to investigate the spectral characteristics of the attention weights, formally establishing the relationship between $\text{Tr}(W)$ and $\text{Tr}(V)$ across the two stages.

\begin{restatable}{corollary}{coroSpectral}
\label{coro:spectral-characteristics}
Under the assumptions in Theorems~\ref{the:stage1-Q}$\sim$\ref{the:stage2-P}, it holds that

    \noindent \textbf{(a)} In the elementary stage within $t_1 \triangleq \Theta\left(\frac{1}{\eta_1 \lambda}\right)$ iterations, the spectral trace satisfies 
    $$\text{Tr}(W_{t_1}) > \text{Tr}(V_{t_1}).$$
    
    \noindent \textbf{(b)} In the specialized stage within $t_2 \triangleq \Theta\left(\frac{1}{\eta_2\lambda}\right)$ iterations, the spectral trace satisfies 
    $$\text{Tr}(W_{t_1+t_2}) < \text{Tr}(V_{t_1+t_2}).$$
\end{restatable}

\textit{Messages Behind Corollary~\ref{coro:spectral-characteristics}.}\quad 
Corollary~\ref{coro:spectral-characteristics} is straightforward to derive from Theorems \ref{the:stage1-Q}$\sim$\ref{the:stage2-P}. 
It implies that when the model is sufficiently trained (at time $t_1+t_2$), relatively small eigenvalues of attention weights store elementary knowledge and large ones store specialized knowledge. 
We defer the detailed proof to Appendix \ref{sec:appendix-spectral}.

\textit{Empirical Observations.}\quad
We provide empirical observations in Appendix~\ref{sec:app-exp} to complement our theoretical findings. 
Specifically, we conduct controlled experiments on a synthetic dataset constructed to match our theoretical framework (Appendix~\ref{sec:app-exp-synthetic}), where the training dynamics exhibit a clear two-stage phenomenon. 
Furthermore, to demonstrate the broader applicability of our insights, we qualitatively explore real-world scenarios by fine-tuning GPT-2 on the CounterFact and HotpotQA datasets (Appendix~\ref{sec:app-exp-counterfact}). We observe both the elementary and specialized learning stages analogous to Theorems~\ref{the:stage1-Q}$\sim$\ref{the:stage2-P}, as well as the spectral characteristics of attention weights in Corollary~\ref{coro:spectral-characteristics}.

\section{Conclusion}\label{sec:conclusions}
This paper provides rigorous proof for the feature-level two-stage learning of transformers in ICL tasks. 
We disentangle token feature structure into two types: \emph{elementary knowledge}, and \emph{specialized knowledge}. 
By employing feature learning and signal-noise decomposition techniques, we analyze the optimization trajectory, finite-time convergence, and spectral characteristics under ICL regimes, offering deeper insights into the optimization process. 
Our theory considers a simple one-layer attention-based model, and it is promising to explore and extend to other architectures in the future.
Our work potentially provides a new perspective and theoretical framework for understanding the optimization dynamics of transformers.

\section*{Acknowledgments}
This research was supported in part by Beijing Natural Science Foundation under Grant Z250001, in part by National Key Research and Development Program of China under Grant 92570203 and Grant 2024YFE0203200, in part by National Natural Science Foundation of China under Grant 62476277, in part by CCF-ALIMAMA TECH Kangaroo Fund under Grant CCF-ALIMAMA OF 2024008, in part by Huawei-Renmin University joint program on Information Retrieval, in part by the fund for building worldclass universities (disciplines) of Renmin University of China, in part by the funds from Beijing Key Laboratory of Big Data Management and Analysis Methods, Gaoling School of Artificial Intelligence, Renmin University of China, in part by the Engineering Research Center of Next-Generation Intelligent Search and Recommendation, in part by the Ministry of Education, from Intelligent Social Governance Interdisciplinary Platform, in part by the Major Innovation \& Planning Interdisciplinary Platform for the ``DoubleFirst Class'' Initiative, in part by Renmin University of China, in part by the Public Policy and Decision-making Research Lab of Renmin University of China, and in part by the Public Computing Cloud of Renmin University of China.

\appendices
\begin{center}
    \Large{Appendices}
\end{center}
\startcontents[section]
\begin{spacing}{1.1} 
    \printcontents[section]{l}{1}{\setcounter{tocdepth}{2}}
\end{spacing}

\newpage
\section{Table of Notations}\label{sec:notation}
\renewcommand{\arraystretch}{1.5}
\begin{longtable}{p{5.9cm}p{8.6cm}}
	\caption{Table of Notations.} \label{tab:notation} \\
	\specialrule{1pt}{0pt}{0pt}
	\toprule
	\textbf{Notation} & \textbf{Description} \\
	\midrule
	\endfirsthead

	\multicolumn{2}{c}%
	{{\tablename\ \thetable{}}} \\
	\toprule
	\textbf{Notation} & \textbf{Description} \\
	\midrule
	\endhead

	\specialrule{1pt}{0pt}{0pt}
	\bottomrule
	\endlastfoot

    $t_1$ & Total iterations of the elementary stage \\
    $t_2$ & Total iterations of the specialized stage \\
	$N$ & Number of training prompts \\
    $L$ & Training prompt length (the last token is a query) \\
    \midrule
	$x^{n}_{i} = [x^n_{i,1}, x^n_{i,2}]^\top \in \mathbb{R}^{2d}$ & Divide the $i$-th token of $n$-th training prompts into two parts \\
	$x^n_{i,1} \sim \mathcal{P} \in \mathbb{R}^{d}$ & The elementary knowledge in a token \\
    $x^n_{i,2} \sim \mathcal{Q} \in \mathbb{R}^{d}$ & The specialized knowledge in a token \\
    $X^n_1 = \begin{bmatrix}x^n_{1,1} & x^n_{2,1} &\cdots & x^n_{L,1}\end{bmatrix} \in \mathbb{R}^{d\times L}$ & Stack of $x^n_{i,1}$ \\  
    $X^n_2 = \begin{bmatrix}x^n_{1,2} & x^n_{2,2} &\cdots & x^n_{L,2}\end{bmatrix} \in \mathbb{R}^{d\times L}$ & Stack of $x^n_{i,2}$ \\
    $X^n = \begin{bmatrix}X^n_1 & 0 \\0 & X^n_2\end{bmatrix} \in \mathbb{R}^{2d \times 2L}$ & Stack of $X^n_1$ and $X^n_2$ \\
    $y^n_i \in \{-1, 1\}$ & Binary classification label \\
    $Y^n = \begin{bmatrix}y^n_{1} & y^n_{2} &\cdots & 0\end{bmatrix} \in \mathbb{R}^{1 \times L}$ & Stack of $y^n_i$ \\
    $\widetilde{Y}^n = \begin{bmatrix}Y^n & Y^n \end{bmatrix} \in \mathbb{R}^{1 \times 2L}$ & Stack of $Y^n_1$ and $Y^n_2$ \\
    \midrule 
    $f(U; X, \widetilde{Y})$ & Normalized ReLU self-attention output, Equation \ref{eq:f-original} \\
    $h(X_1)$ & Virtual network operates on $X_1$, Equation \ref{eq:f-decomp} \\
    $g(X_2)$ & Virtual network operates on $X_2$, Equation \ref{eq:f-decomp} \\
    \midrule
    $U=\begin{bmatrix}W & 0 \\ 0 & V \end{bmatrix} \in \mathbb{R}^{2d \times 2d}$ & Model parameter of normalized ReLU self-attention network\\
    $U=\overline{U}+\widetilde{U} \in \mathbb{R}^{2d \times 2d}$ & Signal-noise decomposition of weight $U$ \\
    $W =\overline{W}+\widetilde{W}\in \mathbb{R}^{d \times d}$ & Model parameter of virtual network $h$, signal-noise decomposition of weight $W$ \\
    $V = \overline{V}+\widetilde{V} \in \mathbb{R}^{d \times d}$ & Model parameter of virtual network $g$, signal-noise decomposition of weight $V$\\
	\midrule
	$\widehat{L}(U)$ & The global empirical loss over weight $U$, Equation \ref{eq:L-loss} \\
    $K_t(\overline{U})$ & The global empirical loss over signal weight $\overline{U}$, Equation \ref{eq:K-loss} \\
    $K^1_t(\overline{W})$ & The surrogate loss over signal weight $\overline{W}$, Equation \ref{eq:K12-loss} \\
    $K^2_t(\overline{V})$ & The surrogate loss over signal weight $\overline{V}$, Equation \ref{eq:K12-loss} \\
\end{longtable}

\newpage
\section{Summary of Technical Contributions}\label{sec:thecontributions}  
Our theoretical framework is inspired by Li et al.~\cite{li2019towards}, but extends to understanding transformers, and constructing ICL format data to facilitate analysis. Our crucial technical contributions and theoretical advancements are highlighted in the following three key aspects:

\begin{enumerate}[label=\textit{(\alph*)},itemsep=0em, parsep=0em, topsep=0em, partopsep=0em, leftmargin=*]
    \item \textit{Architectural complexities of Transformers and ICL framework.}

    Transitioning from the standard two-layer neural networks in Li et al.~\cite{li2019towards} to a Transformer architecture necessitates the In-Context Learning (ICL) framework. 
    ICL elegantly translates supervised learning tasks into a sequence format, aligning with the self-attention mechanism's core function of computing token-to-token similarities. However, this introduces major structural challenges, primarily creating complex \emph{bilinear attention score computations} over context tokens and the corresponding \textit{ReLU activation patterns}.
    To tackle this, we develop essential Lemmas and Corollaries in Appendix~\ref{sec:appendix-key-prop-lemma-coro}, discussing the activation values difference, network output difference and network output upper bound.
    Ultimately, based on the bilinear nature, we construct and prove that the signal weights for components $\mathcal{P}$ and $\mathcal{Q}$ converge to distinct optimal structural targets $W^\star$ and $V^\star$.

    \item \textit{Fine-grained Characterization of Optimization Trajectories.}

    We establish a strictly finer-grained characterization of the feature-level two-stage optimization by analyzing weight trajectory and loss variation. To construct a unified theory, we investigate the F-norm of signal weights and bounding the surrogate losses. Specifically, in Theorems~\ref{the:stage1-Q}$\sim$\ref{the:stage2-P}:
    \begin{itemize}[leftmargin=*]
        \item \textit{Theorem~\ref{the:stage1-Q}:} We explicitly derive a constant lower bound near $\log 2$ for component $\mathcal{Q}$ alongside a vanishing signal weight norm. This proves the model acts purely as a random guess on specialized knowledge, advancing beyond the positive-negative sample output differences analyzed in Li et al.~\cite{li2019towards}.
   
        \item \textit{Theorem~\ref{the:stage1-P}:} 
        Unlike Li et al.~\cite{li2019towards} relies on proof by contradiction to solely bound the global loss, we explicitly decouple this global loss to bound the specific \textit{surrogate loss} for component $\mathcal{P}$. Crucially, we introduce \textit{a novel structural analysis} of the gradient accumulation, proving that the signal weight $\overline{W}_{t_1}$ is dominated by a rank-1 matrix aligned with the target feature $w^\star$. This strictly forces its norm to grow to an asymptotic order of $\Omega(d \log d)$.

        \item \textit{Theorem~\ref{the:stage2-Q}:} We construct a distinct target signal $V^\star$ specifically tailored for the self-attention architecture. Rather than relying on proof by contradiction merely to bound the global loss~\cite{li2019towards}, we leverage it to establish the strict contraction of the signal weight towards this target, guaranteeing its norm scales to $\Omega(d^{9/8}\log d)$. Crucially, we introduce \textit{a novel structural analysis} of signal weight and gradient accumulation. This yields the upper bound of specific \textit{surrogate loss} for component $\mathcal{Q}$.
        
        \item \textit{Theorem~\ref{the:stage2-P}:} Consistent with our approach in Theorem~\ref{the:stage1-P} and Theorem~\ref{the:stage2-Q}, we provide a precise \textit{structural decomposition of the gradient dynamics} unaddressed by Li et al.~\cite{li2019towards}. Specifically, we theoretically prove that the signal weight follows a structured evolution $\overline{W}_{t_1 + t_2} = \alpha \overline{W}_{t_1} + \mu^\prime w^\star (w^\star)^\top + E^\prime$. This reveals that the parameter actively reinforces the target signal direction, offering a fine-grained mechanistic guarantee of knowledge preservation.
    \end{itemize}

    \item \textit{New observations of spectral characteristics in attention weights.}
    
    We derive Corollary \ref{coro:spectral-characteristics} and have empirical observations (Appendix~\ref{sec:app-exp-counterfact}) that relatively small eigenvalues of attention weights store elementary knowledge and large ones store specialized knowledge. Collectively, we find that the two-stage learning process is closely related to the spectral characteristics in attention weights.
\end{enumerate}

\section{Discussions}\label{app:discussion}
We highlight that our core contribution lies in rigorous theoretical derivations and finite-time optimization analysis under a disentangled-feature surrogate model. 
Nevertheless, in this section, we discuss the intuition behind the feature disentanglement framework (using examples from natural language and protein structures) and the theoretical abstraction.

\textit{Feature Disentanglement Framework.}\quad
Our theoretical framework is motivated by the conceptual abstraction of latent factors, where elementary knowledge and specialized knowledge are defined as two orthogonal factors.
\begin{enumerate}[label=\textit{(\alph*)},itemsep=0em, parsep=0em, topsep=0em, partopsep=0em, leftmargin=*]
    \item We take language as an example. Language has two latent factors: syntax and semantics. While a model can use syntactic information to help learn semantics, this is not always reliable. Therefore, a part of semantics (like world knowledge) is independent of syntax. In our theoretical framework, we model syntax as elementary knowledge, \emph{i.e.}, factor $A$. We model the part of semantics that is independent of syntax as specialized knowledge, \emph{i.e.}, factor $\widetilde{A}$ that is orthogonal to $A$.

    \item Similarly, in proteins, there are also latent factors like primary and secondary structure. The primary structure could partly predict the secondary structure, but this prediction is not totally precise. Therefore, there might still be another latent factor that influences the secondary structure, with the help of the primary structure. This relationship can be expressed as, latent factor $A$ (primary structure) $+$ latent factor $\widetilde{A}$ $\to$ latent factor $B$ (secondary structure). In our theoretical framework, we model the primary structure as elementary knowledge, \emph{i.e.}, factor $A$. We model the part of secondary structure not determined by the primary structure as specialized knowledge, \emph{i.e.}, factor $\widetilde{A}$ that is orthogonal to $A$.
\end{enumerate}

\textit{Theoretical Abstraction.}\quad Our theoretical abstraction is inspired by the motivating experiment in Section \ref{sec:introduction}, but with modifications for theoretical simplicity. For the feature structure, the core insights are as follows:
\begin{enumerate}[label=\textit{(\alph*)},itemsep=0em, parsep=0em, topsep=0em, partopsep=0em, leftmargin=*]
    \item The training feature $X$ contains information for both syntax $X_1$ and semantics $X_2$, and the corresponding prediction contains both the syntax part $Y_1$ and the semantics part $Y_2$. To model $(X_1, X_2)$ as orthogonal factors, we assume that they are separable in the data construction to simplify the theoretical derivations.
    
    \item The model architecture (block-diagonal weight) is designed to process these orthogonal factors independently: it uses syntax feature $X_1$ to predict the syntax $Y_1$, and uses semantics feature $X_2$ to predict the semantics $Y_2$. The design guarantees that the network output is linearly decomposable. While the model is jointly optimized under the empirical loss, this structural disentanglement allows us to construct virtual surrogate losses to separately characterize the learning processes of each component.

    \item For syntax features $X_1$, slightly changing the feature might not significantly change the syntax prediction $Y_1$. For example, in a cloze test, slightly changing the title ($X_1$) hardly changes the syntax requirements ($Y_1$). We use a linear separable structure to model this stable phenomenon. The specific form of the linear structure is designed for theoretical simplicity.

    \item For semantics features $X_2$, slightly changing the feature would significantly change the semantics prediction $Y_2$. For example, in a cloze test, slightly changing the title ($X_2$) might completely change the semantics ($Y_2$). We use a non-linear structure to model this phenomenon.

    \item We further simplify the theoretical derivation by considering a binary classification task $Y_1, Y_2 \in \{-1,1\}^L$ without loss of generality. Besides, we set $Y_1=Y_2$ in the main derivation. We note that the above requirements are adopted to ensure mathematical tractability and to maintain the clarity of our main theoretical derivations.

    \item We finally remark that $X_1$ and $X_2$ are not restricted to the syntax and semantics regimes. Many realistic datasets contain such disentangled data structure, \emph{e.g.}, proteins contain both primary and secondary structures.
\end{enumerate}

\section{Empirical Observations}\label{sec:app-exp}
This section provides empirical observations to complement our theoretical findings on both synthetic (Appendix~\ref{sec:app-exp-synthetic}) and real-world datasets, including Counterfact and HotpotQA (Appendix \ref{sec:app-exp-counterfact}). In these settings, we observe the distinct elementary and specialized learning stages and explore the spectral characteristics of attention weights. Our code is available at \hyperlink{https://github.com/zx-gong/Two-Stage-Dynamics}{https://github.com/zx-gong/Two-Stage-Dynamics}.

\subsection{Two-Stage Learning on Synthetic Dataset}\label{sec:app-exp-synthetic}
\begin{wrapfigure}{r}{6.5cm}
    \centering
    \includegraphics[width=0.8\linewidth,height=5cm]{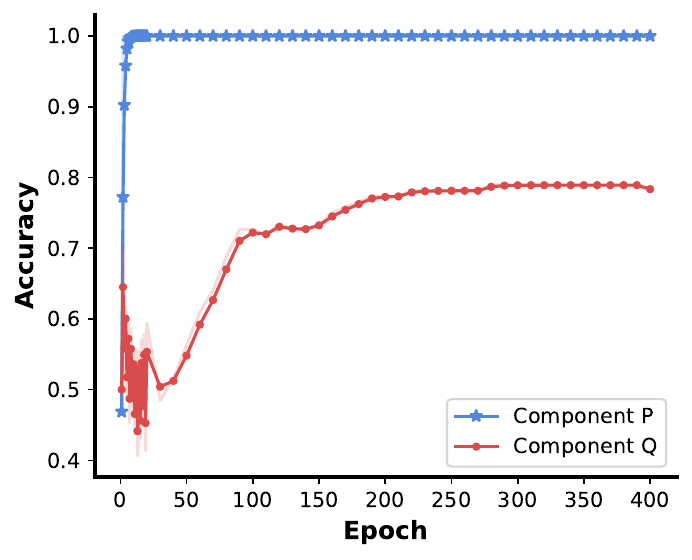}
    \caption{Two-stage Learning of Component $\mathcal{P}$ and $\mathcal{Q}$ on Theoretical Synthetic Data.} %
    \vspace{-3em}
    \label{fig:exp3-synthetic-data}
\end{wrapfigure}
We conduct experiments on a synthetic dataset constructed according to our theoretical framework (Section \ref{sec:problem-setup}), which defines components $\mathcal{P}$ and $\mathcal{Q}$.
The experimental setup uses the following hyperparameters: data dimension $d=10$, $r \triangleq \|\zeta\|_2 = {10}^{-7}$, $u \triangleq \|z\|_2 = 7$, prompt length $L=128$ and $N=128$ training prompts. 
We train a one-layer normalized ReLU self-attention model for $400$ epochs, using the SGD optimizer. The learning rate is annealed from $1.5$ to $0.015$ at epoch $20$. Note that this model is handwritten to precisely match our theoretical setting.
For this binary classification tasks, we measure performance using standard accuracy and optimize the model with the binary cross-entropy (BCE) loss function. 
All experiments are conducted on a single 24GB NVIDIA GeForce RTX 3090 GPU.
As illustrated in Figure \ref{fig:exp3-synthetic-data}, the training dynamics of linear separable component $\mathcal{P}$ and nonlinear separable component $\mathcal{Q}$ exhibit a distinct two-stage phenomenon. This empirical observation closely aligns with our theoretical results.

\subsection{Counterfact and HotpotQA Datasets}\label{sec:app-exp-counterfact}
To qualitatively illustrate our theoretical framework on a real-world task, we conduct experiments on the CounterFact and HotpotQA question-answering datasets.
These experiments are designed to observe: (a) the two-stage learning phenomenon where the model learns elementary knowledge (like syntactic information) before specialized knowledge (like semantic information); and (b) the spectral characteristics of attention weights.

\textit{Datasets and Experimental Setup.}\quad 
The Counterfact dataset~\cite{meng2022locating} consists of knowledge tuples in the form of (subject, relation, answer), which are constructed using entities from Wikidata. By utilizing three paraphrased prompts for each question based on the \texttt{paraphrase\_prompts} field in the original dataset (\emph{e.g.}, `What is the twin city of Wellington? It is', `... The twin city of Wellington is', and `For example, Brocade calls this an Inter-Chassis Link. The twin city of Wellington is'), we create a total of $65,757$ examples.
For this open-ended generation task, we fine-tune the GPT-2 model for $200$ epochs using the AdamW optimizer, with a batch size of $32$ and a learning rate of 5e-6. 
The HotpotQA dataset~\cite{meng2022locating} is a multi-hop question-answering benchmark containing $13,530$ examples. For this experiment, we fine-tune the GPT-2 model for $60$ epochs using the AdamW optimizer, with a batch size of $32$ and a learning rate of 5e-6.
Our implementation is built in PyTorch using the HuggingFace library, with parts of the code adapted from Sharma et al.~\cite{sharma2023truth}. All experiments are conducted on a single 24GB NVIDIA GeForce RTX 3090.

\textit{Evaluation Metrics.}\quad
The labels of the question-answering task are open-ended. Typically, we utilize the GPT-2 model to generate predicted answers via repeat sampling, with the temperature parameter set to $0$. 
We explicitly use \textit{grammatical correctness} on this benchmark as a practical proxy for syntactic knowledge, and use \textit{answer correctness} as a proxy for semantic knowledge. 
For the accuracy computation, a prediction is considered correct if the generated text contains the ground-truth answer, which we constrain to be a single token.
Before comparison, both texts are converted to lowercase and stripped of whitespaces. 
The loss is computed using the cross-entropy function between the output logits of the final time step and the true labels, averaged across all datapoints.

\begin{figure}%
    \centering
    \includegraphics[width=0.9\linewidth]{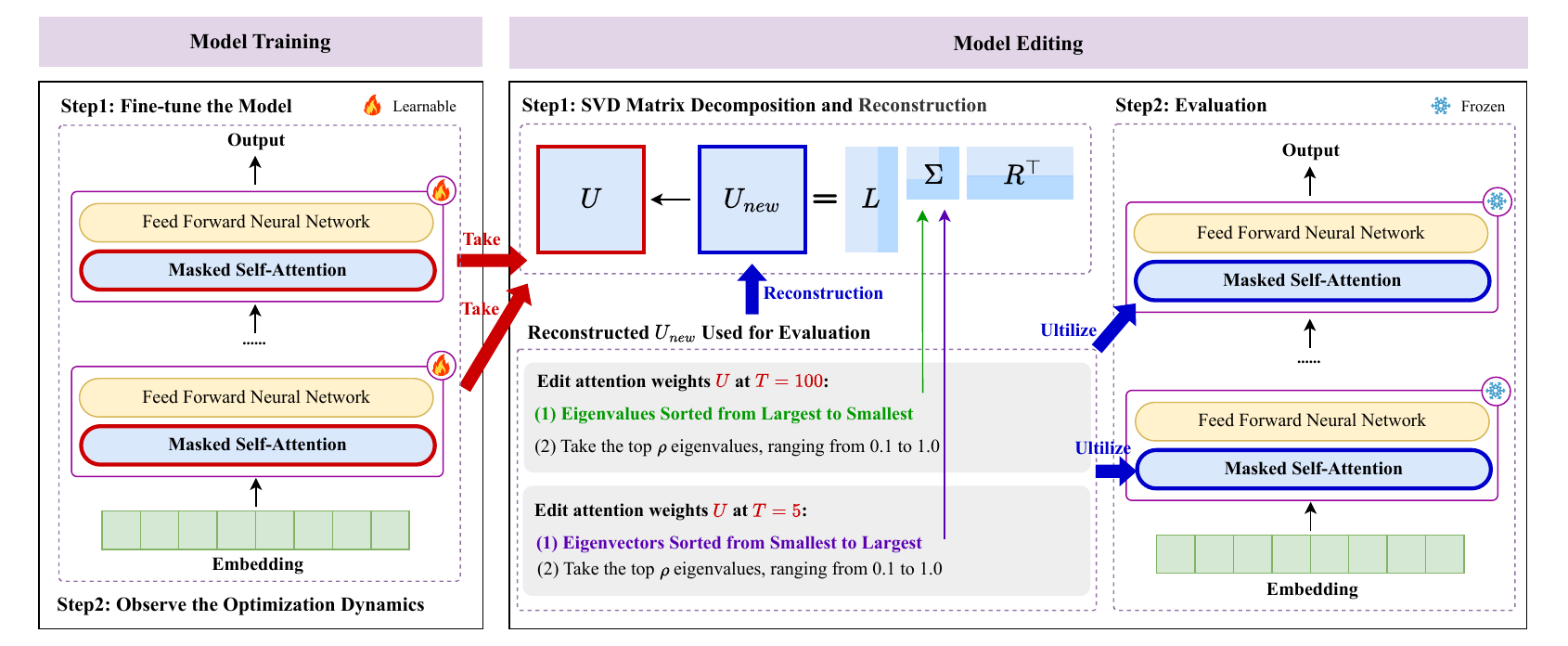}
    \caption{\textbf{Workflow of Experiments.} Taking the Counterfact dataset as an example, the workflow consists of two parts. \textbf{Left Part:} Two-stage learning of disentangled two-type feature structure, through \textit{Model Training} on language datasets. \textbf{Right Part:} Spectral characteristics, through \textit{Model Editing} on the trained models at two moments $T=100$ and $T=5$.}
    \label{fig:exp-overview}
\end{figure}

In the following, we take the Counterfact dataset as an example (results in Figures~\ref{fig:exp-two-stage}, \ref{fig:spectral} and Table~\ref{tab:syntactic_correctness_counterfact}), and make similar observations on HotpotQA dataset (results in Figures~\ref{fig:exp-two-stage-hotpot}$\sim$\ref{fig:exp-spectral-hotpot} and Table~\ref{tab:syntactic_correctness_hotpotqa}).

\textit{Experimental Workflow.}\quad 
As illustrated in Figure \ref{fig:exp-overview}, our experimental workflow is divided into two parts. 
The left part is \textbf{\textit{Model Training}}, which describes the process of verifying two-stage learning of disentangled two-type feature structure. Concretely, we fine-tune the GPT-2 model and observe its optimization dynamics. The right part is \textbf{\textit{Model Editing}}, which describes the process of verifying spectral characteristics. Concretely, we perform Singular Value Decomposition (SVD) on all attention weights across the lower-to-middle layers (layers $1$ to $6$) of the trained model at two key moments $T=100$ and $T=5$ in the left part. We strategically focus on this region as many studies establish it as the primary processing center for elementary syntax and shallow factual semantics. This provides an ideal testbed to demonstrate how these mixed features are inherently disentangled across the eigenvalue spectrum, validating our theory. For example, we edit attention weights $U$ at $T=100$. Using SVD, we sort the eigenvalues from largest to smallest and take the top $\rho$ proportion of the eigenvalues where $\rho$ ranges from $0.1$ to $1.0$. Finally, we utilize the reconstructed $U_{\text{new}}$ to replace the corresponding attention weights, and then make evaluations.

\textit{Observation (1): Two-stage Learning.}\quad Figure \ref{fig:exp-two-stage} presents the training loss over $200$ epochs and highlights three key moments with representative samples (questions, gold answers, and model predictions). At the initial time ($T=1$), many predictions are both syntactically and semantically incorrect. At the intermediate time ($T=5$), we observe a significant decrease in training loss; all predictions meet syntactic requirements (a quantitative analysis of syntactic correctness is provided in Table \ref{tab:syntactic_correctness_counterfact}), but most remain semantically incorrect and inconsistent with the true answers. Thus, the period from $T=1$ to $T=5$ is similar to our theoretical Elementary Stage. At the convergence time ($T=100$), all predictions are syntactically correct, with most being semantically correct and achieving a low loss value. Therefore, the period from $T=6$ to $T=100$ is similar to the Specialized Stage. Overall, this experiment on a language dataset illustrates the two-stage learning process for the two-type feature structure, \emph{i.e.}, syntax and semantics in languages.

\textit{Quantitative Analysis of Syntactic Correctness.}\quad
To formally demonstrate the rapid mastery of syntactic information during the initial training epochs, we conduct a detailed quantitative analysis. A prediction is classified as syntactically correct if it satisfies the fundamental rules of English grammar. To automate this process, we utilize the Python library \texttt{spaCy} and its \texttt{en\_core\_web\_sm} model to perform part-of-speech (POS) tagging. This method allows us to verify fundamental grammatical structures, like the proper use of articles and prepositions. As summarized in Table~\ref{tab:syntactic_correctness_counterfact}, there is a rapid improvement in syntactic correctness. Specifically, the number of incorrect predictions dropped from $1,630$ at epoch $T=1$ to $9$ at epoch $T=2$, and to zero by epoch $T=5$. This improvement supports an elementary stage.

\begin{figure}[t]
    \centering
    \includegraphics[width=0.86\linewidth]{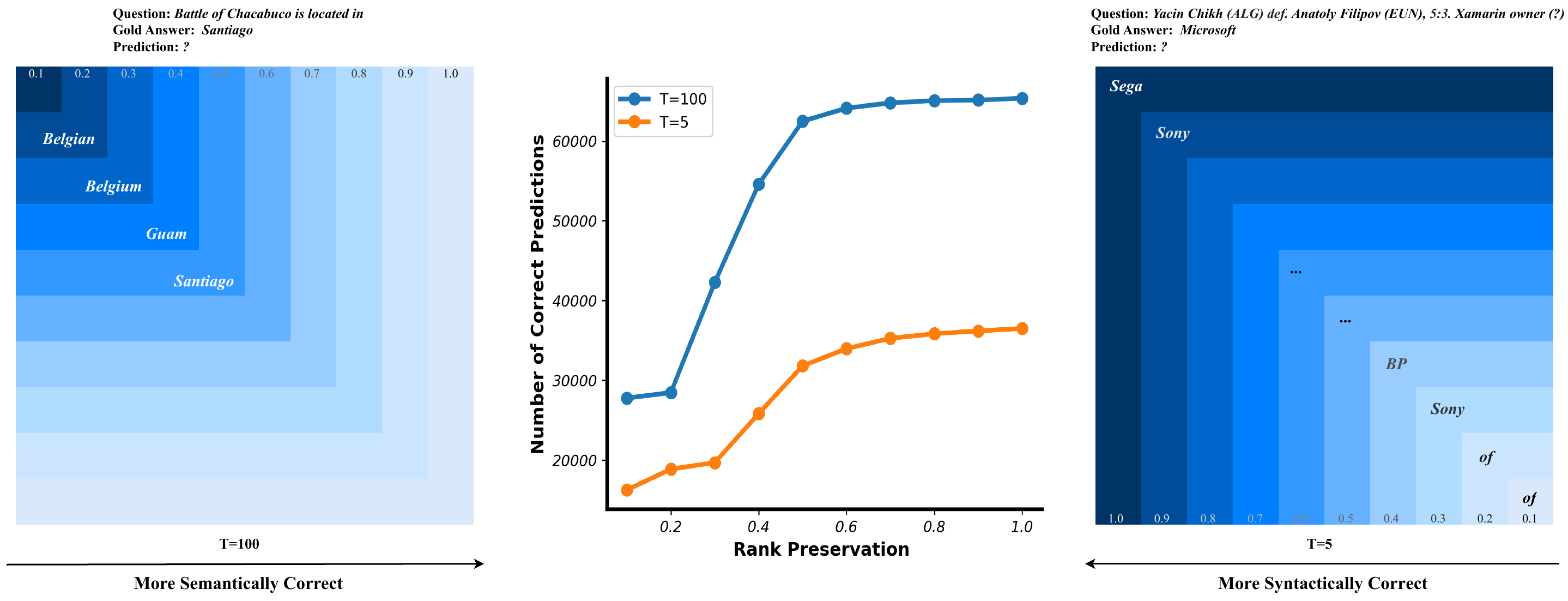}
    \caption{Spectral Characteristics in Attention Weights on Counterfact Dataset.}
    \label{fig:spectral}
\end{figure}

\begin{table}[t]
    \footnotesize
    \centering
    \setlength{\heavyrulewidth}{2pt}
    \setlength{\lightrulewidth}{1pt}
    
    \begin{minipage}[t]{0.48\textwidth}
        \centering
        \caption{Quantitative analysis of syntactic correctness on CounterFact during the initial training epochs.}
        \label{tab:syntactic_correctness_counterfact}
        \begin{tabular}{ccc}
        \toprule
        \textbf{Epoch} & \textbf{\#Syntactically Incorrect} & \textbf{\#Syntactically Correct} \\ \midrule
        T=1  & 1,630 & 64,127 \\
        T=2  & 9 & 65,748 \\
        T=3  & 2 & 65,755 \\
        T=4  & 1 & 65,756 \\
        T=5  & 0 & 65,757 \\
        T=6  & 0 & 65,757 \\
        T=7  & 0 & 65,757 \\
        T=8  & 0 & 65,757 \\
        T=9  & 0 & 65,757 \\
        T=10  & 0 & 65,757 \\
        \vdots & \vdots & \vdots \\ \bottomrule
        \end{tabular}
    \end{minipage}\hfill
    \begin{minipage}[t]{0.48\textwidth}
        \centering
        \caption{Quantitative analysis of syntactic correctness on HotpotQA during the initial training epochs.}
        \label{tab:syntactic_correctness_hotpotqa}
        \begin{tabular}{ccc}
        \toprule
        \textbf{Epoch} & \textbf{\#Syntactically Incorrect} & \textbf{\#Syntactically Correct} \\
        \midrule
        T=1  & 219 & 13,311 \\
        T=2  & 29  & 13,501 \\
        T=3  & 19  & 13,511 \\
        T=4  & 15  & 13,515 \\
        T=5  & 7   & 13,523 \\
        T=6  & 7   & 13,523 \\
        T=7  & 5   & 13,525 \\
        T=8  & 0   & 13,530 \\
        T=9  & 0   & 13,530 \\
        T=10 & 0   & 13,530 \\
        \vdots & \vdots & \vdots \\
        \bottomrule
        \end{tabular}
    \end{minipage}
\end{table}

\textit{Observation (2): Spectral Characteristics.}\quad
We empirically explore the insight from Corollary \ref{coro:spectral-characteristics} on the Counterfact dataset by observing model performances after preserving different eigenvalues.
Specifically, in Figure \ref{fig:spectral}, we perform model editing on the attention layer weights at two key moments: $T=5$ (fully syntactically correct) and $T=100$ (fully syntactically correct, nearly fully semantically correct). Using SVD, we sort the eigenvalues of attention weights and reconstruct the matrices using a rank preservation coefficient $\rho$, which ranges from $0.1$ to $1.0$. 
\begin{itemize}[leftmargin=*]
    \item \textit{Left Figure:} We edit attention weights from the semantically converged model at $T=100$. Eigenvalues are sorted from largest to smallest and we reconstruct the weight matrices preserving the top $\rho$ proportion of the largest eigenvalues. For instance, $\rho=0.1$ means maintaining the top $10\%$ of largest eigenvalues and their corresponding eigenvectors. 
    From left to right, the figure displays $10$ weight matrices, with $\rho$ ranging from $0.1$ to $1.0$. As $\rho$ increases, more of the large eigenvalues are preserved, and the model's predictions become more semantically accurate. This result suggests that large eigenvalues store semantic information (specialized knowledge).

    \item \textit{Right Figure:} We edit attention weights from the syntactically converged model at $T=5$. Eigenvalues are sorted from smallest to largest and we reconstruct the weight matrices preserving the top $\rho$ proportion of the smallest eigenvalues. As $\rho$ increases, more of the small eigenvalues are incorporated, and the model gradually grasps correct syntactic information. This result indicates that small eigenvalues store syntactic information (elementary knowledge).
    
    \item \textit{Middle Figure:} The middle figure shows that the number of correct predictions intuitively increases with higher rank preservation. In summary, the spectral characteristics insights drawn from our theory are empirically reasonable.
\end{itemize}

\begin{figure}
    \centering
    \includegraphics[width=0.86\linewidth]{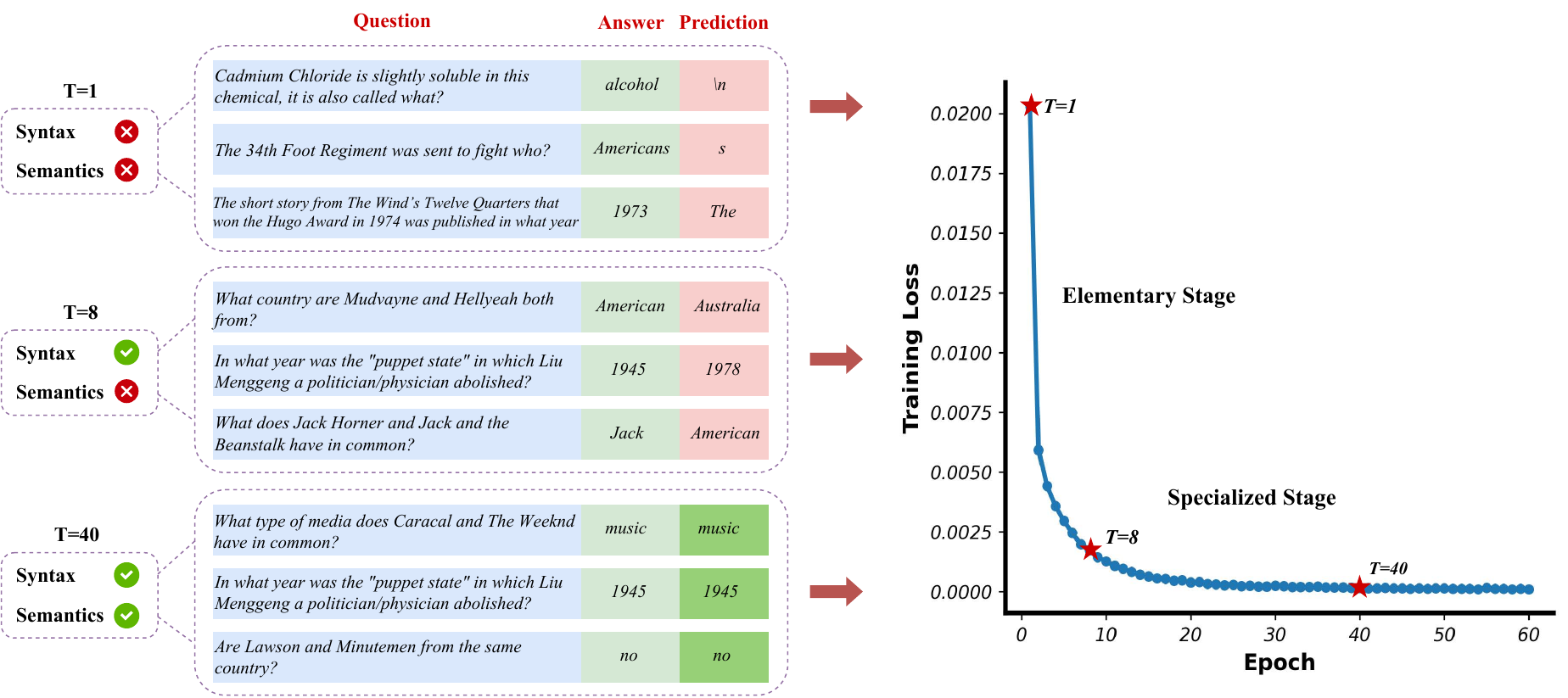}
    \caption{Two-stage Learning of Syntactic and Semantic Information on HotpotQA Dataset.}
    \label{fig:exp-two-stage-hotpot}
\end{figure}

\begin{figure}[t]
    \centering
    \includegraphics[width=0.86\linewidth]{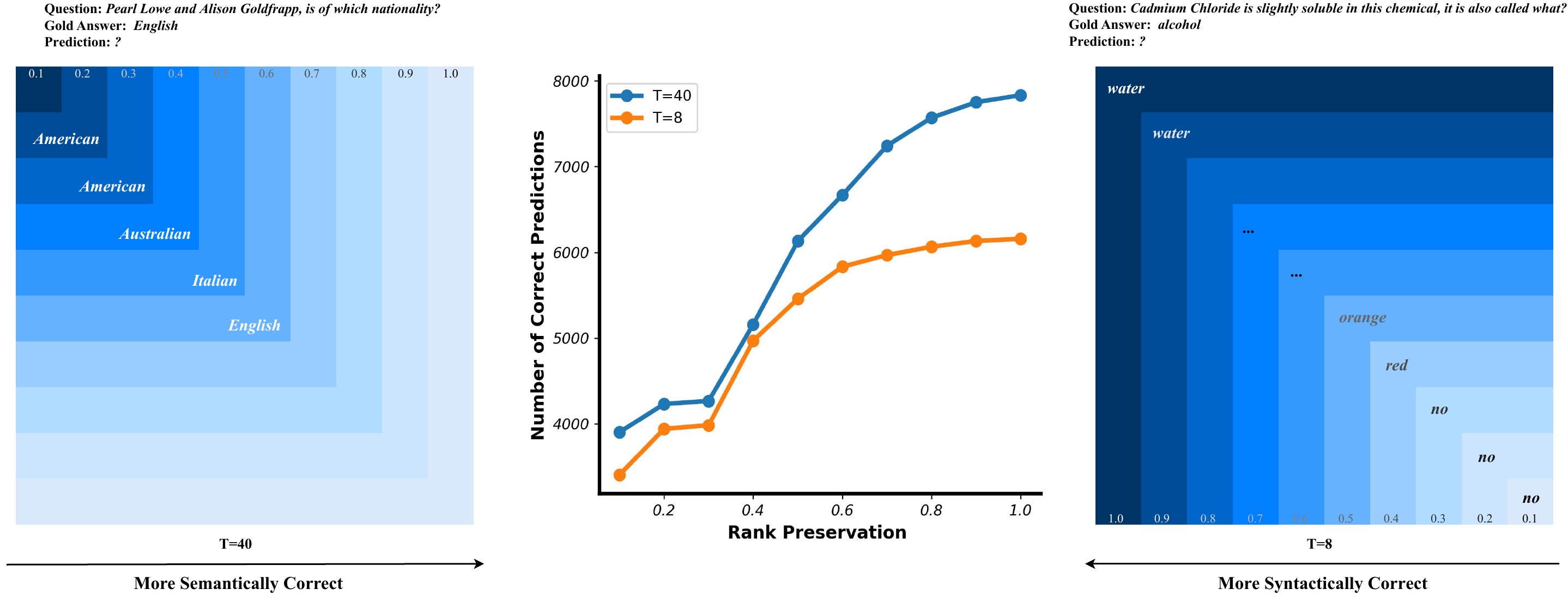}
    \caption{Spectral Characteristics of Attention Weights on HotpotQA Dataset.}
    \label{fig:exp-spectral-hotpot}
\end{figure}

\section{Proofs for Theorems and Corollary}
In this section, we present detailed proofs for the Theorems and Corollary discussed in Section \ref{sec:main-theorems}. Prior to the proofs, we first introduce useful probability concentration inequalities (in Section \ref{app:prob}), followed by some propositions, lemmas, and corollaries (in Section \ref{sec:appendix-key-prop-lemma-coro}).
The proofs of Theorem \ref{the:stage1-Q} and \ref{the:stage1-P} are provided in Section \ref{proof:stage1-Q} and \ref{proof:stage1-P}, respectively,
while the proofs of Theorem \ref{the:stage2-Q} and \ref{the:stage2-P} are provided in Section \ref{proof:stage2-Q} and \ref{proof:stage2-P}.
Finally, we discuss Corollary \ref{coro:spectral-characteristics} with its proof directly derived from the main theorems.

\subsection{Useful Probability Concentration Inequalities}\label{app:prob}
\begin{lemma}[Hoeffding's Inequality]\label{lemma:Hoeffding}
	Let $X_1, \cdots, X_N$ be independent random variables. Assume that $X_i \in [m_i, M_i]$ for every $i$. For any $t>0$, we have $\Pr \left(\sum_{i=1}^N \left(X_i - \mathbb{E}[X_i]\right) \geq t\right) \leq \exp\left(-2t^2/\sum_{i=1}^N (M_i-m_i)^2\right).$
\end{lemma}

\begin{lemma}[Bernstein's Inequality for Bounded Random Variables]\label{lemma:Bernstein}
	Let $X_1, \cdots, X_N$ be i.i.d. and $|X_i| \leq c, \mathbb{E}(X_i) = \mu, \sigma^2 = \frac{1}{N}\sum_{i=1}^N \text{Var}(X_i)$. With probability at least $1-\delta$, $\left|\frac{1}{N}\sum_{i=1}^N X_i - \mu\right| \leq\sqrt{2\sigma^2 \log(1/\delta)/N} + 2c\log(1/\delta)/(3N).$
\end{lemma}

\begin{lemma}[Bernstein's Inequality]\label{lemma:Bernstein-sub-exp}
	Let $X_1, \cdots, X_N$ be \textit{i.i.d.} mean-zero sub-exponential random variables. For any $t\geq 0$ and constant $c>0$, we have $\Pr\left(|\sum_{i=1}^N X_i| \geq t\right) \leq 2 \exp\left(-c \min (\frac{t^2}{\sum_{i=1}^N \|X_i\|^2_{\Psi_1}}, \frac{t}{\max_i \|X_i\|_{\Psi_1}})\right)$.
\end{lemma}

\begin{lemma}[McDiarmid's Inequality]\label{lemma:McDiarmid}
	Let $X_1, \cdots, X_N$ be independent random variables. Let
    $f:\mathbb{R}^n\rightarrow\mathbb{R}$ be a measurable function. Assume that the value of $f(x)$ can change by at most $c_i>0$ under an arbitrary change of a single coordinate of $x\in \mathbb{R}^n$. Then, for any $t>0$, we have $\Pr \left(f(X) - \mathbb{E}[f(X)] \geq t\right) \leq \exp(-2t^2/\sum_{i=1}^N c_i^2),$
    where $X=(X_1, \cdots, X_n)$. 
\end{lemma}

\begin{lemma}[Norm of Matrix with Gaussian Entries]\label{lemma:gaussian-norm}
	Let $A$ be an $n \times n$ random matrix whose entries $A_{ij}$ are independent gaussian random variables with $N(0,\sigma^2)$. Then for any $t>0$, we have $\|A\| \lesssim \sigma \sqrt{n}.$
\end{lemma}

\begin{lemma}[Standard Gaussian Concentration Inequality]\label{lemma:standard-gaussian-concentration}
	Suppose that $X = X_1, \cdots, X_N$ are i.i.d. standard complex Gaussian variables, and suppose $F:\mathbb{C}^n \rightarrow \mathbb{R}$ is a $1$-Lipschitz function with respect to the Euclidean metric. Then $\mathbb{E}[X] < \infty$ and for all $t \geq 0$, $\Pr\left(X-\mathbb{E}[X] > t \right) \leq e^{-t^2}.$
\end{lemma}

\begin{lemma}[Chernoff Bound for Guassian Variables]\label{lemma:chernoff-gaussian}
	Let $X \sim \mathcal{N}(\mu, \sigma^2)$, then $\mathbb{E}[e^{\lambda X} = \exp\left(\mu \lambda + \sigma^2 \lambda^2/2\right)$ and for all $t \geq 0$, $\Pr\left(|X-\mu| > t \right) \leq 2\exp\left(-t^2/(2\sigma^2)\right).$
\end{lemma}

\subsection{Propositions, Lemmas and Corollaries}\label{sec:appendix-key-prop-lemma-coro}
\begin{assumption}\label{ass:x-norm}
	For $X_1, X_2 \in \mathbb{R}^{d \times L}$ that satisfies the data structure, let $i$ be $i$-th row, with $m=\max\{1+\gamma_0, u+r\}$, we have
	\begin{align*}
		&\|[X_1^\top]_i\|_2 \leq 1+\gamma_0, \|X_1^\top\|_F \leq \sqrt{L}(1+\gamma_0). \\
		&\|[X_2^\top]_i\|_2 \leq u+r, \|X_2^\top\|_F \leq \sqrt{L}(u+r). \\
		&\|[X^\top]_i\|_2 \leq \sqrt{2}m, \|X^\top\|_F \leq 2m\sqrt{L}.
	\end{align*}
\end{assumption}
\begin{proof}
	By the definition of $X_1$, we apply the triangle inequality and get $\|x_{i,1}^n \|_2 \le |\gamma_0| \|w^\star\|_2 + \|e_i\|_2 = \gamma_0 + \|e_i\|_2.$
    Since $e_i \sim \mathcal{N}\left(0, \frac{I_{d\times d}}{d}\right)$, the expected squared norm is $\mathbb{E}[\|e_i\|_2^2] = d \cdot \frac{1}{d} = 1$. In high dimensions (large $d$), the norm of this isotropic Gaussian vector strongly concentrates around $1$ with high probability (w.h.p.), i.e., $\|e_i\|_2 \le 1$. Then $\|X_1^\top\|_F = \sqrt{\sum_{i=1}^L \|[X_1^\top]_i\|_2^2} \leq \sqrt{L}(1+\gamma_0).$
    Thus, $\|[X^\top_1]_i\|_2 \leq 1+\gamma_0, \|X^\top_1\|_F \leq \sqrt{L}(1+\gamma_0).$

	By the definition of $X_2$, we get $\|x^n_{i,2}\|_2 \leq u+r$, then $\|X_2^\top\|_F = \sqrt{\sum_{i=1}^L \|[X_2^\top]_i\|_2^2} \leq \sqrt{L}(u+r).$
    Thus, $\|[X^\top_2]_i\|_2 \leq u+r, \|X^\top_2\|_F \leq \sqrt{L}(u+r).$
    
    For $X$, we have $ \|[X^\top]_i\|_2 = \sqrt{\|x_{i,1}^n\|_2^2 + \|x_{i,2}^n\|_2^2}
    \leq \sqrt{(1+\gamma_0)^2 +(u+r)^2}.$
    Let $m=\max\{1+\gamma_0, u+r\}$, $\|[X^\top]_i\|_2 \leq \sqrt{2}m, \|X^\top\|_F = \sqrt{\sum_{i=1}^{2L} \|[X^\top]_i\|_2^2} \leq 2m\sqrt{L}.$
\end{proof}

\begin{proposition}\label{prop:U-signal-noise}
    By signal-noise decomposition, we have the updating rules for signal weight and noise weight:
    \begin{align*}
    	\overline{U}_t = -\sum_{s=1}^t \eta\left(1-\eta \lambda\right)^{t-s} \nabla_{U_{s-1}}\widehat{L}(U_{s-1}), \quad
    	\widetilde{U}_t = \left(1-\eta \lambda\right)^t U_0 - \sum_{s=1}^t \eta\left(1-\eta \lambda\right)^{t-s} \xi_{s-1}.
    \end{align*}
\end{proposition}
\begin{proof}
    Decoupling the signal and noise, signal weight $\overline{U}$ is affected by the gradient updates, and noise weight $\widetilde{U}$ is affected by noise $\xi$. With $U_{t+1} = (1-\gamma_t \lambda)U_t - \gamma_t(\nabla_U \widehat{L}(U_t)+\xi_t )$,
    \begin{align*}
    	\overline{U}_t = -\sum_{s=1}^t \gamma_{s-1}\left(\prod_{i=s}^{t-1}(1-\gamma_i \lambda)\right) \nabla_{U_{s-1}}\widehat{L}(U_{s-1}), \quad
    	\widetilde{U}_t = \left(\prod_{i=0}^{t-1}(1-\gamma_i \lambda)\right)U_0 - \sum_{s=1}^t \gamma_{s-1} \left(\prod_{i=s}^{t-1}(1-\gamma_i \lambda)\right) \xi_{s-1}.
    \end{align*}
    When constant learning rate $\gamma_t = \eta$,
    \begin{align}
    	\overline{U}_t = -\sum_{s=1}^t \eta\left(1-\eta \lambda\right)^{t-s} \nabla_{U_{s-1}}\widehat{L}(U_{s-1}), \quad
    	\widetilde{U}_t = \left(1-\eta \lambda\right)^t U_0 - \sum_{s=1}^t \eta\left(1-\eta \lambda\right)^{t-s} \xi_{s-1}.\label{eq:U-signal-noise}
    \end{align}
    By the block-diagonal structure of $U$, similarly to the signal-noise decomposition of $U$ with $\gamma_t = \eta$, we naturally have
    \begin{align}
    	\overline{W}_t &= -\sum_{s=1}^t \eta\left(1-\eta \lambda\right)^{t-s} \nabla_{W_{s-1}}\widehat{L}(U_{s-1}), \quad
    	\widetilde{W}_t = \left(1-\eta \lambda\right)^t W_0 - \sum_{s=1}^t \eta\left(1-\eta \lambda\right)^{t-s} \xi_{s-1}  \label{eq:W-signal-noise}, \\
    	\overline{V}_t &= -\sum_{s=1}^t \eta\left(1-\eta \lambda\right)^{t-s} \nabla_{V_{s-1}}\widehat{L}(U_{s-1}), \quad
    	\widetilde{V}_t = \left(1-\eta \lambda\right)^t V_0 - \sum_{s=1}^t \eta\left(1-\eta \lambda\right)^{t-s} \xi_{s-1}.  \label{eq:V-signal-noise}
    \end{align}
\end{proof}

\begin{proposition}\label{prop:derivative-L}
	For $U \in \mathbb{R}^{2d \times 2d}, W,V \in \mathbb{R}^{d \times d}, X \in \mathbb{R}^{2d \times 2L}, X_1,X_2 \in \mathbb{R}^{d \times L}, \widetilde{Y} \in \mathbb{R}^{1 \times 2L}, Y \in \mathbb{R}^{1 \times L}$, the gradient of the empirical loss with respect to weight $U$ is:
	\begin{align*}
		\nabla\widehat{L}(U) = \widehat{\mathbb{E}} \left[\frac{y_L}{2L} l^\prime(f(U; X, \widetilde{Y})) X \left( \widetilde{Y}^\top \odot \mathbbm{1}(X^\top U x_L) \right) x_L^\top \right].
	\end{align*}
	Given the block-diagonal structure $U = \mathrm{diag}(W, V)$, the gradients with respect to $W$ and $V$ are:
	\begin{align*}
		\nabla_W\widehat{L}(U) &= \widehat{\mathbb{E}}\left[\frac{y_L}{2L} l^\prime(f(U; X, \widetilde{Y})) X_1 \left( Y^\top \odot \mathbbm{1}(X_1^\top W x_{L,1}) \right) x_{L,1}^\top \right], \\
		\nabla_V\widehat{L}(U) &= \widehat{\mathbb{E}}\left[\frac{y_L}{2L} l^\prime(f(U; X, \widetilde{Y})) X_2 \left( Y^\top \odot \mathbbm{1}(X_2^\top V x_{L,2}) \right) x_{L,2}^\top \right].
	\end{align*}
\end{proposition}

\begin{proof}
	Let $l^\prime$ denote the derivative of the loss function with respect to the margin $y_L f(U;X,\widetilde{Y})$. Applying the chain rule and standard matrix calculus ($\nabla_U(a^\top U b) = a b^\top$), the gradient with respect to $U$ is:
	\begin{align*}
		\nabla\widehat{L}(U) &= \widehat{\mathbb{E}}\left[l^\prime(f(U; X, \widetilde{Y})) \nabla_U \left( \frac{y_L}{2L} \sum_{j=1}^{2L}\tilde{y}_j \text{ReLU}\left([X^\top]_j U x_L\right) \right)\right]   \\
		&= \widehat{\mathbb{E}}\left[\frac{y_L}{2L} l^\prime(f(U; X, \widetilde{Y})) \sum_{j=1}^{2L} \tilde{y}_j \mathbbm{1}([X^\top]_j U x_L) ([X^\top]_j)^\top x_L^\top \right]  \\
        &= \widehat{\mathbb{E}} \left[\frac{y_L}{2L} l^\prime(f(U; X, \widetilde{Y})) X \left( \widetilde{Y}^\top \odot \mathbbm{1}(X^\top U x_L) \right) x_L^\top \right].
	\end{align*}
    Extracting the $i$-th row from $X$ yields the component $[\nabla\widehat{L}(U)]_i$ directly, i.e.,
    $$
    [\nabla\widehat{L}(U)]_i = \widehat{\mathbb{E}}\left[\frac{y_L}{2L} l^\prime(f(U; X, \widetilde{Y})) \sum_{j=1}^{2L} \tilde{y}_j \mathbbm{1}([X^\top]_j U x_L) X_{i,j} x_L^\top \right]
    $$
    
    Because the weight matrix is block-diagonal and inputs are concatenated, the model output decouples into two summations:
    $$f(U; X, \widetilde{Y}) = \frac{1}{2L} \sum_{j=1}^L y_j \left( \text{ReLU}\left([X_1^\top]_j W x_{L,1}\right) + \text{ReLU}\left([X_2^\top]_j V x_{L,2}\right) \right). $$
    Taking the derivative with respect to $W$, the term involving $V$ vanishes, yielding $\nabla_W\widehat{L}(U)$ and its row-wise component $[\nabla_W\widehat{L}(U)]_i$. The derivations for $\nabla_V\widehat{L}(U)$ and $[\nabla_V\widehat{L}(U)]_i$ follow symmetrically.
\end{proof}

\begin{proposition}\label{prop:derivative-L-norm}
    Assume that the loss function $l$ is smooth, \textit{i.e.}, $|l^\prime(\cdot)| \le K_l$. With Assumption~\ref{ass:x-norm} and Proposition~\ref{prop:derivative-L}, we have
    \begin{align*}
    	&\|\nabla\widehat{L}(U)\|_F \le 2 K_l m^2, \|[\nabla\widehat{L}(U)]_i \|_2 \le  2 K_l m^2.\\
    	&\|\nabla_W\widehat{L}(U)\|_F \le \frac{1}{2} K_l (1+\gamma_0)^2, \|[\nabla_W\widehat{L}(U)]_i\|_2 \le \frac{1}{2} K_l (1+\gamma_0)^2. \\
    	&\|\nabla_V\widehat{L}(U)\|_F \le \frac{1}{2} K_l (u+r)^2, \|[\nabla_V\widehat{L}(U)]_i\|_2 \le \frac{1}{2} K_l (u+r)^2,
    \end{align*}
    where $m=\max\{1+\gamma_0, u+r\}$.
\end{proposition}
\begin{proof}
    Taking the Frobenius norm for the gradient expression in Proposition~\ref{prop:derivative-L}, using the property of outer products that $\|ab^\top\|_F = \|a\|_2 \|b\|_2$, and Assumption~\ref{ass:x-norm}, $\|([X^\top]_j)^\top\|_2 \le \sqrt{2}m$ and $\|x_L^\top\|_2 \le \sqrt{2}m$, we obtain:
    \begin{align*}
        \|\nabla\widehat{L}(U)\|_F \le \widehat{\mathbb{E}} \left[ \frac{1}{2L} |l^\prime(f)| \sum_{j=1}^{2L} \|([X^\top]_j)^\top x_L^\top\|_F \right] \le \frac{1}{2L} K_l \sum_{j=1}^{2L} (\sqrt{2}m) (\sqrt{2}m) = 2 K_l m^2.
    \end{align*}
    For the $i$-th row, since the absolute value of any scalar element is bounded by the vector's $L_2$ norm (i.e., $|X_{i,j}| \le \|([X^\top]_j)^\top\|_2 \le \sqrt{2}m$),
    \begin{align*}
        \|[\nabla\widehat{L}(U)]_i \|_2 \le \widehat{\mathbb{E}}\left[ \frac{1}{2L} |l^\prime(f)| \sum_{j=1}^{2L} |X_{i,j}| \cdot \|x_L^\top\|_2 \right] \le \frac{K_l}{2L} \sum_{j=1}^{2L} |X_{i,j}| \|x_L\|_2 \le 2 K_l m^2.
    \end{align*}

    Following the same logic for the sub-matrices $W$ and $V$, the summations are over $L$ terms, and their respective input vector norms are bounded by $1+\gamma_0$ and $u+r$. Substituting these into the steps above directly yields the corresponding bounds of $\frac{1}{2} K_l (1+\gamma_0)^2$ and $\frac{1}{2} K_l (u+r)^2$ for both their Frobenius norms and row norms.
\end{proof}

\begin{proposition}\label{prop:weight-norm}
	With Assumption \ref{ass:x-norm} and Proposition \ref{prop:derivative-L-norm}, we have that the signal weight norm and pre-activation norm satisfy
	\begin{align*}
		&\left\|\overline{U}_t\right\|_F \leq \frac{2 K_l m^2}{\lambda}, 
		\left\|[\overline{U}_t]_i\right\|_2 \leq \frac{2 K_l m^2}{\lambda},
        \left|[X^\top]_i \overline{U} x_L\right| \leq \frac{4K_l m^4}{\lambda}, \\
		&\left\|\overline{W}_t\right\|_F \leq \frac{K_l(1+\gamma_0)^2}{2\lambda}, 
		\left\|[\overline{W}_t]_i\right\|_2 \leq \frac{K_l(1+\gamma_0)^2}{2\lambda},
        \left|[X_1^\top]_i \overline{W} x_{L,1}\right| \leq \frac{K_l (1+\gamma_0)^4}{2\lambda},\\
		&\left\|\overline{V}_t\right\|_F \leq \frac{K_l (u+r)^2}{2\lambda}, 
		\left\|[\overline{V}_t]_i\right\|_2 \leq \frac{K_l (u+r)^2}{2\lambda},
        \left|[X_2^\top]_i \overline{V} x_{L,2}\right| \leq \frac{K_l (u+r)^4}{2\lambda},
	\end{align*}
    where $m=\max\{1+\gamma_0, u+r\}$.
\end{proposition}
\begin{proof}
    By Equations \ref{eq:U-signal-noise}, \ref{eq:W-signal-noise}, and \ref{eq:V-signal-noise}, when $0<1-\eta \lambda < 1$, i.e., $0< \eta \lambda < 1$,
    \begin{align*}
        \sum_{\tau=1}^t \eta(1-\eta\lambda)^{t-\tau} \le \frac{\eta}{1 - (1-\eta\lambda)} = \frac{1}{\lambda}.
    \end{align*}
    Applying the triangle inequality and Proposition \ref{prop:derivative-L-norm}, for $U$, we have
    \begin{align*}
        \left\|\overline{U}_t\right\|_F \leq \sum_{\tau=1}^t \eta\left(1-\eta \lambda\right)^{t-\tau} \|\nabla \widehat{L}(U_{\tau-1})\|_F \leq \frac{2 K_lm^2}{\lambda}.
    \end{align*}
    The bounds for $\|[\overline{U}_t]_i\|_2$ and all norms for $\overline{W}_t$ and $\overline{V}_t$, follow identically by scaling their respective gradient bounds with $\frac{1}{\lambda}$.

    Furthermore, for the pre-activation, we have
    \begin{align*}
        \left|[X^\top]_i \overline{U}_t x_L\right| \leq \|[X^\top]_i\|_2 \|\overline{U}_t\|_F \|x_L\|_2 \leq (\sqrt{2}m)\left(\frac{2K_l m^2}{\lambda}\right)(\sqrt{2}m) = \frac{4K_lm^4}{\lambda}.
    \end{align*}
    Applying the exact same logic for $W$ and $V$ with their respective input norm bounds ($1+\gamma_0$ and $u+r$), we obtain their pre-activation bounds.
\end{proof}

\begin{proposition}\label{prop:weight-tau-t}
	For time $\tau \leq t$, we have 
        $$
        \widetilde{U}_t = (1-\eta \lambda)^{t-\tau} \widetilde{U}_\tau + \Xi_{t,\tau}
        $$
        where $\Xi_{t,\tau} = - \sum_{t^\prime=1}^{t-\tau} \eta (1-\eta \lambda)^{t-\tau-t^\prime} \xi_{\tau+t^\prime-1}$.
\end{proposition}
\begin{proof}
    With Equation \ref{eq:U-signal-noise} in Proposition \ref{prop:U-signal-noise},
    \begin{align*}
        \widetilde{U}_t = \left(1-\eta \lambda\right)^t U_0 - \sum_{s=1}^t \eta\left(1-\eta \lambda\right)^{t-s} \xi_{s-1}.
    \end{align*}
    Thus, for $\tau \leq t$,
	\begin{align*}
		\widetilde{U}_t = (1-\eta \lambda)^{t-\tau} \widetilde{U}_\tau -  \sum_{t^\prime=1}^{t-\tau} \eta (1-\eta \lambda)^{t-\tau-t^\prime} \xi_{\tau+t^\prime-1}
		= (1-\eta \lambda)^{t-\tau} \widetilde{U}_\tau + \Xi_{t,\tau},
	\end{align*}
	where $\Xi_{t,\tau} = - \sum_{t^\prime=1}^{t-\tau} \eta (1-\eta \lambda)^{t-\tau-t^\prime} \xi_{\tau+t^\prime-1}$.
\end{proof}

\begin{lemma}[Refer to Lemma A.8 in Li et al.~\cite{li2019towards}, Lemma 8.2 of Allen-Zhu et al.~\cite{allen2019convergence}]\label{lemma:activation-patterns-difference}
	Let $X\in \mathbb{R}^{2d \times 2L}, x_L \in \mathbb{R}^{2d}$ be a fixed example, with $\|x_L\|_2 \leq \sqrt{2}m$ (Assumption~\ref{ass:x-norm}). With Proposition \ref{prop:weight-norm}, for every $\tau > 0$, let $U = \overline{U} + \widetilde{U}$ where $\widetilde{U} \in \mathbb{R}^{2d \times 2d}$ is a random variable whose columns have i.i.d distribution $\mathcal{N}(0,\tau_0^2 I_{2d\times 2d})$ and $\widetilde{Y} \in \mathbb{R}^{2L}$ such that each entry of $\widetilde{Y}$ is i.i.d. uniform in $\{-1,1\}$. We have that, w.h.p over the randomness of $\widetilde{U}$,  $\forall \overline{U} \in \mathbb{R}^{2d\times 2d}$, we have that 
	\begin{align*}
		\|\mathbbm{1}(X^\top U x_L) - \mathbbm{1}(X^\top\widetilde{U} x_L)\|_1 \lesssim K_l^{4/3} m^{8/3}\lambda^{-4/3} \tau_0^{-4/3}L^{2/3} \triangleq \epsilon_U.
	\end{align*}
	Furthermore,
	\begin{align*}
		\left|N_U(\overline{U}; X, \widetilde{Y}) - N_{\widetilde{U}}(\overline{U}; X, \widetilde{Y})\right| \lesssim K_l^{5/3} m^{16/3}(\log d)^{1/2}\lambda^{-5/3} \tau_0^{-2/3} L^{-2/3} .
	\end{align*}
\end{lemma}
\begin{proof}
	With Lemma A.8 of Li et al.~\cite{li2019towards} and Proposition \ref{prop:weight-norm}, we can compute the difference of activation patterns.
	\begin{align*}
		\|\mathbbm{1}(X^\top U x_L) - \mathbbm{1}(X^\top\widetilde{U} x_L)\|_1 
		&\lesssim \|\overline{U}\|_F^{4/3}\tau_0^{-4/3}(2L)^{2/3} \\
		&\lesssim  (2K_l m^2)^{4/3} \lambda^{-4/3} \tau_0^{-4/3}L^{2/3} \\
        &\lesssim K_l^{4/3} m^{8/3}\lambda^{-4/3} \tau_0^{-4/3}L^{2/3}.
	\end{align*}
    Furthermore, we use the subscript of $N$ to calculate the activation, and the weight in parentheses to calculate the model output, \emph{i.e.,} with $N_U(\overline{U}; X, \widetilde{Y}) = \widetilde{Y}/2L \cdot \mathbbm{1}(X^\top U x_L) \odot (X^\top \overline{U} x_L)$ and $N_{\widetilde{U}}(\overline{U}; X, \widetilde{Y}) = \widetilde{Y}/2L \cdot \mathbbm{1}(X^\top \widetilde{U} x_L) \odot (X^\top \overline{U} x_L)$,
	\begin{align*}
		\left|N_U(\overline{U}; X, \widetilde{Y}) - N_{\widetilde{U}}(\overline{U}; X, \widetilde{Y})\right| 
        = \left|\frac{1}{2L}\sum_{i \in [2L]} [\widetilde{Y}]_i \left(\mathbbm{1}\left([X^\top]_i U x_L\right) - \mathbbm{1}\left([X^\top]_i \widetilde{U} x_L\right)\right) [X^\top \overline{U}]_i x_L \right|
        = \frac{1}{2L}\left|\sum_{i \in [2L]} [\widetilde{Y}]_i \Delta_i \right|,
	\end{align*}
    where we denote $\Delta_i=\left(\mathbbm{1}\left([X^\top]_i U x_L\right) - \mathbbm{1}\left([X^\top]_i \widetilde{U} x_L\right)\right) [X^\top \overline{U}]_i x_L$.

    Using Hoeffding's inequality in Lemma~\ref{lemma:Hoeffding} for $[\widetilde{Y}]_i \Delta_i$, $M_i = |\Delta_i|$, $m_i=-|\Delta_i|$, with probability at least $1-\delta$ and set $\delta=1/d$: $\frac{1}{2L}\left|\sum_{i\in[2L]}[\widetilde{Y}]_i \Delta_i\right| \leq \frac{1}{2L}\sqrt{2\log(2d) \sum_{i\in[2L]}\Delta_i^2}$.
    Among it, for $\sum_{i\in[2L]}\Delta_i^2$, since $\mathbbm{1} \in \{0, 1\}$ only takes values in $\{-1, 0, 1\}$, its squared absolute value equals its absolute value, i.e., $|\mathbbm{1}_i(U) - \mathbbm{1}_i(\widetilde{U})|^2 = |\mathbbm{1}_i(U) - \mathbbm{1}_i(\widetilde{U})|$. Thus,
	\begin{align*}
		\sum_{i\in[2L]} \Delta_i^2 &= \sum_{i\in[2L]} \left(\mathbbm{1}\left([X^\top]_i U x_L\right) - \mathbbm{1}\left([X^\top]_i \widetilde{U} x_L\right)\right)^2 ([X^\top \overline{U}]_i x_L)^2 \\
        &\leq  \max_i \left([X^\top \overline{U}]_i x_L\right)^2 \sum_{i\in[2L]}\left|\mathbbm{1}\left([X^\top]_i U x_L\right) - \mathbbm{1}\left([X^\top]_i \widetilde{U} x_L\right)\right| \\
        &\leq \left(\frac{4K_l m^4}{\lambda}\right)^2 \left\|\mathbbm{1}\left(X^\top U x_L\right) - \mathbbm{1}\left(X^\top \widetilde{U} x_L\right)\right\|_1.
	\end{align*}
    Therefore,
    \begin{align*}
        \left|N_U(\overline{U}; X, \widetilde{Y}) - N_{\widetilde{U}}(\overline{U}; X, \widetilde{Y})\right|
        =&\frac{1}{2L}\left|\sum_{i\in[2L]}[\widetilde{Y}]_i \Delta_i\right|
        \lesssim \frac{1}{2L} \cdot \frac{4K_l m^4}{\lambda}\cdot  \sqrt{\log d}\cdot\left\|\mathbbm{1}\left(X^\top U x_L\right) - \mathbbm{1}\left(X^\top \widetilde{U} x_L\right)\right\|_1^{1/2} \\
        \lesssim& K_l^{5/3} m^{16/3} (\log d)^{1/2} \lambda^{-5/3} \tau_0^{-2/3} L^{-2/3}.
    \end{align*}
\end{proof}

\begin{corollary}\label{coro:lemma-ac-W}
	Let $X_1 \in \mathbb{R}^{d \times L}, x_{L,1} \in \mathbb{R}^{d}$ be a fixed example, with $\|x_{L,1}\|_2 \leq 1+\gamma_0$ (Assumption~\ref{ass:x-norm}). Similarly to Lemma~\ref{lemma:activation-patterns-difference}, then, w.h.p over the randomness of $\widetilde{W}$ and $Y$,  $\forall \overline{W} \in \mathbb{R}^{d\times d}$, we have that 
	\begin{align}\label{eq:def-epsilonW}
		\|\mathbbm{1}(X_1^\top W x_{L,1}) - \mathbbm{1}(X_1^\top\widetilde{W} x_{L,1})\|_1 \lesssim K_l^{4/3} (1+\gamma_0)^{8/3}\lambda^{-4/3} \tau_0^{-4/3}L^{2/3} \triangleq \epsilon_W.
	\end{align}
	Furthermore,
	\begin{align*}
		\left|N_W(\overline{W}; X_1, Y) - N_{\widetilde{W}}(\overline{W}; X_1, Y)\right| \lesssim K_l^{5/3}(1+\gamma_0)^{16/3}(\log d)^{1/2}\lambda^{-5/3}\tau_0^{-2/3}L^{-2/3}.
	\end{align*}
\end{corollary}

\begin{corollary}\label{coro:lemma-ac-V}
	Let $X_2 \in \mathbb{R}^{d \times L}, x_{L,2} \in \mathbb{R}^{d}$ be a fixed example, with $\|x_{L,2}\|_2 \leq u+r$ (Assumption~\ref{ass:x-norm}). Similarly to Lemma~\ref{lemma:activation-patterns-difference}, then, w.h.p over the randomness of $\widetilde{V}$ and $Y$,  $\forall \overline{V} \in \mathbb{R}^{d\times d}$, we have that 
	\begin{align*}
		\|\mathbbm{1}(X_2^\top V x_{L,2}) - \mathbbm{1}(X_2^\top\widetilde{V} x_{L,2})\|_1 \lesssim K_l^{4/3} (u+r)^{8/3}\lambda^{-4/3} \tau_0^{-4/3}L^{2/3} \triangleq \epsilon_V.
	\end{align*}
	Furthermore,
	\begin{align*}
		\left|N_V(\overline{V}; X_2, Y) - N_{\widetilde{V}}(\overline{V}; X_2, Y)\right| \lesssim K_l^{5/3}(u+r)^{16/3}(\log d)^{1/2}\lambda^{-5/3}\tau_0^{-2/3}L^{-2/3}.
	\end{align*}
\end{corollary}

\textbf{\textit{Note.}} For notation clarity in analyzing the decoupled components, we define the individual asymptotic error bounds as $\epsilon_U$, $\epsilon_W$, and $\epsilon_V$ for weights $U$, $W$, and $V$, respectively. In these expressions, $K_l$ is the Lipschitz constant $\mathcal{O}(1)$, $m, 1+\gamma_0, u+r$ are of constant order $\mathcal{O}(1)$, $\lambda$ denotes the $L_2$ regularization coefficient, $\tau_0^2$ denotes the variance of initialization parameter and $L$ is prompt length. With choices in Assumption \ref{ass:choice-hyperparam}, we have order $\Theta\left(d^{4}L^{2/3}\right)$, i.e., $\Theta\left(\text{Poly}(d)\right)$ for $\epsilon_U$, $\epsilon_W$ and $\epsilon_V$.

\begin{lemma}\label{lemma:activation-patterns-t1+t-t1}
	Under the same setting as Lemma \ref{lemma:activation-patterns-difference}, we have
	\begin{align*}
		\left\|\mathbbm{1}\left(X^\top U_{t_1 + t_2}x_L\right) - \mathbbm{1}\left(X^\top U_{t_1}x_L\right)\right\|_1 \lesssim \epsilon_U+  L\sqrt{\frac{\eta_2}{\eta_1}}  + \sqrt{L \log d},
	\end{align*}
	where $\epsilon_U = K_l^{4/3} m^{8/3}\lambda^{-4/3} \tau_0^{-4/3}L^{2/3}$, $m=\max\{1+\gamma_0, u+r\}$. Furthermore,
	\begin{align*}
		\left|N_{U_{t_1+t_2}}(\overline{U}_{t_1+t_2}; X, \widetilde{Y}) - N_{U_{t_1}}(\overline{U}_{t_1+t_2}; X, \widetilde{Y})\right|
		\lesssim \left(\epsilon_U+L\sqrt{\frac{\eta_2}{\eta_1}}+ \sqrt{L \log d}\right)^{1/2}K_l m^4 (\log d)^{1/2} \lambda^{-1} L^{-1}.
	\end{align*}
\end{lemma}

\begin{proof}
	To analyze that how the sign of $U_{t_1+t_2}$ correlates to $U_{t_1}$,
	\begin{align*}
		&\left\|\mathbbm{1}\left(X^\top U_{t_1 + t_2}x_L\right) - \mathbbm{1}\left(X^\top U_{t_1}x_L\right)\right\|_1 \\
		\leq& \underbrace{\left\|\mathbbm{1}\left(X^\top U_{t_1 + t_2}x_L\right) - \mathbbm{1}\left(X^\top \widetilde{U}_{t_1 + t_2} x_L\right)\right\|_1}_{A} +\underbrace{\left\| \mathbbm{1}\left(X^\top \widetilde{U}_{t_1 + t_2} x_L\right) - \mathbbm{1}\left(X^\top \widetilde{U}_{t_1}x_L\right)\right\|_1}_{B} \\
        &+\underbrace{\left\|\mathbbm{1}\left(X^\top \widetilde{U}_{t_1}x_L\right)- \mathbbm{1}\left(X^\top U_{t_1}x_L\right)\right\|_1}_{C}.
	\end{align*}
	By Lemma \ref{lemma:activation-patterns-difference}, terms $A$ and $C$ are each bounded by $\epsilon_U = K_l^{4/3} m^{8/3}\lambda^{-4/3} \tau_0^{-4/3}L^{2/3}$ with high probability.
    
	For term $B$, we first analyze the relationship between $\widetilde{U}_{t_1 + t_2}$ and $\widetilde{U}_{t_1}$. With Proposition \ref{prop:weight-tau-t}, for $\tau \leq t$, we have
	\begin{align*}
		\widetilde{U}_t = (1-\eta \lambda)^{t-\tau} \widetilde{U}_\tau -  \sum_{t^\prime=1}^{t-\tau} \eta (1-\eta \lambda)^{t-\tau-t^\prime} \xi_{\tau+t^\prime-1}
		= (1-\eta \lambda)^{t-\tau} \widetilde{U}_\tau + \Xi_{t,\tau},
	\end{align*}
	where $\Xi_{t,\tau} = - \sum_{t^\prime=1}^{t-\tau} \eta (1-\eta \lambda)^{t-\tau-t^\prime} \xi_{\tau+t^\prime-1}$. 
    
    Assume that there are $t_1$ iterations in the first stage.
    Let $\tau=0$, $t=t_1$, and $t-\tau=t_1$, then $\widetilde{U}_{t_1}=(1-\eta_1\lambda)^{t_1} \widetilde{U}_{0} + \Xi_{t_1, 0},$ where $\Xi_{t_1, 0} = - \sum_{t^\prime=1}^{t_1} \eta_1(1-\eta_1\lambda)^{t_1-t^\prime} \xi_{t^\prime-1}$. For the variance $\text{Var}(\widetilde{U}_{t_1})$, with a technical assumption that $\tau^2_\xi = \frac{\tau_0^2-(1-\eta_1\lambda)^2\tau_0^2}{\eta_1^2}$,
    \begin{align*}
        \text{Var}(\widetilde{U}_{t_1})
        =(1-\eta_1\lambda)^{2 t_1}  \text{Var}(\widetilde{U}_{0}) + \text{Var}(\Xi_{t_1, 0})
        =\tau_0^2.
    \end{align*}
    Assume that there are $t_2$ iterations in the second stage. Let $\tau=t_1, t=t_1+t_2$, and $t-\tau = t_2$, then $\widetilde{U}_{t_1+t_2}=(1-\eta_2\lambda)^{t_2} \widetilde{U}_{t_1} + \Xi_{t_1+t_2, t_1}$, where $\Xi_{t_1+t_2, t_1} = - \sum_{t^\prime=1}^{t_2} \eta_2(1-\eta_2\lambda)^{t_2-t^\prime} \xi_{t_1+t^\prime-1}$.
    For the variance $\text{Var}(\widetilde{U}_{t_1+t_2})$,
    \begin{align*}
    	\text{Var}(\widetilde{U}_{t_1+t_2}) &= (1-\eta_2\lambda)^{2t_2} \text{Var}(\widetilde{U}_{t_1}) + \text{Var}(\Xi_{t_1+t_2, t_1})
        \geq (1-\eta_2\lambda)^{2t_2} \tau_0^2,
    \end{align*}
    Since $\eta_2 \ll \eta_1$, we get $\text{Var}(\widetilde{U}_{t_1+t_2}) = \Omega(\tau_0^2)$.
    For the variance $\text{Var}(\Xi_{t_1+t_2, t_1})$,
    \begin{align*}
    	\text{Var}(\Xi_{t_1+t_2, t_1}) &= \sum_{t^\prime=1}^{t_2} \eta_2^2 (1 - \eta_2\lambda)^{2(t_2-t^\prime)}\tau_\xi^2 
        \leq \frac{2\eta_2\tau_0^2}{\eta_1}.
    \end{align*}
    From the property of Gaussian variables,
    \begin{align}
		\Pr\left[\mathbbm{1}\left(X^\top \widetilde{U}_{t_1 + t_2} x_L\right) \neq \mathbbm{1}\left(X^\top \widetilde{U}_{t_1} x_L\right)\right] \lesssim \sqrt{\frac{\eta_2\tau_0^2  \|X\|_F^2\|x_L\|^2/\eta_1}{\tau_0^2 \|X\|_F^2\|x_L\|^2}} = \sqrt{\frac{\eta_2}{\eta_1}}.
	\end{align}
	Using Hoeffding's inequality in Lemma \ref{lemma:Hoeffding}, this yields an $L_1$ bound of $\lesssim L\sqrt{\frac{\eta_2}{\eta_1}}  + \sqrt{L \log d}$ with probability at least $1-1/d$.
    
	Combining terms $A, B, C$, we obtain
	\begin{align*}
		\left\|\mathbbm{1}\left(X^\top U_{t_1 + t_2}x_L\right) - \mathbbm{1}\left(X^\top U_{t_1}x_L\right)\right\|_1 \lesssim \epsilon_U+L\sqrt{\frac{\eta_2}{\eta_1}}+ \sqrt{L \log d}.
	\end{align*}
	Furthermore, similar to Lemma~\ref{lemma:activation-patterns-difference} (using Hoeffding's Inequality),
	\begin{align*}
		&\left|N_{U_{t_1+t_2}}(\overline{U}_{t_1+t_2}; X, \widetilde{Y}) - N_{U_{t_1}}(\overline{U}_{t_1+t_2}; X, \widetilde{Y})\right| \\
        =& \left|\frac{1}{2L}\sum_{i \in [2L]} [\widetilde{Y}]_i \left(\mathbbm{1}\left([X^\top]_i U_{t_1+t_2} x_L\right) - \mathbbm{1}\left([X^\top]_i U_{t_1} x_L\right)\right) [X^\top]_i \overline{U}_{t_1+t_2} x_L \right|\\
		\lesssim& \frac{1}{2L} \cdot \max_i \left|[X^\top]_i\overline{U}_{t_1+t_2} x_L\right| \cdot \sqrt{\log d}\cdot \left\|\mathbbm{1}\left(X^\top U_{t_1+t_2} x_L\right) - \mathbbm{1}\left(X^\top U_{t_1} x_L\right)\right\|_1^{1/2}  \\
		\lesssim& \left(\epsilon_U+L\sqrt{\frac{\eta_2}{\eta_1}}+ \sqrt{L \log d}\right)^{1/2}K_l m^4 (\log d)^{1/2} \lambda^{-1} L^{-1}.
	\end{align*}
\end{proof}

\begin{corollary}\label{coro:lemma-ac-W-t1+t-t}
    Let $X_1 \in \mathbb{R}^{d \times L}, x_{L,1} \in \mathbb{R}^{d}$ be a fixed example, with $\|x_{L,1}\|_2 \leq 1+\gamma_0$ (Assumption~\ref{ass:x-norm}). Similarly to Lemma~\ref{lemma:activation-patterns-t1+t-t1}, then, w.h.p over the randomness of $\widetilde{W}$ and $Y$,  $\forall \overline{W} \in \mathbb{R}^{d\times d}$, we have that 
	\begin{align*}
		\left\|\mathbbm{1}\left(X_1^\top W_{t_1 + t_2}x_{L,1}\right) - \mathbbm{1}\left(X_1^\top W_{t_1}x_{L,1}\right)\right\|_1 \lesssim \epsilon_W+  L\sqrt{\frac{\eta_2}{\eta_1}}  + \sqrt{L \log d},
	\end{align*}
	where $\epsilon_W = K_l^{4/3} (1+\gamma_0)^{8/3}\lambda^{-4/3} \tau_0^{-4/3}L^{4/3}$. Furthermore,
	\begin{align*}
		|N_{W_{t_1+t_2}}(\overline{W}_{t_1+t_2}; X_1, {Y}) - N_{W_{t_1}}(\overline{W}_{t_1+t_2}; X_1, {Y})|
        \lesssim \left(\epsilon_W+L\sqrt{\frac{\eta_2}{\eta_1}}+ \sqrt{L \log d}\right)^{1/2}K_l (1+\gamma_0)^4 (\log d)^{1/2} \lambda^{-1} L^{-1}.
	\end{align*}
\end{corollary}
\begin{corollary}\label{coro:lemma-ac-V-t1+t-t}
	Let $X_2 \in \mathbb{R}^{d \times L}, x_{L,2} \in \mathbb{R}^{d}$ be a fixed example, with $\|x_{L,2}\|_2 \leq u+r$ (Assumption~\ref{ass:x-norm}). Similarly to Lemma~\ref{lemma:activation-patterns-t1+t-t1}, then, w.h.p over the randomness of $\widetilde{V}$ and $Y$,  $\forall \overline{V} \in \mathbb{R}^{d\times d}$, we have that 
	\begin{align*}
		\left\|\mathbbm{1}\left(X_2^\top V_{t_1 + t_2}x_{L,2}\right) - \mathbbm{1}\left(X_2^\top V_{t_1}x_{L,2}\right)\right\|_1 \lesssim \epsilon_V+  L\sqrt{\frac{\eta_2}{\eta_1}}  + \sqrt{L \log d},
	\end{align*}
	where $\epsilon_V = K_l^{4/3} (u+r)^{8/3}\lambda^{-4/3} \tau_0^{-4/3}L^{4/3}$. Furthermore,
	\begin{align*}
		\left|N_{V_{t_1+t_2}}(\overline{V}_{t_1+t_2}; X_2, {Y}) - N_{V_{t_1}}(\overline{V}_{t_1+t_2}; X_2, {Y})\right|
        \lesssim \left(\epsilon_W+L\sqrt{\frac{\eta_2}{\eta_1}}+ \sqrt{L \log d}\right)^{1/2}K_l (u+r)^4 (\log d)^{1/2} \lambda^{-1} L^{-1}.
	\end{align*}
\end{corollary}

\begin{proposition}\label{prop:N-tU-tU}
	Under the same setting as Lemma \ref{lemma:activation-patterns-difference}, we have w.h.p over the randomness of $\widetilde{U}$,
	\begin{align*}
		\left|N_{\widetilde{U}}(\widetilde{U}; X,\widetilde{Y})\right| \lesssim \tau_0 m^2(\log d) L^{-1/2}.
	\end{align*}
\end{proposition}
\begin{proof}
	We have $N_{\widetilde{U}}(\widetilde{U}; X,\widetilde{Y}) = \frac{1}{2L} \sum_{i \in [2L]}[\widetilde{Y}]_i \left[[X^\top]_i \widetilde{U} x_L\right]_{+}$.
    Using the Chernoff bound for Gaussian variable in Lemma~\ref{lemma:chernoff-gaussian}, we have w.h.p $1-1/d$:
    \begin{align*}
        \left|\left[[X^\top]_i \widetilde{U} x_L\right]_{+}\right| \leq \left|[X^\top]_i \widetilde{U} x_L\right|
        \lesssim \tau_0 m^2\sqrt{\log d}.
    \end{align*}
    \emph{i.e.}, $\Pr \left( \left|[X^\top]_i\widetilde{U} x_L\right| \lesssim \tau_0 m^2\sqrt{\log d} \right)  \geq 1-\frac{1}{d}$. 
	Using Hoeffding's inequality in Lemma \ref{lemma:Hoeffding}, since $[\widetilde{Y}]_i \in \{-1,1\}$,  $m_i = -1/2L|[[X^\top]_i \widetilde{U} x_L]_{+}|$, $M_i = 1/2L|[[X^\top]_i \widetilde{U} x_L]_{+}|$, with $1-1/d$ prob, we get
	\begin{align*}
		\left|N_{\widetilde{U}}(\widetilde{U}; X,\widetilde{Y})\right| = \left|\frac{1}{2L}\sum_{i \in [2L]}[\widetilde{Y}]_i \left[[X^\top]_i \widetilde{U} x_L\right]_{+}\right| \lesssim \tau_0 m^2 (\log d) L^{-1/2}.
	\end{align*}
\end{proof}

\begin{proposition}\label{prop:N-U-tU}
	Under the same setting as Lemma \ref{lemma:activation-patterns-difference}, with Proposition \ref{prop:N-tU-tU}, we have w.h.p over the randomness of $\widetilde{U}$, $\forall \overline{U}\in \mathbb{R}^{2d\times 2d}$,
	\begin{align*}
		\left|N_{U}(\widetilde{U}; X, \widetilde{Y}) - N_{\widetilde{U}}(\widetilde{U}; X, \widetilde{Y})\right| \lesssim K_l^{2/3}m^{10/3}{(\log d)} \lambda^{-2/3}\tau_0^{1/3} L^{-2/3},
	\end{align*}
	and
	$$
	\left|N_{U}(\widetilde{U}; X, \widetilde{Y})\right| \lesssim  K_l^{2/3}m^{10/3}(\log d) \lambda^{-2/3}\tau_0^{1/3} L^{-2/3} + \tau_0 m^2 (\log d) L^{-1/2}  \triangleq \epsilon_{U,1}.
	$$
\end{proposition}
\begin{proof}
    Similar to Lemma~\ref{lemma:activation-patterns-difference}, we use Hoeffding's Inequality with probability at least $1-1/d$,
	\begin{align*}
		\left|N_{U}(\widetilde{U}; X, \widetilde{Y}) - N_{\widetilde{U}}(\widetilde{U}; X, \widetilde{Y})\right| 
        =& \left|\frac{1}{2L}\sum_{i \in [2L]} [\widetilde{Y}]_i \left(\mathbbm{1}\left([X^\top]_i U x_L\right) - \mathbbm{1}\left([X^\top]_i \widetilde{U} x_L\right)\right) [X^\top]_i \widetilde{U} x_L \right|\\
		\lesssim & \frac{1}{2L} \cdot \max_i \left|[X^\top]_i \widetilde{U} x_L\right| \cdot \sqrt{\log d} \cdot \left\|\mathbbm{1}\left(X^\top U x_L\right) - \mathbbm{1}\left(X^\top \widetilde{U} x_L\right)\right\|_1^{1/2}.
	\end{align*}
    As derived in Proposition~\ref{prop:N-tU-tU}, the pre-activation scalar for the noise matrix is bounded by $\left|[X^\top]_i \widetilde{U} x_L\right| \lesssim \tau_0 m^2\sqrt{\log d}.$ Then,
    \begin{align*}
		\left|N_{U}(\widetilde{U}; X, \widetilde{Y}) - N_{\widetilde{U}}(\widetilde{U}; X, \widetilde{Y})\right| 
        \lesssim & \frac{1}{2L} \cdot \max_i \left|[X^\top]_i \widetilde{U} x_L\right| \cdot \sqrt{\log d} \cdot \left\|\mathbbm{1}\left(X^\top U x_L\right) - \mathbbm{1}\left(X^\top \widetilde{U} x_L\right)\right\|_1^{1/2} \\
		\lesssim& \tau_0 m^2 (\log d) L^{-1}\epsilon_U^{1/2} \\
        =& K_l^{2/3}m^{10/3}(\log d) \lambda^{-2/3}\tau_0^{1/3} L^{-2/3}.
	\end{align*}
	With Proposition \ref{prop:N-tU-tU}, using triangle inequality, we have
    \begin{align*}
        \left| N_U(\widetilde{U}; X, \widetilde{Y}) \right| &= \left| N_U(\widetilde{U}; X, \widetilde{Y}) - N_{\widetilde{U}}(\widetilde{U}; X, \widetilde{Y}) + N_{\widetilde{U}}(\widetilde{U}; X, \widetilde{Y}) \right| \\
        &\le \left| N_U(\widetilde{U}; X, \widetilde{Y}) - N_{\widetilde{U}}(\widetilde{U}; X, \widetilde{Y}) \right| + \left| N_{\widetilde{U}}(\widetilde{U}; X, \widetilde{Y}) \right| \\
        &\lesssim K_l^{2/3}m^{10/3} (\log d)\lambda^{-2/3}\tau_0^{1/3} L^{-2/3} + \tau_0 m^2 (\log d) L^{-1/2}.
    \end{align*}
\end{proof}

\begin{corollary}\label{coro:N-W-tW}
	Let $X_1 \in \mathbb{R}^{d \times L}, x_{L,1} \in \mathbb{R}^{d}$ be a fixed example, with $\|x_{L,1}\|_2 \leq 1+\gamma_0$ (Assumption~\ref{ass:x-norm}). Similarly to Proposition~\ref{prop:N-U-tU}, then, w.h.p over the randomness of $\widetilde{W}$ and $Y$,  $\forall \overline{W} \in \mathbb{R}^{d\times d}$, we have that
    \begin{align}\label{eq:def-epsilonw1}
    	   \left|N_{W}(\widetilde{W}; X_1, Y)\right| \lesssim K_l^{2/3}(1+\gamma_0)^{10/3}(\log d) \lambda^{-2/3}\tau_0^{1/3}L^{-2/3}+ \tau_0 (1+\gamma_0)^2 (\log d) L^{-1/2}  \triangleq \epsilon_{W,1}.
    \end{align}
\end{corollary}

\begin{corollary}\label{coro:N-V-tV}
	Let $X_2 \in \mathbb{R}^{d \times L}, x_{L,2} \in \mathbb{R}^{d}$ be a fixed example, with $\|x_{L,2}\|_2 \leq m$ (Assumption~\ref{ass:x-norm}). Similarly to Proposition~\ref{prop:N-U-tU}, then, w.h.p over the randomness of $\widetilde{V}$ and $Y$,  $\forall \overline{V} \in \mathbb{R}^{d\times d}$, we have that 
	\begin{align}\label{eq:def-epsilonV1}
		\left|N_{V}(\widetilde{V}; X_2, Y)\right| \lesssim K_l^{2/3}(u+r)^{10/3}(\log d) \lambda^{-2/3}\tau_0^{1/3}L^{-2/3}+ \tau_0 (u+r)^2 (\log d) L^{-1/2}\triangleq \epsilon_{V,1}.
	\end{align}
\end{corollary}

\textbf{\textit{Note.}} For notation clarity in analyzing the decoupled components, we define the individual asymptotic error bounds as $\epsilon_{U,1}$, $\epsilon_{W,1}$, and $\epsilon_{V,1}$ for weights $U$, $W$, and $V$, respectively.
In these expressions, $K_l$ is the Lipschitz constant of $\mathcal{O}(1)$, $m, 1+\gamma_0, u+r$ are of constant order $\mathcal{O}(1)$, $\lambda$ denotes the $L_2$ regularization coefficient, $\tau_0$ is the standard deviation of the initialization, and $L$ is the prompt length. 
With choices in Assumption~\ref{ass:choice-hyperparam}, the first term in the bound scales as $\Theta\big( (\log d) d^{3/2} L^{-2/3} \big)$; the second term scales as $\Theta\big((\log d) d^{-1/2}L^{-1/2} \big)$. Thus, the first term asymptotically dominates, yielding an overall magnitude of $\Theta\big( (\log d) d^{3/2} L^{-2/3} \big)$, i.e., $\Theta\big( (\log d) d^{3/2} (\text{Poly(d)})^{-2/3} \big)$. This bound naturally converges to $0$ with a polynomial degree of $L$ (greater than $2.25$, i.e., $L=\Omega(d^{2.25})$).

\subsection{Proof for the Elementary Stage: Proof of Theorem \ref{the:stage1-Q}}\label{proof:stage1-Q}

\begin{proof}
Using noise part to compute activation and signal part as weight.
\begin{align*}
	\widetilde{g}_t(X_2) = N_{\widetilde{V}_t}(\overline{V}_t; X_2, Y)
	=  Y/L \left(\mathbbm{1}\left(X_2^\top \widetilde{V}_t x_{L,2}\right) \odot \left(X_2^\top \overline{V}_t x_{L,2}\right)\right).
\end{align*}
Using triangle inequality, with Corollary \ref{coro:lemma-ac-V} and \ref{coro:N-V-tV},
\begin{align*}
	&\left| g_t(X_2) - \widetilde{g}_t(X_2) \right| \\
    =& \left|N_{V_t}(V_t; X_2, Y) - N_{\widetilde{V}_t}(\overline{V}_t; X_2, Y)\right| \\
	\leq& \left|N_{V_t}(\overline{V}_t; X_2,Y) - N_{\widetilde{V}_t}(\overline{V}_t; X_2,Y)\right|+ \left|N_{V_t}(\widetilde{V}_t; X_2,Y)\right| \\
	\lesssim& K_l^{5/3}(u+r)^{16/3}(\log d)^{1/2}\lambda^{-5/3}\tau_0^{-2/3}L^{-2/3}
    +K_l^{2/3}(u+r)^{10/3}(\log d)\lambda^{-2/3}\tau_0^{1/3}L^{-2/3} + \tau_0 (u+r)^2 (\log d) L^{-1/2}.
\end{align*}
With choices of $u+r = \Theta(1)$, $\tau_0 = \Theta(1/\sqrt{d})$, and $\lambda= \Theta(1/d^{5/2})$, we substitute these into the upper bound:
\begin{align*}
    |g_t(X_2) - \widetilde{g}_t(X_2)| &\lesssim \frac{(\log d)^{1/2} d^{9/2}}{L^{2/3}} + \frac{(\log d) d^{3/2}}{L^{2/3}} + \frac{ \log d}{d^{1/2}L^{1/2}},
\end{align*}
where the first term dominates the remaining terms.
We require $L = \Theta(\text{Poly}(d))$ with a large polynomial degree (greater than $6.75$, i.e., $L=\Omega(d^{6.75})$), the polynomial in the denominator asymptotically absorbs the numerator. This guarantees that the bound converges to $0$ at a polynomial rate, yielding:
\begin{align}\label{eq:g-tildeg}
    |g_t(X_2) - \widetilde{g}_t(X_2)| = \mathcal{O}\left( \frac{1}{\text{Poly}(d)} \right). 
\end{align}

\textbf{In the following, we focus on $\widetilde{g}_t(X_2)$.}
\begin{definition}\label{def:epsilon-tV}
	For any time $t$, input $X \in \mathbb{R}^{d \times L}$ with query $x_L \in \mathbb{R}^d$, define $\epsilon^{X,x_L}_t \triangleq \{i \in [L]: [X^\top]_i \widetilde{V}_t x_L \geq 0\}$ and $\overline{\epsilon}^{X,x_L}_t \triangleq \{i \in [L]: [X^\top]_i \widetilde{V}_t x_L < 0\}$. Note that $X$ aligns with $X_2$ and $x_L$ aligns with $x_{L,2}$. Then $\mathbbm{1}(\epsilon) \subset \{0,1\}^L$. Naturally, we have
	\begin{align*}
		\mathbbm{1}(\epsilon^{X,x_L}_t) = \mathbbm{1}(X^\top \widetilde{V}_t x_L).
	\end{align*}
\end{definition}
Let $Q_t= \text{diag}(Y^\top)X_2^\top \overline{V}_t$, then
\begin{align*}
	\widetilde{g}_t(X_2) &= N_{\widetilde{V}_t}(\overline{V}_t; X_2, Y) 
	= Y/L \left(\mathbbm{1}\left(X_2^\top \widetilde{V}_t x_{L,2}\right) \odot \left(X_2^\top \overline{V}_t x_{L,2}\right)\right) \\
	&= 1/L \cdot \mathbbm{1}\left(X_2^\top \widetilde{V}_t x_{L,2}\right)^\top \left(\text{diag}(Y^\top)X_2^\top \overline{V}_t\right) x_{L,2} \\
	&= 1/L \cdot \mathbbm{1}\left(X_2^\top \widetilde{V}_t x_{L,2}\right)^\top Q_t x_{L,2}.
\end{align*}

\textbf{To simplify, we use $X$ that represents $X_2$ and $x_L$ represents $x_{L,2}$, in this Theorem, if there is no confusion.}

Define $\widetilde{g}_t (X, z-\zeta)$ as sequence $X$ with $x_L=z-\zeta$, similarly for $\widetilde{g}_t (X, z+\zeta)$ and $\widetilde{g}_t (X, z)$. With Definition \ref{def:epsilon-tV},
\begin{align}
	&\left|\widetilde{g}_t (X, z-\zeta) + \widetilde{g}_t (X, z+\zeta) - 2 \widetilde{g}_t (X, z)\right| \nonumber\\
	=& 1/L \cdot \left| \mathbbm{1}\left(\epsilon^{X, z-\zeta}_t\right)^\top Q_t (z-\zeta) + \mathbbm{1}\left(\epsilon^{X, z+\zeta}_t\right)^\top Q_t(z+\zeta)-2\mathbbm{1}\left(\epsilon^{X,z}_t\right)^\top Q_t z \right| \nonumber\\
	\leq& 1/L \cdot \underbrace{\left|\left(\mathbbm{1}\left(\epsilon^{X,z-\zeta}_t\right) + \mathbbm{1}\left(\epsilon^{X,z+\zeta}_t\right) -2\mathbbm{1}\left(\epsilon^{X,z}_t\right) \right)^\top Q_t z\right|}_{\Phi}+1/L \cdot \underbrace{\left|\left(\mathbbm{1}\left(\epsilon^{X,z+\zeta}_t\right)-\mathbbm{1}\left(\epsilon^{X, z-\zeta}_t\right)\right)^\top Q_t\zeta\right|}_{\Psi}. \label{eq:Psi}
\end{align}

\textbf{Deal with term $\Psi$.}\quad
With Assumption \ref{ass:x-norm} that $\left\|\zeta\right\|_2=r$,
\begin{align*}
	\left|\left(\mathbbm{1}\left(\epsilon^{X,z+\zeta}_t\right)-\mathbbm{1}\left(\epsilon^{X,z-\zeta}_t\right)\right)^\top Q_t \zeta\right| &\leq \left\|\left(\mathbbm{1}\left(\epsilon^{X, z+\zeta}_t\right)-\mathbbm{1}\left(\epsilon^{X, z-\zeta}_t\right)\right)^\top Q_t\right\|_2 \left\|\zeta\right\|_2 \\
	&\leq r \left|\epsilon^{X, z+\zeta}_t \oplus \epsilon^{X, z-\zeta}_t\right| \cdot \max\left\|[Q_t]_i\right\|_2.
\end{align*}
\textbf{For $\epsilon^{X, z+\zeta}_t \oplus \epsilon^{X, z-\zeta}_t$ in term $\Psi$.}\quad
For $i \in \epsilon^{X, z+\zeta}_t \oplus \epsilon^{X, z-\zeta}_t$, with $[X^\top]_i \widetilde{V}_t (z+\zeta) \geq 0$ and $[X^\top]_i \widetilde{V}_t (z-\zeta) \leq 0$, then
\begin{align*}
	-[X^\top]_i \widetilde{V}_t \zeta \leq [X^\top]_i  \widetilde{V}_t z \leq [X^\top]_i \widetilde{V}_t \zeta \implies
	\left|[X^\top]_i \widetilde{V}_t z\right| \leq \left|[X^\top]_i \widetilde{V}_t \zeta\right|.
\end{align*}
We have $\text{Var}\left([X^\top]_i\widetilde{V}_t \zeta\right) \leq \tau_0^2 (u+r)^2 r^2$ since this is a Gaussian random variable together with Assumption~\ref{ass:x-norm}.
Then, using the Chernoff bound for Gaussian variable in Lemma \ref{lemma:chernoff-gaussian}, we have w.h.p $1-1/d$:
\begin{align}
	\left|[X^\top]_i \widetilde{V}_t z\right| \leq\left|[X^\top]_i\widetilde{V}_t \zeta\right|\lesssim \tau_0 r(u+r)\sqrt{\log d}.\label{eq:XVz_lower_bound}
\end{align}
By the property of standard Gaussian random variable (together with $r = \Theta\left(1/\text{Poly}(d)\right), u=\Theta(1)$), we get the upper bound
\begin{align}
	\Pr \left( |[X^\top]_i\widetilde{V}_t z| \lesssim \tau_0 r(u+r) \sqrt{\log d} \right)
	\leq\Pr \left( \left| \frac{[X^\top]_i\widetilde{V}_t z}{\tau_0 u(u-r)} \right| \lesssim \frac{\tau_0 r(u+r) \sqrt{\log d}}{\tau_0 u(u-r)} \right) \lesssim \frac{r(u+r) \sqrt{\log d}}{u(u-r)}.\label{eq:XVz_upper_bound}
\end{align}

With Bernstein inequality in Lemma \ref{lemma:Bernstein}, define new random variable $R_i = \mathbb{I}(|[X^\top\widetilde{V}_t]_i z| \lesssim \tau_0 r(u+r) \sqrt{\log d})$ where $\mathbb{I}(\cdot)$ is the indicator function, $\mathbb{E}\left[R_i \right]=\Pr(|[X^\top\widetilde{V}_t]_i z| \lesssim \tau_0 r(u+r) \sqrt{\log d}) \lesssim \frac{r(u+r) \sqrt{\log d}}{u(u-r)}$, $\sigma_{R_i}^2 \leq 1$. Then w.h.p. $1-1/d$ we have
\begin{align*}
	\sum_{i=1}^L R_i \lesssim \sqrt{L \log d} + \log d + \frac{L r(u+r) \sqrt{\log d}}{u(u-r)},
\end{align*}
\emph{i.e.}, $|\epsilon_t^{X, z-\zeta} \oplus \epsilon_t^{X, z+\zeta}| \lesssim \sqrt{L \log d} + \log d + \frac{L r(u+r) \sqrt{\log d}}{u(u-r)}$. For large $L$,
\begin{align}
	|\epsilon_t^{X, z-\zeta} \oplus \epsilon_t^{X, z+\zeta}| \lesssim \left(\sqrt{L} + \frac{L r(u+r)}{u(u-r)}\right) \sqrt{\log d}.  \label{eq:epsilon-tV}
\end{align}
\textbf{For $[Q_t]_i$ in term $\Psi$.}\quad 
For $Q_t= \text{diag}(Y^\top)X^\top \overline{V}_t$, using  Cauchy-Schwarz inequality, Assumption \ref{ass:x-norm} and Proposition \ref{prop:weight-norm},
\begin{align}
	\left\|[Q_t]_i\right\|_2 = \left\|y_i \sum_{j=1}^d [X^\top]_{ij}[\overline{V}_t]_j\right\|_2
	\leq\|[X]_i\|_2 \|\overline{V}_t\|_F
	\lesssim \frac{K_l (u+r)^3}{\lambda}. \label{eq:Q-i}
\end{align}
\textbf{Combine Equation \ref{eq:epsilon-tV} and \ref{eq:Q-i}.}\quad 
For term $\Psi$, we have
\begin{align}
	\left|\left(\mathbbm{1}\left(\epsilon^{X,z+\zeta}_t\right)-\mathbbm{1}\left(\epsilon^{X,z-\zeta}_t\right)\right)^\top Q_t \zeta\right| &\leq \left\|\left(\mathbbm{1}\left(\epsilon^{X, z+\zeta}_t\right)-\mathbbm{1}\left(\epsilon^{X, z-\zeta}_t\right)\right)^\top Q_t\right\|_2 \left\|\zeta\right\|_2 \nonumber\\
	&\leq r \left|\epsilon^{X, z+\zeta}_t \oplus \epsilon^{X, z-\zeta}_t\right| \cdot \max\left\|[Q_t]_i\right\|_2 \nonumber\\
	&\lesssim \left( \sqrt{L} + \frac{L r(u+r)}{u(u-r)} \right) \frac{K_l r (u+r)^3 \sqrt{\log d}}{\lambda}.\label{eq:Psi}
\end{align}
Since then, we have completed term $\Psi$ in Equation \ref{eq:Psi}.

\textbf{Deal with term $\Phi$.}\quad 
Let $a = \left(\mathbbm{1}\left(\epsilon^{X,z-\zeta}_t\right) + \mathbbm{1}\left(\epsilon^{X,z+\zeta}_t\right) -2\mathbbm{1}\left(\epsilon^{X,z}_t\right) \right)^\top$, then
$$
\left(\mathbbm{1}\left(\epsilon^{X,z-\zeta}_t\right) + \mathbbm{1}\left(\epsilon^{X,z+\zeta}_t\right) -2\mathbbm{1}\left(\epsilon^{X,z}_t\right) \right)^\top Q_t z = a^\top Q_t z.
$$
According to the definition of $Q_t$ and $\overline{V}_t$, we have
\begin{align*}
	a^\top Q_t &= a^\top \text{diag}(Y^\top)X^\top \overline{V}_t 
	= a^\top  \text{diag}(Y^\top) X^\top \sum_{\tau=1}^t \eta_1\left(1-\eta_1 \lambda\right)^{t-\tau} \nabla_{V_{\tau-1}} \widehat{L}(U_{\tau-1}) \\
	&=a^\top  \sum_{\tau=1}^t \eta_1\left(1-\eta_1 \lambda\right)^{t-\tau} \Delta Q_{\tau-1},
\end{align*}
where $\Delta Q_\tau = \text{diag}(Y^\top)X^\top \nabla_{V_\tau} \widehat{L}(U_\tau)$. Then
\begin{align*}
	\left|\left(\mathbbm{1}\left(\epsilon^{X,z-\zeta}_t\right) + \mathbbm{1}\left(\epsilon^{X,z+\zeta}_t\right) -2\mathbbm{1}\left(\epsilon^{X,z}_t\right) \right)^\top Q_t z\right|
    \leq \eta_1 u \sum_{\tau=1}^t \left\|\left(\mathbbm{1}\left(\epsilon^{X,z-\zeta}_t\right) + \mathbbm{1}\left(\epsilon^{X,z+\zeta}_t\right)-2\mathbbm{1}\left(\epsilon^{X,z}_t\right)\right)^\top \Delta Q_{\tau-1} \right\|_2.
\end{align*}
\textbf{For $\Delta Q_\tau$ in term $\Phi$.}\quad
We begin with the following definition.
\begin{definition}\label{def:G/epsilon-V}
	For any time $t$, input $X \in \mathbb{R}^{d \times L}$ with query $x_L \in \mathbb{R}^d$, define $\mathcal{G}^{X,x_L}_\tau \triangleq \{i \in [L]: [X^\top]_i {V}_\tau x_L \geq 0\}$ and $\overline{\mathcal{G}}^{X,x_L}_\tau \triangleq \{i \in [L]: [X^\top]_i {V}_\tau x_L < 0\}$. Similar to Definition \ref{def:epsilon-tV}, note that $X$ aligns with $X_2$ and $x_L$ aligns with $x_{L,2}$.
\end{definition}

Suppose $i,j$ satisfy that, for input $x_L = z-\zeta$ and $x_L = z+\zeta$ have the same activation pattern, then with Definition \ref{def:G/epsilon-V}, 
\begin{align*}
	i,j \in \mathcal{G}^{X,z-\zeta}_\tau  \cap \mathcal{G}^{X,z+\zeta}_\tau \text{\ or\ } i,j \in \overline{\mathcal{G}}^{X,z-\zeta}_\tau  \cap \overline{\mathcal{G}}^{X,z+\zeta}_\tau. 
\end{align*}
Consider the relationship between $[\Delta Q_\tau]_i$ and $[\Delta Q_\tau]_j$ for the above $i,j$. We have $\Delta Q_\tau = \text{diag}(Y^\top)X^\top \nabla_{V_\tau} \widehat{L}(U_\tau)$.
With Proposition \ref{prop:derivative-L}, then
\begin{align*}
	[\Delta Q_\tau]_i = y_i \left[ X^\top\right]_i \nabla_{V_\tau} \widehat{L}(U_\tau)
    &=\widehat{\mathbb{E}}\left[\frac{1}{2L} l^\prime(f(U_\tau; X, \widetilde{Y}))y_L y_i \sum_{k=1}^L y_k \mathbbm{1}([X^\top]_k V_\tau x_L) \langle [X^\top]_i,[X^\top]_{k} \rangle x_{L}^\top\right]; \\
	[\Delta Q_\tau]_j = y_j \left[X^\top\right]_j \nabla_{V_\tau} \widehat{L}(U_\tau)
    &=\widehat{\mathbb{E}}\left[\frac{1}{2L} l^\prime(f(U_\tau; X, \widetilde{Y}))y_L y_j \sum_{k=1}^L y_k \mathbbm{1}([X^\top]_k V_\tau x_L) \langle [X^\top]_j,[X^\top]_{k} \rangle x_{L}^\top\right].
\end{align*}
With Assumption~\ref{ass:x-norm} and Proposition~\ref{prop:derivative-L-norm}, $\left\|[\Delta Q_{\tau}]_i\right\|_2=\left\|y_i \left[ X^\top\right]_i \nabla_{V_\tau} \widehat{L}(U_\tau)\right\|_2 \leq  1/2 K_l(u+r)^3$.

For all $x_{L} \in \{z, z-\zeta, z+\zeta\}$, $i,j \in \mathcal{G}^{X,z-\zeta}_\tau  \cap \mathcal{G}^{X,z+\zeta}_\tau$, and then $i,j \in \mathcal{G}^{X,z}_\tau$. Thus, $\mathbbm{1}([X^\top]_i V_\tau x_{L}) = \mathbbm{1}([X^\top]_j V_\tau x_{L}) = 1$.

For a fixed $X$, the gradient matrix $\nabla_{V_\tau}\widehat{L}(U_\tau)$ is identical for all tokens. If $[X]_i = [X]_j$, then $y_i = y_j$, which yields $[\Delta Q_\tau]_i= [\Delta Q_\tau]_j.$ If $[X]_i, [X]_j \in \{z-\zeta, z+\zeta\}$, then $y_i = y_j$ and $[\Delta Q_\tau]_i - [\Delta Q_\tau]_j = \pm 2\zeta^\top C$ (where $C = y_i \nabla_{V_\tau}\widehat{L}(U_\tau)$.)
If $[X_2]_i, [X_2]_j \in \{z \pm \zeta, z\}$, then $y_i = -y_j$ and $[\Delta Q_\tau]_i - [\Delta Q_\tau]_j = (2z\pm \zeta)^\top C \ \text{or}\  (-2z \pm \zeta)^\top C$.

\noindent\textbf{For $\left(\mathbbm{1}\left(\epsilon^{X,z-\zeta}_t\right) + \mathbbm{1}\left(\epsilon^{X,z+\zeta}_t\right)-2\mathbbm{1}\left(\epsilon^{X,z}_t\right)\right)^\top \Delta Q_\tau$ in term $\Phi$.}\quad
With Definition \ref{def:epsilon-tV}, we have
\begin{align*}
	&\mathbbm{1}\left(\epsilon^{X,z-\zeta}_t\right) + \mathbbm{1}\left(\epsilon^{X,z+\zeta}_t\right) -2\mathbbm{1}\left(\epsilon^{X,z}_t\right) \\
	=& \mathbbm{1}\left( \epsilon^{X,z-\zeta}_t \cap  \epsilon^{X,z}_t \right)+ \mathbbm{1}\left(\epsilon^{X,z-\zeta}_t \setminus \epsilon^{X,z}_t\right) + \mathbbm{1}\left( \epsilon^{X,z+\zeta}_t \cap  \epsilon^{X,z}_t \right)+ \mathbbm{1}\left(\epsilon^{X,z+\zeta}_t \setminus \epsilon^{X,z}_t\right) \\
	&- \mathbbm{1}\left( \epsilon^{X,z}_t \cap  \epsilon^{X,z-\zeta}_t \right) - \mathbbm{1}\left(\epsilon^{X,z}_t \setminus \epsilon^{X,z-\zeta}_t\right) - \mathbbm{1}\left( \epsilon^{X,z}_t \cap  \epsilon^{X,z+\zeta}_t \right) - \mathbbm{1}\left(\epsilon^{X,z}_t \setminus \epsilon^{X,z+\zeta}_t\right) \\
	=& \underbrace{\mathbbm{1}\left(\epsilon^{X,z+\zeta}_t \setminus \epsilon^{X,z}_t\right)- \mathbbm{1}\left(\epsilon^{X,z}_t \setminus \epsilon^{X,z-\zeta}_t\right)}_{\text{Part\ I}} + \underbrace{\mathbbm{1}\left(\epsilon^{X,z-\zeta}_t \setminus \epsilon^{X,z}_t\right) - \mathbbm{1}\left(\epsilon^{X,z}_t \setminus \epsilon^{X,z+\zeta}_t\right)}_{\text{Part\ II}}.
\end{align*}
Observe that Part I and Part II are similar, and we deal with Part I first. 

Let $A=\epsilon^{X,z+\zeta}_t \setminus \epsilon^{X,z}_t$ and $B=\epsilon^{X,z}_t \setminus \epsilon^{X,z-\zeta}_t$. 
Based on the above high probability results, we have $\left|[X^\top]_i\widetilde{V}_t z\right|, \left|[X^\top]_i\widetilde{V}_t \zeta\right| \lesssim \tau_0 r(u+r) \sqrt{\log d}$. For input $x_L \in \{z, z-\zeta, z+\zeta\}$ and time $\tau$, we have $\left|[X^\top]_k\widetilde{V}_\tau (x_L-z)\right| \lesssim \tau_0 r(u+r)\sqrt{\log d}$.
Based on Proposition \ref{prop:weight-norm}, we have $\left|[X^\top]_k\overline{V}_\tau x_L\right|\leq \frac{K_l (u+r)^4}{2\lambda}.$

To facilitate the analysis of sets $A$ and $B$, we aim to partition the tokens into subsets based on their pre-activation values. Specifically, since we need to discuss $\Delta Q_\tau$ which is closely related to set $\mathcal{G}_\tau^{X,x_L}$, we construct a set $\mathcal{F}_\tau^+$ where for $k \in \mathcal{F}_\tau^+$ then $k \in \mathcal{G}_\tau^{X,x_L}$ (Definition~\ref{def:G/epsilon-V}).
Let $\mathcal{F}_\tau^+ \triangleq \{i \in [L]: [X^\top]_i\widetilde{V}_\tau z \gtrsim \kappa\}$ for some threshold $\kappa$.
Then,
\begin{align*}
	[X^\top]_k V_\tau x_L &= [X^\top]_k\widetilde{V}_\tau x_L + [X^\top]_k\overline{V}_\tau x_L \\
	&\geq [X^\top]_k\widetilde{V}_\tau z - \left|[X^\top]_k\widetilde{V}_\tau (x_L-z)\right| - \left|[X^\top]_k\overline{V}_\tau x_L\right| \\
	&\gtrsim \kappa - \tau_0 r(u+r) \sqrt{\log d} - \frac{K_l (u+r)^4}{2\lambda}.
\end{align*}
To satisfy $k \in \mathcal{G}_\tau^{X,x_L}$, let $\kappa - \tau_0 r(u+r) \sqrt{\log d} - \frac{K_l (u+r)^4}{2\lambda}=0$, i.e., $\kappa \triangleq \tau_0 r(u+r) \sqrt{\log d} + \frac{K_l (u+r)^4}{2\lambda}$.
Therefore, we present the following definition:
\begin{definition}\label{def:F/epsilon-tV-whp}
	For any time $\tau$, input $X \in \mathbb{R}^{d \times L}$ with query $x_L=z \in \mathbb{R}^d$, define $\mathcal{F}_\tau^+ \triangleq \{i \in [L]: [X^\top]_i\widetilde{V}_\tau z \gtrsim \kappa \}$, $\mathcal{F}_\tau^- \triangleq \{i \in [L]: [X^\top]_i\widetilde{V}_\tau z \lesssim -\kappa \}$ and $\mathcal{F}_\tau^c \triangleq \{i \in [L]: \left|[X^\top]_i\widetilde{V}_\tau z\right| \lesssim \kappa\}$. Similar to Definition \ref{def:epsilon-tV}, note that $X$ aligns with $X_2$.
\end{definition}
With Definition \ref{def:F/epsilon-tV-whp},
\begin{align}
	&\left\|\left(\mathbbm{1}\left(\epsilon^{X,z+\zeta}_t \setminus \epsilon^{X,z}_t\right)- \mathbbm{1}\left(\epsilon^{X,z}_t \setminus \epsilon^{X,z-\zeta}_t\right)\right)^\top \Delta Q_{\tau} \right\|_2 \nonumber\\
	=&\left\|\sum_{i \in A}  [\Delta Q_{\tau}]_i - \sum_{i \in B}  [\Delta Q_{\tau}]_i\right\|_2  \nonumber\\
	\leq&\left\|\sum_{i \in A \cap \mathcal{F}_\tau^+}  [\Delta Q_{\tau}]_i - \sum_{i \in B \cap \mathcal{F}_\tau^+}  [\Delta Q_{\tau}]_i\right\|_2 +  \left\|\sum_{i \in A \cap \mathcal{F}_\tau^-}  [\Delta Q_{\tau}]_i - \sum_{i \in B \cap \mathcal{F}_\tau^-}  [\Delta Q_{\tau}]_i\right\|_2
    + \left\|\sum_{i \in A \cap \mathcal{F}_\tau^c}  [\Delta Q_{\tau}]_i - \sum_{i \in B \cap \mathcal{F}_\tau^c}  [\Delta Q_{\tau}]_i\right\|_2.\label{eq:A-B-F}
\end{align}
For any two tokens $k, l \in \mathcal{F}_\tau^+$, they satisfy $[X^\top]_k\widetilde{V}_\tau z \geq \kappa$ and $[X^\top]_l\widetilde{V}_\tau z \geq \kappa$ by definition. As established above, we have $k,l \in \mathcal{G}_\tau^{X,x_L}$, which directly implies $k,l \in \mathcal{G}^{X,z-\zeta}_\tau \cap \mathcal{G}^{X,z+\zeta}_\tau$. 
Similarly, for $k, l \in \mathcal{F}_\tau^-$, we have $k, l \in \overline{\mathcal{G}}^{X,z-\zeta}_\tau \cap \overline{\mathcal{G}}^{X,z+\zeta}_\tau$. 
Thus, for tokens $k,l \in \mathcal{F}_\tau^+$ or $k,l \in \mathcal{F}_\tau^-$, we know that $[\Delta Q_\tau]_k$ and $[\Delta Q_\tau]_l$ exhibit the identical relationships as derived previously for $[\Delta Q_\tau]_i$ and $[\Delta Q_\tau]_j$.
Conversely, for tokens $k, l \in \mathcal{F}_\tau^c$, their activation patterns $[X^\top]_k V_\tau x_L$ and $[X^\top]_l V_\tau x_L$ are not guaranteed to remain consistent. In such cases, $[\Delta Q_\tau]_k$ and $[\Delta Q_\tau]_l$ cannot cancel each other out in the third term of Equation \ref{eq:A-B-F} and this norm must be bounded by the absolute cardinality $|A \cap \mathcal{F}_\tau^c|$ and $|B \cap \mathcal{F}_\tau^c|$.

Therefore, with the definition of data structure, we assume that any two tokens share the identical features $[X]_i = [X]_j$ with probability $P$. That is, for tokens within $\mathcal{F}_\tau^+,\mathcal{F}_\tau^-$, we have $[\Delta Q_\tau]_i= [\Delta Q_\tau]_j$ with probability $P$. Then with $\left\|[\Delta Q_{\tau}]_i\right\|_2 \leq  1/2 K_l(u+r)^3$, Equation~\ref{eq:A-B-F} becomes
\begin{align*}
	\text{LHS}
	\leq& \frac{K_l}{2}(u+r)^3\Bigl(P\left||A \cap \mathcal{F}_\tau^+| - |B \cap \mathcal{F}_\tau^+|\right| + P\left||A \cap \mathcal{F}_\tau^-| - |B \cap \mathcal{F}_\tau^-|\right| \\
    &+ (1-P)\left(|A \cap \mathcal{F}_\tau^+| + |B \cap \mathcal{F}_\tau^+|\right)
	+ (1-P)\left(|A \cap \mathcal{F}_\tau^-| + |B \cap \mathcal{F}_\tau^-|\right) 
    + |A \cap \mathcal{F}_\tau^c| + |B \cap \mathcal{F}_\tau^c|\Bigr). 
\end{align*}
\textbf{For $|A \cap \mathcal{F}_\tau^+|$, $|B \cap \mathcal{F}_\tau^+|$ and $\left||A \cap \mathcal{F}_\tau^+| - |B \cap \mathcal{F}_\tau^+|\right|$.}\quad
These sets are related to $[X^\top]_i \widetilde{V}_t z, [X^\top]_i \widetilde{V}_t \zeta$, $[X^\top]_i \widetilde{V}_\tau z$. With Proposition \ref{prop:weight-tau-t} and $\eta=\eta_1$, at time $\tau \leq t$, we have
\begin{align*}
	[X^\top]_i\widetilde{V}_tz &= (1-\eta_1 \lambda)^{t-\tau} [X^\top]_i\widetilde{V}_\tau z -  \sum_{t^\prime=1}^{t-\tau} \eta_1 (1-\eta_1 \lambda)^{t-\tau-t^\prime} [X^\top]_i\xi_{\tau+t^\prime-1}z \\
	&= (1-\eta_1 \lambda)^{t-\tau} [X^\top]_i\widetilde{V}_\tau z + [X^\top]_i\Xi_{t,\tau}z,
\end{align*}
where $\Xi_{t,\tau} = - \sum_{t^\prime=1}^{t-\tau} \eta_1 (1-\eta_1 \lambda)^{t-\tau-t^\prime} \xi_{\tau+t^\prime-1}$. Let $Y_1 = [X^\top]_i\widetilde{V}_\tau z$, $Y_2=[X^\top]_i\widetilde{V}_t z$, $Y_3 = [X^\top]_i\widetilde{V}_t\zeta$, $Y_4 = [X^\top]_i\Xi_{t,\tau}z$, $\beta=(1-\eta_1 \lambda)^{t-\tau}\lesssim 1$, we have $Y_2 = \beta Y_1 + Y_4$.

Since the noise matrix $\widetilde{V}_\tau$ preserves the initialization variance $\tau_0^2$ across its independent Gaussian entries, $Y_1, Y_2, Y_3$ are zero-mean Gaussians. With Assumption \ref{ass:x-norm}, their variances satisfy: $\sigma_1^2 = \sigma_2^2 = \tau_0^2 u^2(u+r)^2$ and $\sigma_3^2 = \tau_0^2 r^2(u+r)^2$. Furthermore, due to $\langle z, \zeta \rangle = 0$ and the independent noise $\Xi_{t,\tau}$, $Y_3$ is independent of $Y_1$ and $Y_2$.

Consider $Y_4$, denote its variance as $\sigma_{t,\tau}^2$. Because $\Xi_{t,\tau}$ consists of independent noise injected after time $\tau$, $Y_4$ is unconditionally independent of $Y_1$ (and $Y_4$ is independent of $Y_3$). The variance of $Y_2$ is decomposed as:
\begin{align*}
    \sigma_{t,\tau}^2 = \sigma_2^2 - (1-\eta_1 \lambda)^{2(t-\tau)} \sigma_1^2 \gtrsim \tau_0^2 u^2(u+r)^2 \eta_1 \lambda(t-\tau).
\end{align*}

Since $(Y_2, Y_4)$ are jointly Gaussian, we apply the standard Gaussian conditioning formulas. Utilizing $\text{Cov}(Y_4, Y_2) = \text{Var}(Y_4) = \sigma_{t,\tau}^2$, the conditional distribution simply follows as:
\begin{align*}
    (Y_4 \mid Y_2 = z) \sim \mathcal{N}\left( (1-\beta^2)z, \, \beta^2 \sigma_{t,\tau}^2 \right).
\end{align*}

Recall Definition~\ref{def:epsilon-tV} and Definition~\ref{def:F/epsilon-tV-whp}, we have
\begin{align*}
    \Pr(A \cap \mathcal{F}_\tau^+) =\Pr[i\in\epsilon_t^{X,z+\zeta},i\notin\epsilon_t^{x,z},i\in\mathcal{F}_\tau^+] =\Pr\left[Y_2+Y_3\geq0,Y_2<0,Y_1\geq\kappa\right].
\end{align*}
With the Chernoff bound for Gaussian variables (Lemma \ref{lemma:chernoff-gaussian}) and $\kappa = \tau_0 r(u+r) \sqrt{\log d} + \frac{K_l (u+r)^4}{2\lambda}$, we evaluate the probabilities for sets $A$ and $B$. Since $\sigma_3 = \tau_0 r(u+r) \ll \sigma_2 = \tau_0 u(u+r)$, the joint Gaussian probability can be bounded by:
\begin{align*}
	\Pr(A \cap \mathcal{F}_\tau^+) &= \mathbb{E}_{Y_2}\left[\Pr\left[Y_3\geq-Y_2\mid Y_2\right]\mathbf{1}(Y_2<0)\Pr\left[Y_4\leq Y_2-\beta \kappa \mid Y_2\right]\right] \\
    &\lesssim \int_{-\infty}^0 \frac{1}{\sqrt{2\pi}\sigma_2} e^{-\frac{z^2}{2\sigma_2^2}} e^{-\frac{z^2}{2\sigma_3^2}} e^{-\frac{(z-\beta\kappa-(1-\beta^2)z)^2}{2 \beta^2 \sigma_{t,\tau}^2}}dz  \lesssim \frac{r}{u}.
\end{align*}
Similarly, for $\Pr(B \cap \mathcal{F}_\tau^+)$, we follow the identical probability derivations to obtain $\Pr(B \cap \mathcal{F}_\tau^+) \lesssim \frac{r}{u}$.

Using Bernstein inequality in Lemma \ref{lemma:Bernstein}, to bound $|A \cap \mathcal{F}_\tau^+|$ and $|B \cap \mathcal{F}_\tau^+|$. Suppose $M_i = \mathbf{1}(i \in \epsilon_t^{z+\zeta}, i \notin \epsilon_t^z, i \in \mathcal{F}_\tau^+)$ and $N_i = \mathbf{1}(i \in \epsilon_t^z, i \notin \epsilon_t^{z-\zeta}, i \in \mathcal{F}_\tau^+)$. We have $\mathbb{E}[M_i] = \Pr(M_i) \lesssim \frac{r}{u}$, $\text{Var}(M_i) \leq \mathbb{E}[M_i^2] = \mathbb{E}[M_i] \lesssim \frac{r}{u}$ (similarly for $\mathbb{E}[N_i]$ and $\text{Var}(N_i)$), then with high probability at least $1-\delta$, let $\delta=\frac{1}{d}$:
\begin{align*}
    |A \cap \mathcal{F}_\tau^+| &= \left|\sum_{i=1}^L M_i\right| \lesssim \sqrt{L\cdot \text{Var}(M_i)\cdot \log d} + \log d + \frac{rL}{u}
    \lesssim \sqrt{\frac{rL}{u} \log d} + \log d + \frac{rL}{u}.
\end{align*}
Finally, under the assumption that the prompt length $L$ is large such that $\frac{rL}{u} = \Omega(\log d)$ given $L=\Theta\left(\text{Poly}(d)\right)$ and $\frac{r}{u}=\Theta(\frac{1}{\text{Poly}(d)})$, the $\log d$ term is dominated by the remaining terms. Thus, we conclude that 
\begin{align}\label{eq:AandB}
	|A \cap \mathcal{F}_\tau^+| \lesssim \sqrt{\frac{rL}{u} \log d}+\frac{rL}{u},\quad
	|B \cap \mathcal{F}_\tau^+| \lesssim \sqrt{\frac{rL}{u} \log d}+\frac{rL}{u}.
\end{align}

Furthermore, we derive
\begin{align*}
	\left|\Pr(A \cap \mathcal{F}_\tau^+) - \Pr(B \cap \mathcal{F}_\tau^+)\right|
    = \mathbb{E}_{Y_2}\left[\mathbf{1}(Y_2 < 0)\Pr\left[Y_3 \geq -Y_2\mid Y_2\right] |\Pr\left[Y_4\leq Y_2-\beta \kappa\mid Y_2\right] - \Pr\left[Y_4\leq -Y_2-\beta \kappa\mid -Y_2\right]|\right].
\end{align*}
With $Y_2 <0$, and $(Y_4 \mid Y_2 = z) \sim \mathcal{N}\left( (1-\beta^2)z, \, \beta^2 \sigma_{t,\tau}^2 \right)$, we express the inner difference $\Delta P_{Y_4}$ using the standard Gaussian cumulative distribution function $\Phi$:
\begin{align*}
    \Delta P_{Y_4}
    = \left| \Phi\left( \frac{Y_2 - \beta\kappa - (1-\beta^2)Y_2}{\beta\sigma_{t,\tau}} \right) - \Phi\left( \frac{-Y_2 - \beta\kappa - (-(1-\beta^2)Y_2)}{\beta\sigma_{t,\tau}} \right) \right| 
    \lesssim \frac{|Y_2|}{\sigma_{t,\tau}}.
\end{align*}
Since $\sigma_3 = \tau_0 r(u+r) \ll \sigma_2 = \tau_0 u(u+r)$, we have
\begin{align*}
	\left|\Pr(A \cap \mathcal{F}_\tau^+) - \Pr(B \cap \mathcal{F}_\tau^+)\right| \lesssim \mathbb{E}_{Y_2}\left[\mathbf{1}(Y_2 < 0) \Pr\left[Y_3 \geq -Y_2\mid Y_2\right] \frac{|Y_2|}{\sigma_{t,\tau}}\right] 
	\lesssim \frac{r^2}{u^2\sqrt{\eta_1 \lambda(t-\tau)}}.
\end{align*}
Using Bernstein inequality in Lemma \ref{lemma:Bernstein}, to bound $\left||A \cap \mathcal{F}_\tau^+| - |B \cap \mathcal{F}_\tau^+|\right|$. Suppose $M_i = \mathbf{1}(i \in \epsilon_t^{z+\zeta}, i \notin \epsilon_t^z, i \in \mathcal{F}_\tau^+)$, $N_i = \mathbf{1}(i \in \epsilon_t^z, i \notin \epsilon_t^{z-\zeta}, i \in \mathcal{F}_\tau^+)$, $W_i = M_i - N_i$. We know that $|\mathbb{E}[W_i]|= \left|\Pr(M_i) - \Pr(N_i)\right| \lesssim \frac{r^2}{u^2\sqrt{\eta_1 \lambda(t-\tau)}}$ and $\text{Var}(W_i) \leq \mathbb{E}[W_i^2] = \mathbb{E}[|W_i|] = \mathbb{E}[M_i+N_i] =\Pr(M_i)+\Pr(N_i) \lesssim \frac{r}{u}$.
With high probability at least $1-1/d$,
\begin{align*}
    \left||A \cap \mathcal{F}_\tau^+| - |B \cap \mathcal{F}_\tau^+|\right| &= \left|\sum_{i=1}^L W_i \right|
    \lesssim \sqrt{L \cdot \text{Var}(W_i) \cdot \log d} + \log d + \frac{ r^2L}{u^2\sqrt{\eta_1 \lambda(t-\tau)}} \\
    &\lesssim  \sqrt{L \frac{r}{u} \log d} + \log d + \frac{ r^2L}{u^2\sqrt{\eta_1 \lambda(t-\tau)}}.
\end{align*}
Finally, under the assumption that the prompt length $L$ is large such that $\frac{rL}{u} = \Omega(\log d)$ given $L=\Theta\left(\text{Poly}(d)\right)$ and $\frac{r}{u}=\Theta(\frac{1}{\text{Poly}(d)})$, the $\log d$ term is dominated by the first term. Thus, we conclude that 
\begin{align}\label{eq:AminusB}
    \left||A \cap \mathcal{F}_\tau^+| - |B \cap \mathcal{F}_\tau^+|\right| \lesssim \sqrt{\frac{rL}{u} \log d} + \frac{ r^2L}{u^2\sqrt{\eta_1 \lambda(t-\tau)}}.
\end{align}
\textbf{For $|A \cap \mathcal{F}_\tau^-|$, $|B \cap \mathcal{F}_\tau^-|$ and $\left||A \cap \mathcal{F}_\tau^-| - |B \cap \mathcal{F}_\tau^-|\right|$.}\quad
Similar to the above part, we have
\begin{align*}
	|A \cap \mathcal{F}_\tau^-| \lesssim \sqrt{\frac{rL}{u} \log d}+\frac{rL}{u},\ 
	|B \cap \mathcal{F}_\tau^-| \lesssim \sqrt{\frac{rL}{u} \log d}+\frac{rL}{u}, \ 
	\left||A \cap \mathcal{F}_\tau^-| - |B \cap \mathcal{F}_\tau^-|\right| \lesssim \sqrt{\frac{rL}{u} \log d} + \frac{ r^2L}{u^2\sqrt{\eta_1 \lambda(t-\tau)}}.
\end{align*}
\textbf{For $|A \cap \mathcal{F}_\tau^c|$ and $|B \cap \mathcal{F}_\tau^c|$.}\quad With $(Y_4 \mid Y_2 = z) \sim \mathcal{N}\left( (1-\beta^2)z, \beta^2 \sigma_{t,\tau}^2 \right)$ and standard Gaussian cumulative distribution function $\Phi$,
\begin{align*}
    \Pr\left[Y_2 - \beta\kappa \leq Y_4 \leq Y_2 + \beta\kappa \mid Y_2\right]
    = \Phi\left( \frac{Y_2 + \beta\kappa - (1-\beta^2)Y_2}{\beta\sigma_{t,\tau}} \right) - \Phi\left( \frac{Y_2 - \beta\kappa - (1-\beta^2)Y_2}{\beta\sigma_{t,\tau}} \right) 
    \lesssim \frac{\kappa}{\sigma_{t,\tau}}.
\end{align*}
Since $\sigma_3 = \tau_0 r(u+r) \ll \sigma_2 = \tau_0 u(u+r)$, we have
\begin{align*}
    \Pr[i\in\epsilon_{t}^{z+\zeta},i\notin\epsilon_{t}^{z},i\in\mathcal{F}_{\tau}^{c}]
	&=\mathbb{E}_{Y_2}\left[\Pr[Y_3\geq -Y_2 \mid Y_2]\mathbf{1}(Y_2<0)\Pr\left[Y_2-\beta\kappa \leq Y_4 \leq Y_2+\beta \kappa \mid Y_2 \right]\right] \\
	&\lesssim \int_{-\infty}^0 \frac{1}{\sqrt{2\pi}\sigma_2} e^{-\frac{z^2}{2\sigma_2^2}} e^{-\frac{z^2}{2\sigma_3^2}}\frac\kappa{\sigma_{t,\tau}} dz \\
    &\lesssim \frac{r^2 \sqrt{\log d}}{u^2 \sqrt{\eta_1 \lambda(t-\tau)}} + \frac{K_l r(u+r)^3}{2\lambda \tau_0 u^2 \sqrt{\eta_1 \lambda(t-\tau)}}
    \triangleq p_\tau.
\end{align*}
Using Bernstein inequality in Lemma \ref{lemma:Bernstein}, with high probability at least $1-\delta$ and let $\delta = \frac{1}{d}$, we have
\begin{align*}
    |A \cap \mathcal{F}_\tau^c| = \left|\sum_{i=1}^L M_i \right| \lesssim  \sqrt{L \cdot \text{Var}(M_i) \cdot \log d} + \log d  + L \cdot \mathbb{E}[M_i]
    \lesssim  \sqrt{L p_\tau \log d} + \log d + L p_\tau.
\end{align*}
These are similar for $|B \cap \mathcal{F}_\tau^c|$. 
For $\frac{rL}{u} = \Omega(\log d)$ given $L=\Theta\left(\text{Poly}(d)\right)$ and $\frac{r}{u}=\Theta(\frac{1}{\text{Poly}(d)})$, we conclude that
\begin{align}\label{eq:AandBandFc}
    |A \cap \mathcal{F}_\tau^c| \lesssim \sqrt{L p_\tau \log d} + L p_\tau,\quad |B \cap \mathcal{F}_\tau^c| \lesssim \sqrt{L p_\tau \log d} + L p_\tau.
\end{align}
Combining Equation~\ref{eq:AandB}, \ref{eq:AminusB}, and \ref{eq:AandBandFc}, Equation~\ref{eq:A-B-F} finally becomes
\begin{align*}
    \text{LHS}
	\lesssim& \frac{K_l}{2}(u+r)^3 \Bigg[2P\left(\sqrt{\frac{rL}{u} \log d} + \frac{ r^2L}{u^2\sqrt{\eta_1 \lambda(t-\tau)}}\right)
    + 4(1-P)\left(\sqrt{\frac{rL}{u} \log d}+\frac{rL}{u} \right) + 2(\sqrt{L p_\tau \log d}+L p_\tau) \Bigg]\\
    \lesssim& K_l(u+r)^3 \Bigg[ \frac{rL}{u} + \sqrt{\frac{rL}{u} \log d} + L p_\tau + \sqrt{L p_\tau \log d} \Bigg],
\end{align*}
where the last step comes from that the term $\frac{r^2 L}{u^2 \sqrt{\eta_1 \lambda(t-\tau)}}$ is dominated by $L p_\tau$. 
When $t \leq \frac{1}{\eta_1 \lambda}$, define $C_p \triangleq \frac{r^2 \sqrt{\log d}}{u^2} + \frac{K_l r(u+r)^3}{2\lambda \tau_0 u^2}$, which gives $p_\tau \triangleq \frac{C_p}{\sqrt{\eta_1 \lambda(t-\tau)}}$.
\begin{align*}
    &\eta_1 \sum_{\tau=1}^t p_{\tau-1} \leq \frac{\eta_1 C_p}{\sqrt{\eta_1 \lambda}} \int_0^t \frac{1}{\sqrt{x}} dx = 2 C_p \sqrt{\frac{\eta_1 t}{\lambda}} \leq 2 C_p \lambda^{-1}, \\
    &\eta_1 \sum_{\tau=1}^t \sqrt{p_{\tau-1} \log d} \leq \frac{\eta_1 \sqrt{C_p \log d}}{(\eta_1 \lambda)^{1/4}} \int_0^t x^{-1/4} dx = \frac{4}{3}\sqrt{C_p \log d} \frac{(\eta_1 t)^{3/4}}{\lambda^{1/4}} \leq \frac{4}{3} \sqrt{C_p \log d}\cdot \lambda^{-1}.
\end{align*}
Thus,
\begin{align*}
	\left|\left(\mathbbm{1}\left(\epsilon^{X,z-\zeta}_t\right) + \mathbbm{1}\left(\epsilon^{X,z+\zeta}_t\right) -2\mathbbm{1}\left(\epsilon^{X,z}_t\right) \right)^\top Q_t z\right|
	\leq& \eta_1 u \sum_{\tau=1}^t \left\|\left(\mathbbm{1}\left(\epsilon^{X,z-\zeta}_t\right) + \mathbbm{1}\left(\epsilon^{X,z+\zeta}_t\right)-2\mathbbm{1}\left(\epsilon^{X,z}_t\right)\right)^\top \Delta Q_{\tau-1} \right\|_2 \\
	\lesssim& \lambda^{-1} K_l u (u+r)^3 \Bigg[ \frac{rL}{u} + \sqrt{\frac{rL}{u} \log d} + L C_p + \sqrt{L C_p \log d} \Bigg].
\end{align*}
With Assumption~\ref{ass:choice-hyperparam}, we have $C_p = \Theta\left(\frac{rd}{u}\right).$ Thus $L C_p$ and $\sqrt{L C_p \log d}$ dominate the terms $\frac{rL}{u}$ and $\sqrt{\frac{rL}{u} \log d}$. By condensing the dominated terms, we conclude that term $\Phi$ is bounded by:
\begin{align}
    \left|\left(\mathbbm{1}\left(\epsilon^{X,z-\zeta}_t\right) + \mathbbm{1}\left(\epsilon^{X,z+\zeta}_t\right) -2\mathbbm{1}\left(\epsilon^{X,z}_t\right) \right)^\top Q_t z\right|
    \lesssim \lambda^{-1} K_l u (u+r)^3 \left( L C_p + \sqrt{L C_p \log d} \right). \label{eq:Phi}
\end{align}

\textbf{Combine term $\Psi$ and term $\Phi$.}\quad
Combining Equation~\ref{eq:Psi} and \ref{eq:Phi},
\begin{align*}
	&\left|\widetilde{g}_t (X, z-\zeta) + \widetilde{g}_t (X, z+\zeta) - 2 \widetilde{g}_t (X, z)\right| \\
	\leq& \frac{1}{L} \left|\left(\mathbbm{1}\left(\epsilon^{z-\zeta}_t\right) + \mathbbm{1}\left(\epsilon^{z+\zeta}_t\right) -2\mathbbm{1}\left(\epsilon^{z}_t\right) \right)^\top Q_t z\right| + \frac{1}{L} \left|\left(\mathbbm{1}\left(\epsilon^{z+\zeta}_t\right)-\mathbbm{1}\left(\epsilon^{z-\zeta}_t\right)\right)^\top Q_t\zeta\right| \\
	\lesssim& \lambda^{-1} L^{-1} K_l u(u+r)^3 \left( L C_p + \sqrt{L C_p \log d} \right) + \lambda^{-1} L^{-1} K_l r(u+r)^3 \sqrt{\log d} \left( \sqrt{L} + \frac{Lr(u+r)}{u(u-r)} \right)\\
    \lesssim& \sqrt{d} \left( \frac{r d}{u} + \sqrt{\frac{r d \log d}{u L}} \right) + r \sqrt{d \log d} \left( \frac{1}{\sqrt{L}} + \frac{r}{u} \right) \\
    =&\frac{r}{u} d^{3/2} + d \sqrt{\frac{r \log d}{u L}} + r \sqrt{\frac{d \log d}{L}} + \frac{r^2}{u} \sqrt{d\log d},
\end{align*}
which comes from Assumption~\ref{ass:choice-hyperparam} and $\frac{r}{u}=\Theta(1/\text{Poly}(d)), u+r=\Theta(1)$, $C_p = \Theta\left(\frac{rd}{u}\right)$, 

With a large polynomial degree of $L$ and $r=\Theta(d^{-9/4})$, the dominant term $\frac{r}{u} d^{3/2}$ converges to $0$ at a polynomial rate.
\begin{align}\label{eq:tildeg-minus}
    \left|\widetilde{g}_t (X, z-\zeta) + \widetilde{g}_t (X, z+\zeta) - 2 \widetilde{g}_t (X, z)\right| = \mathcal{O}\left(\frac{1}{\text{Poly}(d)}\right).
\end{align}

\textbf{Deal with $|g_{t_1}(X_2)|$.}\quad
With Equation~\ref{eq:g-tildeg} and \ref{eq:tildeg-minus}, $\left|g_{t_1} (X_2, z-\zeta) + g_{t_1} (X_2, z+\zeta) - 2 g_{t_1} (X_2, z)\right| \lesssim 1/\text{Poly}(d) \triangleq \xi_g$. Then
\begin{align*}
	g_{t_1}(X_2,z) = \frac{1}{2}\left(g_{t_1}(X_2,z+\zeta)+g_{t_1}(X_2,z-\zeta)\right) + \gamma,
\end{align*}
where the residual term $|\gamma| \leq \frac{1}{2}\xi_g$.

By Hoeffding's inequality, the empirical proportions of $x_{L,2} \in \{z-\zeta, z+\zeta, z\}$ in the training sequence are lower-bounded by $\frac{1}{4}-\epsilon_N$, $\frac{1}{4}-\epsilon_N$, and $\frac{1}{2}-\epsilon_N$ respectively with high probability, where $\epsilon_N = \mathcal{O}(\sqrt{\log d/N})$. 
Given the true labels $y=-1$ for $z\pm\zeta$ and $y=1$ for $z$, denote $g_x \triangleq g_{t_1}(X_2, x)$, then the surrogate loss satisfies
\begin{align*}
	K^2_{t_1}(\overline{V}_{t_1}) \ge \left(\frac{1}{4} - \epsilon_N\right)\left[l\left(-\frac{1}{2}g_{z+\zeta}\right) + l\left(-\frac{1}{2}g_{z-\zeta}\right)\right] + \left(\frac{1}{2} - \epsilon_N\right) l\left(\frac{1}{2}g_z\right).
\end{align*}

Since $l(\cdot)$ is convex, we define a non-negative Jensen's gap $A$:
\begin{align*}
    A \triangleq l(-\frac{1}{2}g_{z+\zeta}) + l(-\frac{1}{2}g_{z-\zeta}) - 2l\left(-\frac{1}{2}(\frac{1}{2}g_{z+\zeta}+\frac{1}{2}g_{z-\zeta})\right) =l(-\frac{1}{2}g_{z+\zeta}) + l(-\frac{1}{2}g_{z-\zeta}) - 2l\left(-\frac{1}{2}g_z + \frac{1}{2}\gamma\right) \ge 0,
\end{align*}
i.e., $l(-\frac{1}{2}g_{z+\zeta}) + l(-\frac{1}{2}g_{z-\zeta}) = A + 2l\left(-\frac{1}{2}g_z + \frac{1}{2}\gamma\right)$.

Since $l(\cdot)$ is 1-Lipschitz continuous, we have $l(-\frac{1}{2}g_z + \frac{1}{2}\gamma) \ge l(-\frac{1}{2}g_z) - \left|\frac{1}{2}\gamma\right|$. Substituting this back and denoting $\tilde{l}(g_z)=\frac{1}{2}\left(l(\frac{1}{2}g_z)+l(-\frac{1}{2}g_z)\right)$,
\begin{align}\label{eq:L-V-lower-bound}
	K^2_{t_1}(\overline{V}_{t_1}) \ge \left(\frac{1}{4} - \epsilon_N\right) \left(A+2l(-\frac{1}{2}g_z) - |\gamma|\right) + \left(\frac{1}{2} - \epsilon_N\right) l(g_z) \ge \left(\frac{1}{4} - \epsilon_N\right)A + (1 - 4\epsilon_N)\tilde{l}(g_z) - \frac{1}{8}\xi_g.
\end{align}
With logistic loss $l \geq \log 2$, and then $\tilde{l} \geq \log 2$, which establishes the lower bound:
\begin{align}\label{eq:stage1-K-V}
    K^2_{t_1}(\overline{V}_{t_1})\gtrsim\log 2  - \frac{1}{\text{Poly}(d)} - \sqrt{\frac{\log d}{N}}.
\end{align}
This indicates that the network $g$ remains unoptimized during this phase. According to the definition of $g$ (where the normalization factor $1/L$ with $L = \Theta(\text{Poly}(d))$), we have $K^2_0(\overline{V}_0) \lesssim \log 2 + 1/\text{Poly}(d)$ by Taylor expansion. Since gradient-based optimization intrinsically minimizes the loss, it naturally follows that $K^2_{t_1}(\overline{V}_{t_1}) \le \log 2 + \xi^\prime$, where $\xi^\prime$ is a negligible optimization fluctuation with an order of $\xi^\prime = \mathcal{O}(1/\text{Poly}(d))$.
Together with Equation~\ref{eq:L-V-lower-bound}, we have
\begin{align*}
    \left(\frac{1}{4} - \epsilon_N\right)A + (1 - 4\epsilon_N) \tilde{l}(g_z) - \frac{1}{8}\xi_g &\leq \log2 + \xi^\prime\\
    \left(\frac{1}{4} - \epsilon_N\right)A + (1 - 4\epsilon_N) (\tilde{l}(g_z) - \log 2) &\leq \xi^\prime + \frac{1}{8}\xi_g + 4\epsilon_N \log 2.
\end{align*}
With $\epsilon_N = \mathcal{O}\left(\sqrt{\frac{\log d}{N}}\right)$ and $N=\Theta(\text{Poly}(d))$, we have $\epsilon_N = o(1)$ and $\frac{1}{4} - \epsilon_N = \Theta(1)$, $1 - 4\epsilon_N = \Theta(1)$. The right-hand side can be consolidated into a unified asymptotic bound: $\xi^\prime + \frac{1}{8}\xi_g + 4\epsilon_N \log 2 = \mathcal{O}(\xi^\prime + \xi_g + \epsilon_N)$.
Thus, since $(\frac{1}{4} - \epsilon_N)A \ge 0$ and $\tilde{l}(g_z)-\log 2 \geq 0$, we obtain:
\begin{align*}
    A = \mathcal{O}(\xi^\prime + \xi_g + \epsilon_N),\quad \tilde{l}(g_z) - \log 2 = \mathcal{O}(\xi^\prime + \xi_g + \epsilon_N). 
\end{align*}
Let $u=\frac{1}{2}g_{z+\zeta}$ and $v=\frac{1}{2}g_{z-\zeta}$. Since the empirical loss is bounded, the logits $u$ and $v$ are within a bounded interval $[-M, M]$. Over $[-M, M]$, the logistic loss $l(\cdot)$ is $\mu$-strongly convex, where $\mu > 0$ (depending on $M$). Thus, the Jensen's gap $A$ satisfies 
$$
A = l(-u) + l(-v) - 2l\left(\frac{-u-v}{2}\right) \ge \frac{\mu}{4} (-u - (-v))^2 = \frac{\mu}{4} (u-v)^2.
$$
Combining with $A \lesssim \xi^\prime + \xi_g + \epsilon_N$, we have 
\begin{align*}
    (u-v)^2 \lesssim \xi^\prime + \xi + \epsilon_N\ 
    \Rightarrow |u-v| \lesssim \sqrt{\xi^\prime + \xi_g + \epsilon_N}.
\end{align*}
As established, $\tilde{l}(g_z) = \frac{1}{2}l(\frac{1}{2}g_z) + \frac{1}{2}l(-\frac{1}{2}g_z)$ achieves its global minimum at $g_z = 0$.
By applying the $\mu$-strong convexity of $\tilde{l}(\cdot)$ around its minimum $g_z = 0$,
\begin{align*}
    \tilde{l}(g_z) &\ge \tilde{l}(0) + \tilde{l}'(0)(g_z - 0) + \frac{\mu}{2}(g_z - 0)^2
    \geq \log 2 + \frac{\mu}{2} g_z^2.
\end{align*}
Combining with $\tilde{l}(g_z) - \log 2 \lesssim \xi^\prime + \xi_g + \epsilon_N$, we have
\begin{align*}
    g_z^2 \lesssim \xi^\prime + \xi_g + \epsilon_N\ 
    \Rightarrow |g_z|=|g_{t_1}(X_2,z)| \lesssim \sqrt{\xi^\prime + \xi_g + \epsilon_N}.
\end{align*}
By triangle inequality $u = \frac{u+v}{2} + \frac{u-v}{2} = \frac{1}{2}(g_z - \gamma) + \frac{u-v}{2}$, 
\begin{align*}
    |u| &= \frac{1}{2}|g_{t_1}(X_2,z+\zeta)| = \left|\frac{1}{2}g_z - \frac{1}{2}\gamma + \frac{u-v}{2}\right| \\
    &\lesssim \frac{1}{2}\sqrt{\xi^\prime + \xi_g + \epsilon_N}+\frac{1}{4}\xi_g + \frac{1}{2}\sqrt{\xi^\prime + \xi_g + \epsilon_N}.
\end{align*}
Because $\xi_g=\Theta(\frac{1}{\text{Poly(d)}})$ is sufficiently small, the linear term is dominated asymptotically by the square root. Thus, $|u| \lesssim \sqrt{\xi^\prime + \xi_g + \epsilon_N}$. By symmetry, the same bound holds for $|v| = \frac{1}{2}|g_{t_1}(X_2,z-\zeta)|$.
In total, we have
\begin{align}\label{eq:g-bound}
    |g_{t_1}(X_2,z)|, |g_{t_1}(X_2,z-\zeta)|, |g_{t_1}(X_2,z+\zeta)| \lesssim \sqrt{\xi^\prime + \xi_g + \epsilon_N} \lesssim \sqrt{\frac{1}{\text{Poly(d)}} + \sqrt{\frac{\log d}{N}}},
\end{align}
which is of $(\frac{\log d}{N})^{1/4}$ order with $N=\Theta(\text{Poly}(d))$.

\textbf{Deal with $\|\overline{V}_{t_1}\|_F$.}\quad
We now evaluate the empirical gradient $\nabla_{\overline{V}} K_{t_1}(\overline{U}_{t_1})$. Let $f_{t_1} \triangleq f(U_{t_1}; X, \widetilde{Y})$ denote the global network outputs. Let $g_{0} \triangleq g(\widetilde{V}_0; X_2,Y)$ and $g_{t_1} \triangleq g(V_{t_1}; X_2,Y)$ denote the local non-linear outputs such that $\nabla_V f = \nabla_V g$.
\begin{align}
    \|\nabla_{\overline{V}} K_{t_1}(\overline{U}_{t_1})\|_F
    =  \left\| \widehat{\mathbb{E}}\left[ l^\prime(f_{t_1}) \frac{1}{2} \nabla_V g_{t_1}\right] \right\|_F 
    \le \underbrace{\left\| \widehat{\mathbb{E}}\left[ l^\prime(f_{t_1}) \frac{1}{2}\left( \nabla_V g_{t_1} - \nabla_V g_0 \right)\right] \right\|_F}_{\text{Term A}} + \underbrace{\left\| \widehat{\mathbb{E}}\left[ l^\prime(f_{t_1})\frac{1}{2}\nabla_V g_0 \right]\right\|_F}_{\text{Term B}}. \label{eq:grad-decomp}
\end{align}

\textbf{For Term A}, using Corollary \ref{coro:lemma-ac-V} and $V_0 \approx \widetilde{V}_0$, the activation difference $\Delta_j \triangleq \mathbbm{1}([X_2^\top]_j V_{t_1} x_{L,2}) - \mathbbm{1}([X_2^\top]_j \widetilde{V}_0 x_{L,2})$ satisfies $\|\Delta\|_1 \lesssim \epsilon_V = \Theta(d^4 L^{2/3})$. With $|l^\prime(\cdot)| \le K_l$,
\begin{align*}
    \text{Term A} \le \widehat{\mathbb{E}}\left[ \frac{K_l}{2L} \left\| \sum_{j=1}^L y_j \Delta_j ([X_2^\top]_j)^\top x_{L,2}^\top \right\|_F \right] \le \widehat{\mathbb{E}}\left[ \frac{K_l(u+r)}{2L} \left\| \sum_{j=1}^L y_j \Delta_j ([X_2^\top]_j)^\top \right\|_2 \right],
\end{align*}
where we decouple the rank-1 outer product norm via $\|ab^\top\|_F = \|a\|_2\|b\|_2$ and appliy $\|x_{L,2}\|_2 \le u+r$.

Since the labels $y_j$ are independent, we apply Hoeffding's inequality in Lemma~\ref{lemma:Hoeffding} coordinate-wise. Taking a union bound over the $d$ coordinates, we obtain with high probability $1-1/d$:
\begin{align*}
    \left\| \sum_{j=1}^L y_j \Delta_j ([X_2^\top]_j)^\top \right\|_2 \lesssim \sqrt{ \log d \sum_{j=1}^L \Delta_j^2 \left\| [X_2^\top]_j \right\|_2^2 }.
\end{align*}
Since $\Delta_j \in \{-1, 0, 1\}$, $\Delta_j^2 = |\Delta_j|$. Thus, the summation in the upper bound is bounded by $(u+r)^2 \|\Delta\|_1 \lesssim (u+r)^2 \epsilon_V$. We obtain
\begin{align*}
    \text{Term A} \lesssim K_l(u+r)^2 L^{-1}\sqrt{\epsilon_V \log d} = \mathcal{O}\left( d^2 L^{-2/3} \sqrt{\log d} \right).
\end{align*}

\textbf{For Term B}, we express the empirical expectation explicitly over the $N$ independent samples. 
Crucially, the global output decomposes as $f_{t_1}^n = 1/2 h_{t_1}^n + 1/2 g_{t_1}^n$, where $h_{t_1}^n$ depends on the linear features $X_1^n$ and $g_{t_1}^n$ depends on the non-linear features $X_2^n$.
We split Term B accordingly:
\begin{align*}
    \text{Term B} \le \underbrace{\left\| \frac{1}{N}\sum_{n=1}^N l^\prime\left(\frac{1}{2}h_{t_1}^{n}\right) \frac{1}{2}\nabla_V g_0^{n} \right\|_F}_{\text{Term B}_1} + \underbrace{\left\| \frac{1}{N}\sum_{n=1}^N \left( l^\prime(f_{t_1}^{n}) - l^\prime\left(\frac{1}{2}h_{t_1}^{n}\right) \right)\frac{1}{2} \nabla_V g_0^{n} \right\|_F}_{\text{Term B}_2}.
\end{align*}

\textbf{For $\text{Term B}_1$}, let $C_n \triangleq l^\prime(\frac{1}{2}h_{t_1}^{n}) \frac{1}{2}\nabla_V g_0^{n}$. 
With Assumption~\ref{ass:x-norm} and $|l^\prime(\cdot)| \le K_l$, we have:
\begin{align*}
    \|C_n\|_F \le \frac{K_l}{2L} \sum_{j=1}^L \left\| \mathbbm{1}([X_2^\top]_j \widetilde{V}_0 x_{L,2}) ([X_2^\top]_j)^\top x_{L,2}^\top \right\|_F \le \frac{K_l}{2}(u+r)^2 \triangleq M.
\end{align*}

We now evaluate $\|\mathbb{E}[C_n]\|_F$. Let $\mathbf{y}^n \triangleq \{y_1^n, \dots, y_L^n\}$. By our data construction, $X_1^n$ and $X_2^n$ are conditionally independent given $\mathbf{y}^n$. Substitute the logistic loss derivative ($\sigma$ is the sigmoid function) and $\nabla_V g_0^n$. Notably, the attention output only aggregates over the context tokens $j \in [L-1]$. Let $G(x, x') \triangleq \mathbbm{1}\left(x^\top \widetilde{V}_0 x'\right) x (x')^\top$, $\mathbb{E}[C_n]$ becomes:
\begin{align}
    \mathbb{E}[C_n] &= \mathbb{E}_{X_1^n, \mathbf{y}^n} \left[ -\sigma\left(-\frac{1}{2} y_L^n h_{t_1}^n\right) \frac{1}{2L} \sum_{j=1}^{L-1} \mathbb{E}_{X_2^n \mid \mathbf{y}^n}\left[ y_j^n G(x_{j,2}^n, x_{L,2}^n) \right] \right]\nonumber\\
    &= \mathbb{E}_{X_1^n, y_L^n} \left[ -\sigma\left(-\frac{1}{2} y_L^n h_{t_1}^n\right) \frac{1}{2L} \sum_{j=1}^{L-1} \mathbb{E}_{\mathbf{y}_{-L}^n} \mathbb{E}_{X_2^n \mid \mathbf{y}^n}\left[ y_j^n G(x_{j,2}^n, x_{L,2}^n) \right] \right]\label{eq:C_n_bound}.
\end{align}
In the last step, we decouple the query label $y_L^n$ from the remaining context labels, denoted as $\mathbf{y}_{-L}^n \triangleq \mathbf{y}^n \setminus \{y_L^n\}$.

With $\left|-\sigma(\cdot)\right| \le 1$, Equation~\ref{eq:C_n_bound} simplifies to:
\begin{align}
    \|\mathbb{E}[C_n]\|_F \le \mathbb{E}_{y_L^n} \left[ \frac{1}{2L} \sum_{j=1}^L \left\| \mathbb{E}_{\mathbf{y}_{-L}^n}\mathbb{E}_{ X_2^n \mid \mathbf{y}^n} \left[ y_j^n G(x_{j,2}^n, x_{L,2}^n)  \right] \right\|_F \right].\label{eq:C_n_Fnorm}
\end{align}
Within the norm, for any token $j\neq L$, $G(x_{j,2}^n, x_{L,2}^n)$ solely depends on $x_j$ and $x_L$. Conditioned on their labels $y_j^n$ and $y_L^n$, these features are conditionally independent of all other labels. Thus,
\begin{align*}
    \mathbb{E}_{\mathbf{y}_{-L}^n}\mathbb{E}_{ X_2^n \mid \mathbf{y}^n} \left[ y_j^n G(x_{j,2}^n, x_{L,2}^n) \right]
    =\mathbb{E}_{x_{L,2}^n \mid y_L^n}\mathbb{E}_{y_j^n}\mathbb{E}_{ x_{j,2}^n\mid y_j^n}\left[ y_j^n G(x_{j,2}^n, x_{L,2}^n) \right].
\end{align*}
Equation~\ref{eq:C_n_Fnorm} becomes
\begin{align}
    \|\mathbb{E}[C_n]\|_F
    \le \frac{1}{2L} \mathbb{E}_{x_{L,2}^n, y_L^n} \left( \sum_{j \neq L} \left\| \mathbb{E}_{y_j^n} \left[ \mathbb{E}_{x_{j,2}^n \mid y_j^n} \left[ y_j^n G(x_{j,2}^n, x_{L,2}^n) \right] \right] \right\|_F + \left\| y_L^n G(x_{L,2}^n, x_{L,2}^n) \right\|_F \right)\label{eq:C_n_Fnorm_2}.
\end{align}
If $y_j^n = 1$, $\mathbb{E}_{x_{j,2}^n\mid y_j^n=1} \left[G(x_{j,2}^n, x_{L,2}^n) \right] = G(z, x_{L,2}^n)$. If $y_j^n = -1$, $\mathbb{E}_{x_{j,2}^n\mid y_j^n=-1} \left[-G(x_{j,2}^n, x_{L,2}^n) \right]= -\frac{1}{2}G(z-\zeta, x_{L,2}^n) - \frac{1}{2}G(z+\zeta, x_{L,2}^n)$.
Thus,
\begin{align}
     \mathbb{E}_{y_j^n} \left[ \mathbb{E}_{x_{j,2}^n \mid y_j^n} \left[ y_j^n G(x_{j,2}^n, x_{L,2}^n) \right] \right] = -\frac{1}{4} \left[ G(z-\zeta, x_{L,2}^n) + G(z+\zeta, x_{L,2}^n) - 2G(z, x_{L,2}^n) \right] \label{eq:G-difference},
\end{align}
where we reconstruct a second-order finite difference, operating on the context feature rather than the query feature as in Equation~\ref{eq:tildeg-minus}.
Let $I_x \triangleq \mathbbm{1}\left(x^\top \widetilde{V}_0 x_{L,2}^n\right) \in \{0, 1\}$ denote the scalar indicator state for a context feature $x$. The difference in Equation~\ref{eq:G-difference} is expanded as $[ (I_{z-\zeta} + I_{z+\zeta} - 2I_z) z + (I_{z+\zeta} - I_{z-\zeta}) \zeta ] (x_{L,2}^n)^\top$. Thus,
\begin{align*}
    \left\| \mathbb{E}_{y_j^n} \Big[ \mathbb{E}_{x_{j,2}^n \mid y_j^n} \left[ y_j^n G(x_{j,2}^n, x_{L,2}^n) \right] \Big] \right\|_F 
    =& \frac{1}{4} \left\| \left[ (I_{z-\zeta} + I_{z+\zeta} - 2I_z) z + (I_{z+\zeta} - I_{z-\zeta}) \zeta \right] (x_{L,2}^n)^\top \right\|_F  \\
    \leq& \frac{1}{4} \left(|I_{z-\zeta} + I_{z+\zeta} - 2I_z| \cdot \|z\|_2 + |I_{z+\zeta} - I_{z-\zeta}|\cdot  \|\zeta\|_2\right) \|x_{L,2}^n\|_2   \\
    \leq& \frac{1}{4} (u+r) \left(u |I_{z-\zeta} + I_{z+\zeta} - 2I_z|  + r |I_{z+\zeta} - I_{z-\zeta}|\right).
\end{align*}
For $|I_{z+\zeta} - I_{z-\zeta}|$. This term takes its maximum value of $1$ if and only if the two indicators output opposite states, i.e., $(z+\zeta)^\top \widetilde{V}_0 x_{L,2}^n \ge 0, (z-\zeta)^\top \widetilde{V}_0 x_{L,2}^n < 0$ or $(z+\zeta)^\top \widetilde{V}_0 x_{L,2}^n < 0, (z-\zeta)^\top \widetilde{V}_0 x_{L,2}^n \geq 0$. This implies
\begin{align*}
    -|\zeta^\top \widetilde{V}_0 x_{L,2}^n| \leq z^\top \widetilde{V}_0 x_{L,2}^n \leq |\zeta^\top \widetilde{V}_0 x_{L,2}^n| \implies
    |z^\top \widetilde{V}_0 x_{L,2}^n| \leq |\zeta^\top \widetilde{V}_0 x_{L,2}^n|.
\end{align*}
With the Gaussian initialization $[\widetilde{V}_0]_{ij} \sim \mathcal{N}(0, \tau_0^2)$, we have $\text{Var}(\zeta^\top \widetilde{V}_0 x_{L,2}^n) \leq \tau_0^2 r^2 (u+r)^2$.
Similarly to Equation~\ref{eq:XVz_lower_bound} and \ref{eq:XVz_upper_bound}, we have w.h.p $1-\delta$:
\begin{align*}
    \Pr \left( |z^\top \widetilde{V}_0 x_{L,2}^n| \lesssim \tau_0 r (u+r) \sqrt{\log d} \right) \ge 1 - \frac{1}{d}, \quad
    \Pr \left( |z^\top \widetilde{V}_0 x_{L,2}^n| \lesssim \tau_0 r (u+r) \sqrt{\log d} \right) \lesssim \frac{r(u+r)}{u(u-r)} \sqrt{\log d}.
\end{align*}
Thus,
$$
\mathbb{E}_{x_{L,2}^n, y_L^n} \left[ |I_{z+\zeta} - I_{z-\zeta}| \right] = \Pr\left( |I_{z+\zeta} - I_{z-\zeta}| = 1 \right) \lesssim \frac{r(u+r)}{u(u-r)} \sqrt{\log d} \triangleq p.
$$
For $|I_{z-\zeta} + I_{z+\zeta} - 2I_z|$.  The indicator differences $|I_{z-\zeta}-I_z|$ and $|I_{z+\zeta}-I_z|$ cannot take $1$ simultaneously. Thus, similarly to $|I_{z+\zeta} - I_{z-\zeta}| $, $\mathbb{E}_{x_{L,2}^n,y_L^n} \big[ |I_{z-\zeta} + I_{z+\zeta} - 2I_z| \big] \lesssim p.$
Therefore, for the first term in Equation~\ref{eq:C_n_Fnorm_2}:
\begin{align*}
    \text{LHS}
    \leq \mathbb{E}_{x_{L,2}^n \mid y_L^n} \left[ \frac{1}{4} (u+r) \left(u |I_{z-\zeta} + I_{z+\zeta} - 2I_z|  + r |I_{z+\zeta} - I_{z-\zeta}|\right) \right] 
    \leq \frac{1}{4} (u+r)^2 p.
\end{align*}

For the second term in Equation~\ref{eq:C_n_Fnorm_2}, i.e., $j = L$, the norm is bounded by $\mathbb{E}_{x_{L,2}^n \mid y_L^n} [ \|x_{L,2}^n\|_2^2] \leq (u+r)^2$.
Combining it with the $j \neq L$ case, we obtain $\|\mathbb{E}[C_n]\|_F \lesssim (u+r)^2 p + \frac{(u+r)^2}{L}$. With $p = \Theta\left(\frac{\sqrt{\log d}}{\text{Poly}(d)}\right)$ and $L = \Theta(\text{Poly}(d))$, we conclude:
\begin{align}\label{eq:C_n_F_norm_final}
    \|\mathbb{E}[C_n]\|_F = \mathcal{O}\left(\frac{\sqrt{\log d}}{\text{Poly}(d)}\right).
\end{align}

To bound $\text{Term B}_1 =\left\| \frac{1}{N} \sum_{n=1}^N C_n \right\|_F \triangleq H$, we first bound its expectation. By the triangle inequality and Jensen's inequality on the concave square root function,
\begin{align*}
    \mathbb{E}[H] \le \|\mathbb{E}[C_n]\|_F + \sqrt{ \frac{1}{N^2} \sum_{n=1}^N \mathbb{E} \left[ \| C_n\|_F^2 \right] } \le \|\mathbb{E}[C_n]\|_F + \frac{M}{\sqrt{N}}.
\end{align*}
Since replacing a single sample $C_i$ changes $H$ by at most $2M/N$ (where $\|C_n\|_F \le M \triangleq \frac{K_l}{2}(u+r)^2$), we apply McDiarmid's inequality in Lemma~\ref{lemma:McDiarmid}. With high probability $1-1/d$, we have:
\begin{align*}
    \text{Term B}_1 \le \mathbb{E}[H] + \sqrt{\frac{2M^2}{N}\log d} &\lesssim \frac{\sqrt{\log d}}{\text{Poly}(d)} +\frac{K_l(u+r)^2}{2\sqrt{N}} + \frac{K_l(u+r)^2\sqrt{\log d}}{2\sqrt{N}},
\end{align*}
which is of $\sqrt{\frac{\log d}{N}}$ order with $N=\Theta(\text{Poly}(d))$.

\textbf{For $\text{Term B}_2$}, with the smoothness of the logistic derivative ($K_{l^\prime}$) and Equation~\ref{eq:g-bound}, we derive that
\begin{align*}
    \text{Term B}_2 \le \frac{1}{N}\sum_{n=1}^N \left| l^\prime(f_{t_1}^{n}) - l^\prime\left(\frac{1}{2}h_{t_1}^{n}\right) \right|\cdot \left\| \frac{1}{2}\nabla_V g_0^{n} \right\|_F 
    \lesssim (u+r)^2 \sqrt{\frac{1}{\text{Poly}(d)} + \sqrt{\frac{\log d}{N}}},
\end{align*}
which is of $(\frac{\log d}{N})^{1/4}$ order with $N=\Theta(\text{Poly}(d))$.

\textbf{Combining Term A, $\text{Term B}_1$, and $\text{Term B}_2$ into Equation~\ref{eq:grad-decomp}},
\begin{align*}
    \left\|\nabla_{\overline{V}} K_{t_1}(\overline{U}_{t_1})\right\|_F
    &\lesssim (\log d)^{1/2} d^{2}L^{-2/3}+\left(\frac{\log d}{N}\right)^{1/2} + \left(\frac{\log d}{N}\right)^{1/4}.
\end{align*}
Given a large polynomial data size $N = \Theta(\text{Poly}(d))$, the statistical errors (the second and third terms) are dominated by the first term.
Finally, we track the optimization trajectory for the specialized signal weight $\overline{V}_{t_1}$. Within the elementary stage duration $t_1 = \Theta(\frac{1}{\eta_1 \lambda})$, the signal weight starts from $\overline{V}_0 = \mathbf{0}$:
\begin{align*}
    \|\overline{V}_{t_1}\|_F \le \eta_1 \sum_{s=1}^{t_1} \left\|\nabla_{\overline{V}} K_{s-1}(\overline{U}_{s-1})\right\|_F
    \le \eta_1 t_1 \max_{1 \le s \le t_1} \left\|\nabla_{\overline{V}} K_{s-1}(\overline{U}_{s-1})\right\|_F 
    \lesssim  (\log d)^{1/2} d^{2}\lambda^{-1} L^{-2/3}.
\end{align*}
With a large prompt length ($L=\Omega(d^{6.75})$), $\|\overline{V}_{t_1}\|_F$ converges asymptotically to zero at a polynomial rate:
\begin{align}
    \|\overline{V}_{t_1}\|_F = \mathcal{O}\left( \frac{1}{\text{Poly}(d)} \right).
\end{align}
This concludes the proof, formally establishing that the signal weight regarding the non-linear features receives no effective learning during the elementary stage.
\end{proof}

\subsection{Proof for the Elementary Stage: Proof of Theorem \ref{the:stage1-P}}\label{proof:stage1-P}

\begin{proof}
According to Theorem \ref{the:stage1-Q}, we conclude that the large learning rate creates too much noise to learn $Q$. Also, from above we conclude that in the first stage, the network weight $V_{t_1}$ on $Q$ changes small. We begin by the following definition.
\begin{definition}\label{def:stage1-Ustar}
	In the elementary stage, denote optimal weight as $U^\star_1 = \begin{bmatrix}
		W^\star & 0 \\ 0 & \overline{V}_{t_1} = \Delta V
	\end{bmatrix}$ with initial signal weight $\overline{W}_0 = \overline{V}_0 =\textbf{0}_{d\times d}$, where $W^\star \triangleq d\log(1/\epsilon_{W,1}) w^\star (w^\star)^\top  \in \mathbb{R}^{d \times d}$, and $\|\overline{V}_{t_1}\|_F \lesssim 1/\text{Poly(d)}$. 
\end{definition}
In this section, we primarily focus on the process of optimizing from $\overline{W}_0$ to $W^\star$. 

\textbf{We first evaluate the surrogate loss at the optimal signal weight $W^\star$.} In this evaluation, the network's non-linear activation states are fixed by the true trajectory state at iteration $t_1$, i.e., $W_{t_1} = \overline{W}_{t_1} + \widetilde{W}_{t_1}$. We consider
\begin{align*}
	h_{t_1}(X_1) 
	= N_{W_{t_1}}(W^\star+\widetilde{W}_t; X_1, Y) 
	= N_{W_{t_1}}(W^\star; X_1, Y) + N_{W_{t_1}}(\widetilde{W}_t; X_1, Y).
\end{align*}
According to the data structure of $X_1$, we have $\gamma_0 = 1/\sqrt{d}$, $\|w^\star\|^2 =1$. With Definition \ref{def:stage1-Ustar} and Assumption \ref{ass:x-norm}, we find that
\begin{align}
	\|W^\star\|_F^2 =d^2\log^2(1/\epsilon_{W,1}) \|w^\star (w^\star)^\top\|_F^2
	= d^2\log^2(1/\epsilon_{W,1}). \label{eq:Wstar-norm}
\end{align}

\textbf{Deal with term $N_{W_{t_1}}(\widetilde{W}_t; X_1, Y)$.}\quad
With Corollary \ref{coro:N-W-tW} and Assumption~\ref{ass:choice-hyperparam}, we have
\begin{align}
	N_{W_{t_1}}(\widetilde{W}_t; X_1, Y) \lesssim \epsilon_{W,1},  \label{eq:N-W-tW}
\end{align}
where $\epsilon_{W,1} =\Theta\big( (\log d) d^{3/2}L^{-2/3} \big)$.

\textbf{Deal with term $N_{W_{t_1}}(W^\star; X_1, Y)$.} Note that the indicator function depends on $W_{t_1}$ and the attention score depends on $W^\star$:
\begin{align*}
	y_L N_{W_{t_1}}(W^\star; X_1, Y)
    = \frac{1}{L} \sum_{i=1}^L y_L y_i \mathbbm{1}([X_1^\top]_i W_{t_1} x_{L,1}) \cdot \left(d\log(1/\epsilon_{W,1})[X_1^\top]_i w^\star (w^\star)^\top x_{L,1}\right).
\end{align*}
Given the data construction $x_{i,1} = y_i \frac{1}{\sqrt{d}}w^\star + e_i$ and the label $y_i = \text{sign}(\langle w^\star, e_i \rangle)$, we have $[X_1^\top]_i w^\star = y_i \frac{1}{\sqrt{d}} + \langle e_i, w^\star \rangle = y_i \left( \frac{1}{\sqrt{d}} + |\langle e_i, w^\star \rangle| \right).$
Similarly, for the query token $x_{L,1}$ with its label $y_L$, we have $(w^\star)^\top x_{L,1} = y_L \left( \frac{1}{\sqrt{d}} + |\langle e_L, w^\star \rangle| \right).$

For $y_Ly_i d\log(1/\epsilon_{W,1})[X_1^\top]_i w^\star (w^\star)^\top x_{L,1}$, with $\epsilon_{W,1} =\Theta\big( (\log d) d^{3/2} L^{-2/3} \big)$ (it simplifies to $\Theta\big(1/\text{Poly}(d) \big)$ given $L=\Omega(L^{2.25})$), 
\begin{align*}
	y_L y_i d\log(1/\epsilon_{W,1})[X_1^\top]_i w^\star (w^\star)^\top x_{L,1} 
	= y_L y_i d\log(1/\epsilon_{W,1}) y_i \left( \frac{1}{\sqrt{d}} + |\langle e_i, w^\star \rangle| \right) y_L \left( \frac{1}{\sqrt{d}} + |\langle e_L, w^\star \rangle| \right) 
	\gtrsim \log(1/\epsilon_{W,1}).
\end{align*}
Then,
\begin{align*}
	y_L N_{W_{t_1}}(W^\star; X_1, Y)
    \gtrsim \frac{1}{L} \sum_{i=1}^L \mathbbm{1}([X_1^\top]_i W_{t_1} x_{L,1}) \cdot \log(1/\epsilon_{W,1})
    \gtrsim \frac{1}{L}\log(1/\epsilon_{W,1}) \left\|\mathbbm{1}(X_1^\top W_{t_1} x_{L,1})\right\|_1.
\end{align*}
For $\left\|\mathbbm{1}(X_1^\top W_{t_1} x_{L,1})\right\|_1$, using Corollary~\ref{coro:lemma-ac-W} and the triangle inequality, we have $\left\|\mathbbm{1}(X_1^\top W_{t_1} x_{L,1})\right\|_1 \ge \left\|\mathbbm{1}(X_1^\top \widetilde{W}_{t_1} x_{L,1})\right\|_1 - \epsilon_W$, thus further consider $\left\|\mathbbm{1}(X_1^\top \widetilde{W}_{t_1} x_{L,1})\right\|_1= \sum_{i\in[L]} \mathbbm{1}([X_1^\top]_i \widetilde{W}_{t_1} x_{L,1}),$ where $\mathbbm{1}([X_1^\top]_i \widetilde{W}_{t_1} x_{L,1})$ is Bernoulli r.v.. Using Hoeffding's inequality in Lemma~\ref{lemma:Hoeffding}, with probability at least $1 - 1/d$, $\|\mathbbm{1}(X_1^\top \widetilde{W}_{t_1} x_{L,1})\|_1 \gtrsim L - \sqrt{L \log d}.$
We obtain
\begin{align*}
    \left\|\mathbbm{1}(X_1^\top W_{t_1} x_{L,1})\right\|_1 \ge \left\|\mathbbm{1}(X_1^\top \widetilde{W}_{t_1} x_{L,1})\right\|_1 - \epsilon_W \gtrsim L - \sqrt{L \log d} - \epsilon_W.
\end{align*}
Finally,
\begin{align}
	y_L N_{W_{t_1}}(W^\star; X_1, Y) \gtrsim \frac{\log(1/\epsilon_{W,1})}{L} \left( L - \sqrt{L\log d}-\epsilon_{W,1} \right)
	= \log(1/\epsilon_{W,1})\left(1-\sqrt{\frac{\log d}{L}}-\frac{\epsilon_{W,1}}{L}\right). \label{eq:N-W-Wstar}
\end{align}

\textbf{Combine Equation \ref{eq:N-W-tW} and Equation \ref{eq:N-W-Wstar}.}\quad
Combine Equation \ref{eq:N-W-tW} and Equation \ref{eq:N-W-Wstar}, we have
\begin{align*}
	y_LN_{W_{t_1}}(W^\star+\widetilde{W}_t; X_1, Y)
	\geq y_LN_{W_{t_1}}(W^\star; X_1, Y) - \left|y_LN_{W_{t_1}}(\widetilde{W}_t; X_1, Y)\right|
    \gtrsim \log(1/\epsilon_{W,1})\left(1-\sqrt{\frac{\log d}{L}}-\frac{\epsilon_{W,1}}{L} \right) - \epsilon_{W,1}.
\end{align*}
With Assumption~\ref{ass:choice-hyperparam}, we consider the surrogate loss of signal weight $W^\star$ at time $t$ (with small noise weight $\widetilde{W}_t$),
\begin{align*}
	K^1_{t_1}(W^\star)
	\lesssim& \log\left(1+\exp\left( - \log(1/\epsilon_{W,1})\left(1-\sqrt{\frac{\log d}{L}}-\frac{\epsilon_{W,1}}{L} \right) + \epsilon_{W,1}\right)\right) \\
	\lesssim& \log\left(1+\exp\left( - \log(1/\epsilon_{W,1})\right) \right)\\
    =&\log\left(1 + \epsilon_{W,1}\right) \leq \epsilon_{W,1}.
\end{align*}
The derivation follows from the fact that the term $-\log(1/\epsilon_{W,1})=\Theta(-\log d)$ in the exponent dominates (i) $\log(1/\epsilon_{W,1}) \sqrt{\frac{\log d}{L}} = \Theta(\frac{(\log d)^{3/2}}{L^{1/2}})$, (ii) $\log(1/\epsilon_{W,1}) \frac{\epsilon_{W,1}}{L}  = \Theta(\frac{(\log d)d^4}{L^{1/3}})$, and (ii) $\epsilon_{W,1} = \Theta(\frac{(\log d) d^{3/2}}{L^{2/3}})$, when given a large polynomial of $L=\text{Poly}(d)$. Finally, $K^1_{t_1}(W^\star)=\mathcal{O}(1/\text{Poly}(d))$.

\textbf{Deal with gradient descent to find $W^\star$.}\quad
We know that $\overline{U}_{t+1} = (1-\gamma_t \lambda)\overline{U}_{t} -\gamma_t\nabla_U \widehat{L}(U_t)$. Because of the block-diagonal structure of $U$, we then have
\begin{align*}
    \overline{W}_{t+1} = (1-\gamma_t \lambda)\overline{W}_{t} -\gamma_t \nabla_W \widehat{L}(U_t).
\end{align*}
With the definition of $K_t(\overline{U})$ in Equation \ref{eq:K-loss}, the chain rule dictates that the gradient of $K$ with respect to the signal component is equivalent to the gradient of $\widehat{L}$ evaluated at the current total weight $U_t = \overline{U}_t + \widetilde{U}_t$, i.e., 
\begin{align*}
    \nabla_{\overline{U}} K_t(\overline{U}_t) \equiv \nabla_{U} \widehat{L}(U_t) \Rightarrow \nabla_{\overline{W}} K_t(\overline{U}_t) = \nabla_{W} \widehat{L}(U_t).
\end{align*}
Thus, applying learning rate $\eta_1$, the gradient descent of signal $\overline{W}$ is,
\begin{align*}
	\overline{W}_{t+1} = (1-\eta_1\lambda)\overline{W}_t -\eta_1 \nabla_{\overline{W}} K_t(\overline{U}_t).
\end{align*}
Let $\|W^\star\|_F = d\log(1/\epsilon_{W,1}) \triangleq B$ from Equation \ref{eq:Wstar-norm}, and $\|\nabla_{\overline{W}} K(\overline{U}_t)\|_F= \|\nabla_{W} \widehat{L}(U_t)\|_F \leq \frac{1}{2}K_l (1+\gamma_0)^2 \triangleq L_K$ from Proposition~\ref{prop:derivative-L-norm}. 
Assume that $\|\overline{W}_t - W^\star \|_F \leq \|\overline{W}_0 - W^\star \|_F = B$, then we can measure the distance:
\begin{align*}
	\left\|\overline{W}_{t+1}-W^\star\right\|_F^2 
    =&\left\|(1-\eta_1\lambda)\overline{W}_t - \eta_1 \nabla_{\overline{W}} K_t(\overline{U}_t) -W^\star\right\|_F^2 \\ 	
    =&\left\|(1-\eta_1\lambda)(\overline{W}_t-W^\star) - \eta_1(\lambda W^\star + \nabla_{\overline{W}} K_t(\overline{U}_t))\right\|_F^2 \\
	=& \left\|(1-\eta_1\lambda)(\overline{W}_t-W^\star)\right\|_F^2  + \eta_1^2\left\|\lambda W^\star + \nabla_{\overline{W}} K_t(\overline{U}_t)\right\|_F^2 - 2\eta_1(1-\eta_1\lambda)\langle \overline{W}_t-W^\star, \lambda W^\star\rangle \nonumber\\
	&- 2\eta_1 (1-\eta_1\lambda) \langle \overline{W}_t-W^\star, \nabla_{\overline{W}} K_t(\overline{U}_t) \rangle \\
	\leq& \left\|(1-\eta_1\lambda)(\overline{W}_t-W^\star)\right\|_F^2  + 2\eta_1^2(\lambda^2 B^2 + L_K^2) + 2\eta_1\lambda(1-\eta_1\lambda)B^2 \\
	&- 2\eta_1 (1-\eta_1\lambda) \left( K_t(\overline{W}_t, \overline{V}_t) - K_t(W^\star, \overline{V}_t) \right).
\end{align*}
For the sake of contradiction, assume that the global loss has not entered the target neighborhood basin, i.e., for $t<t_1$, $K_t(\overline{W}_t, \overline{V}_t) - K_t(W^\star, \overline{V}_t) \geq C$. Let $\lambda B^2 \leq \frac{1}{8}C$ and $\eta_1 \leq \frac{C}{16(\lambda^2 B^2 + L_K^2)}$ (which implies $\eta_1\lambda \leq 1/4$), the descent terms dominates the positive perturbation terms. Among these, $B=\Theta(d \log d)$, $C=\Theta((\log d)^2/\sqrt{d})$, $\frac{C}{8(\lambda^2 B^2 + K^2)}=\Theta((\log d)^2/\sqrt{d})$, thus we set $\eta_1=\Theta(1/\sqrt{d})$. We then have:
\begin{align*}
	\left\|\overline{W}_{t+1}-W^\star\right\|_F^2 
	&\leq \left\|\overline{W}_t-W^\star\right\|_F^2 + 2\eta_1^2(\lambda^2 B^2 + L_K^2) + 2\eta_1\lambda(1-\eta_1\lambda)B^2 - 2\eta_1(1-\eta_1\lambda)C \\
    &\leq \|\overline{W}_t-W^\star\|_F^2 + 2\eta_1^2(\lambda^2 B^2 + L_K^2) + 2\eta_1\lambda B^2 - 2\eta_1(1-\eta_1\lambda)C \\
    &\leq \|\overline{W}_t-W^\star\|_F^2 + \frac{1}{8}\eta_1 C + \frac{1}{4}\eta_1 C -\frac{3}{2}\eta_1 C \\
    &= \|\overline{W}_t-W^\star\|_F^2 -\eta_1 C
\end{align*}
Thus, in the elementary stage with $t_1\triangleq\frac{B^2}{\eta_1C}$ iterations, $\left\|\overline{W}_{t_1}-W^\star\right\|_F^2 \leq \left\|\overline{W}_0-W^\star\right\|_F^2  - t_1 \eta_1 C < 0,$ which is a contradiction. This implies that within $t_1$ steps:
$$
K_{t_1}(\overline{W}_{t_1}, \overline{V}_{t_1}) - K_{t_1}(W^\star, \overline{V}_{t_1}) \leq C,
$$
where $C=\Theta((\log d)^2/\sqrt{d})$, $t_1=\Theta(\frac{1}{\eta_1\lambda})$.

\textbf{Next, we decouple the global loss $K_{t_1}(\overline{U}_{t_1})=K_{t_1}(\overline{W}_{t_1}, \overline{V}_{t_1})$ into the surrogate loss $K^1_{t_1}(\overline{W}_{t_1})$} to evaluate the pure learning progress on the linear separable component $\mathcal{P}$.
Since $l(\cdot)$ is $K_l$-Lipschitz continuous. At iteration $t_1$,
\begin{align*}
    \left| K_{t_1}(\overline{W}_{t_1}, \overline{V}_{t_1}) - K^1_{t_1}(\overline{W}_{t_1}) \right| 
    \leq& \frac{1}{N} \sum_{n=1}^N \left| l\left( y_L N_{U_{t_1}}(\overline{U}_{t_1}+\widetilde{U}_{t_1}; X^n, \widetilde{Y}^n) \right) - l\left(y_L \cdot 1/2 N_{W_{t_1}}(\overline{W}_{t_1}+\widetilde{W}_{t_1}; X_1^n, Y^n)\right) \right| \\
    \leq& K_l \cdot \max_{n \in [N]} \left| 1/2 N_{V_{t_1}}(V_{t_1}; X_2^n, Y^n) \right|,
\end{align*}
where the last step comes from Equation \ref{eq:f-decomp}.

As established in Theorem~\ref{the:stage1-Q}, in Equation~\ref{eq:g-bound}, the forward output $g(X_2)$ (or $N_{V_{t_1}}(V_{t_1};X_2,Y)$) remains weak, we then have:
$$ 
\left| K_{t_1}(\overline{W}_{t_1}, \overline{V}_{t_1}) - K^1_{t_1}(\overline{W}_t) \right| \lesssim K_l \sqrt{\frac{1}{\text{Poly}(d)}+\sqrt{\frac{\log d}{N}}}
\triangleq \Delta_K.
$$
Thus, the surrogate loss for component $\mathcal{P}$ satisfies
\begin{align*}
	K^1_{t_1}(\overline{W}_{t_1}) &\leq K_{t_1}(\overline{W}_{t_1}, \overline{V}_{t_1}) + \Delta_K 
    \leq K_{t_1}(W^\star, \overline{V}_{t_1}) + C + \Delta_K \\
    &\leq K^1_{t_1}(W^\star) + C + 2\Delta_K \\
    &\lesssim \epsilon_{W,1} + C + 2\Delta_K,
\end{align*}
where $\epsilon_{W,1} = \Theta(\frac{(\log d) d^{3/2}}{L^{2/3}})$, $C=\Theta((\log d)^2/\sqrt{d})$, $\Delta_K = \Theta((\frac{\log d}{N})^{1/4})$.
In conclusion,
$$
K^1_{t_1}(\overline{W}_{t_1}) \lesssim \frac{1}{\text{Poly}(d)} + \frac{(\log d)^2}{\sqrt{d}} +  \left(\frac{\log d}{N}\right)^{1/4}.
$$

\textbf{Finally, we look at the signal weight norm.} From the definition of the surrogate loss on the linear separable component $\mathcal{P}$, denote the upper bound of loss as $\mathcal{E} \triangleq \frac{1}{\text{Poly}(d)} + \frac{(\log d)^2}{\sqrt{d}} +  \left(\frac{\log d}{N}\right)^{1/4}$, i.e., $K^1_{t_1}(\overline{W}_{t_1}) \lesssim \mathcal{E}$,
\begin{align*}
    \mathcal{E} \gtrsim K^1_{t_1}(\overline{W}_{t_1}) = \frac{1}{N} \sum_{n=1}^N l\left(\frac{1}{2}y_L^n h_{t_1}(X_1^n)\right)
    \geq \log\left(1 + \exp\left(- \frac{1}{N} \sum_{n=1}^N y_L^n h_{t_1}(X_1^n)\right)\right),
\end{align*}
where the last step comes from Jensen's inequality on convex logistic loss.

Then we have $\frac{1}{N} \sum_{n=1}^N y_L^n h_{t_1}(X_1^n) \gtrsim -\log (e^{\mathcal{E}}-1).$
With $\mathcal{E} \rightarrow 0$, we have $-\log (e^{\mathcal{E}}-1)\approx \log(1/\mathcal{E})$. It implies that
\begin{align}\label{eq:mathcalE}
    \frac{1}{N} \sum_{n=1}^N y_L^n h_{t_1}(X_1^n) \gtrsim \log(1/\mathcal{E}).
\end{align}
With the definition of $h_{t_1}(X_1^n)$,
\begin{align}
    \left|y_L^n h_{t_1}(X_1^n)\right| 
    \leq \frac{1}{L} \sum_{j=1}^L \left| ([(X_1^n)^\top]_j)^\top \overline{W}_{t_1} x_{L,1}^n \right|+\frac{1}{L} \sum_{j=1}^L \left| ([(X_1^n)^\top]_j)^\top \widetilde{W}_{t_1} x_{L,1}^n \right| \label{eq:yh-decomposition}.
\end{align}
For the first term in Equation~\ref{eq:yh-decomposition}: Starting from $\overline{W}_0 =\mathbf{0}$, we have $\overline{W}_{t_1} = -\eta_1 \sum_{s=1}^{t_1} \nabla_{\overline{W}} K_{s-1}(\overline{U}_{s-1})$, i.e.,
\begin{align*}
    \overline{W}_{t_1}
    &= -\eta_1 \sum_{s=1}^{t_1} \frac{1}{N}\sum_{n=1}^N \left[\frac{1}{2L}l^\prime\left(f_{s-1}^n\right) y_L^n \sum_{j=1}^L y_j^n \mathbbm{1}\left([(X_1^n)^\top]_j W_{s-1} x_{L,1}^n\right) ([(X_1^n)^\top]_j)^\top (x_{L,1}^n)^\top\right].
\end{align*}
Let $c_{s,n,j} \triangleq -\frac{\eta_1}{2NL}l^\prime\left(f_{s-1}^n\right)\mathbbm{1}\left([(X_1^n)^\top]_j W_{s-1} x_{L,1}^n\right) \geq 0$ denote the non-negative scalar multiplier (the derivative $l^\prime<0$). Substituting the data generation process $([(X_1^n)^\top]_j)^\top = y_j^n \gamma_0 w^\star + e_j^n$ and $x_{L,1}^n = y_L^n \gamma_0 w^\star + e_L^n$,
\begin{align*}
    y_L^n y_j^n ([(X_1^n)^\top]_j)^\top (x_{L,1}^n)^\top &= y_L^n y_j^n \left( y_j^n \gamma_0 w^\star + e_j^n \right) \left( y_L^n \gamma_0 w^\star + e_L^n \right)^\top \\
    &= \gamma_0^2 w^\star (w^\star)^\top + \gamma_0 y_L^n w^\star (e_L^n)^\top + \gamma_0 y_j^n e_j^n (w^\star)^\top + y_L^n y_j^n e_j^n (e_L^n)^\top.
\end{align*}
Summing these expanded terms, $\overline{W}_{t_1} = \mu_{t_1} w^\star (w^\star)^\top + E_{t_1}$, where:
\begin{align*}
    \mu_{t_1} = \gamma_0^2 \sum_{s=1}^{t_1} \sum_{n=1}^N \sum_{j=1}^L c_{s,n,j}, \quad
    E_{t_1} = \sum_{s=1}^{t_1} \sum_{n=1}^N \sum_{j=1}^L c_{s,n,j} \Big( \gamma_0 y_L^n w^\star (e_L^n)^\top + \gamma_0 y_j^n e_j^n (w^\star)^\top + y_L^n y_j^n e_j^n (e_L^n)^\top \Big).
\end{align*}
We have $c_{s,n,j} = \Theta\left(\frac{\eta_1}{NL}\right)$, then $\mu_{t_1} w^\star (w^\star)^\top$ term accumulates linearly:
\begin{align}\label{eq:mu_t1_bound}
    \mu_{t_1} = \Theta\left( \gamma_0^2 \cdot N \cdot L \cdot \frac{\eta_1 t_1}{NL} \right) = \Theta\left( \gamma_0^2 \eta_1 t_1 \right).
\end{align}
Let $Z_{n,j} \triangleq \gamma_0 y_L^n w^\star (e_L^n)^\top + \gamma_0 y_j^n e_j^n (w^\star)^\top + y_L^n y_j^n e_j^n (e_L^n)^\top$, $E_{t_1}=\sum_{s=1}^{t_1} \sum_{n=1}^N \sum_{j=1}^L c_{s,n,j}Z_{n,j}$. Notice that $Z_{n,j}$ are invariant across the optimization steps $s$, then $E_{t_1} = \sum_{n=1}^N \sum_{j=1}^L \left( \sum_{s=1}^{t_1} c_{s,n,j} \right) Z_{n,j}$.
We analyze $E_{t_1}$ coordinate-wise. Since $e \sim \mathcal{N}(0, I/d)$, the coordinates of $Z_{n,j}$ are independent zero-mean sub-exponential variables with $\psi_1$-norm bounded by $\mathcal{O}(1/d)$. With $\sum_{s=1}^{t_1} c_{s,n,j} = \Theta\left(\frac{\eta_1 t_1}{NL}\right)$, we apply Bernstein's inequality in Lemma~\ref{lemma:Bernstein-sub-exp}, we obtain the element-wise bound $|(E_{t_1})_{u,v}| \lesssim  \frac{\eta_1 t_1}{d} \sqrt{\frac{\log(d/\delta)}{NL}} \triangleq t$ with high probability $1-\delta/d^2$. Taking a union bound over all $d^2$ coordinates (with $\delta=1/d$),
\begin{align}\label{eq:E_t1_bound}
    \|E_{t_1}\|_F \lesssim \sqrt{ \sum_{u=1}^d \sum_{v=1}^d t^2 } =  \eta_1 t_1 \sqrt{\frac{\log d}{NL}}.
\end{align}

Comparing Equation~\ref{eq:mu_t1_bound} and Equation~\ref{eq:E_t1_bound}, with $\gamma_0 = 1/\sqrt{d}$, we have $\mu_{t_1} = \Theta(\frac{\eta_1 t_1}{d})$, yielding:
\begin{align*}
    \frac{\mu_{t_1}}{\|E_{t_1}\|_F} \gtrsim \frac{\eta_1 t_1 / d}{\eta_1 t_1 \sqrt{\log d / (NL)}} = \frac{1}{d} \sqrt{\frac{NL}{\log d}}.
\end{align*}
With polynomial datasize $N = \Theta(\text{Poly}(d))$ and prompt length $L=\Theta(\text{Poly}(d))$, this ratio diverges to infinity, establishing that $\mu_{t_1} \gg \|E_{t_1}\|_F$. Consequently, the structural expansion of $\overline{W}_{t_1} = \mu_{t_1} w^\star (w^\star)^\top + E_{t_1}$ is strictly dominated by the first term, i.e., $\overline{W}_{t_1} \approx \mu_{t_1} w^\star (w^\star)^\top$. With $\|w^\star\|_2=1$, we have $\mu_{t_1} \approx \|\overline{W}_{t_1}\|_F$.
Thus,
\begin{align}
    \left| ([(X_1^n)^\top]_j)^\top \overline{W}_{t_1} x_{L,1}^n \right| 
    \leq \mu_{t_1} \left| ([(X_1^n)^\top]_j)^\top w^\star (w^\star)^\top x_{L,1}^n \right| 
    \lesssim \frac{1}{d} \|\overline{W}_{t_1}\|_F \label{eq:yh-first-term}.
\end{align}
For the second term in Equation~\ref{eq:yh-decomposition}, we analyze $\left| [X_1^n]_j^\top \widetilde{W}_{t_1} x_{L,1}^n \right|$.
\begin{align*}
    [X_1^n]_j^\top \widetilde{W}_{t_1} x_{L,1}^n 
    &= \gamma_0^2 y_j^n y_L^n (w^\star)^\top \widetilde{W}_{t_1} w^\star + \gamma_0 y_j^n (w^\star)^\top \widetilde{W}_{t_1} e_L^n + \gamma_0 y_L^n (e_j^n)^\top \widetilde{W}_{t_1} w^\star + (e_j^n)^\top \widetilde{W}_{t_1} e_L^n.
\end{align*}
The first term is bounded by $\frac{1}{d} \|\widetilde{W}_{t_1}\|_F$. 
For the second and third terms, since $e \sim \mathcal{N}(0, I/d)$, $(w^\star)^\top \widetilde{W}_{t_1} e_L^n$ and $(e_j^n)^\top \widetilde{W}_{t_1} w^\star$ are zero-mean Gaussian variables with variance bounded by $\frac{1}{d}\|\widetilde{W}_{t_1}\|_F^2$. By standard Gaussian tail bounds, the two terms are bounded by $\frac{\sqrt{\log d}}{d} \|\widetilde{W}_{t_1}\|_F$.
For the last term, when $j \neq L$, $(e_j^n)^\top \widetilde{W}_{t_1} e_L^n$ is bounded by $\frac{\log d}{d}\|\widetilde{W}_{t_1}\|_F$ with high probability $1-1/d$; when $j = L$, $e_L^\top \widetilde{W}_{t_1} e_L$ is bounded by $\frac{1}{\sqrt{d}}\|\widetilde{W}_{t_1}\|_F$. Notice that $\widetilde{W}_{t_1} \sim \mathcal{N}(0,\tau_0^2)$ and $\tau_0 =\Theta(1/{\sqrt{d}})$, in total,
\begin{align}
    \frac{1}{L}\sum_{j=1}^N \left|[X_1^n]_j^\top \widetilde{W}_{t_1} x_{L,1}^n \right| \lesssim \frac{\log d}{\sqrt{d}}\label{eq:yh-second-term}
\end{align}

Combining Equation~\ref{eq:yh-first-term} and Equation~\ref{eq:yh-second-term},
\begin{align*}
    \left|y_L^n h_{t_1}(X_1^n)\right|
     \leq \frac{1}{L} \sum_{i=1}^L \left| ([(X_1^n)^\top]_j)^\top \overline{W}_{t_1} x_{L,1}^n \right|+\frac{1}{L} \sum_{i=1}^L \left| ([(X_1^n)^\top]_j)^\top \widetilde{W}_{t_1} x_{L,1}^n \right| 
    \leq \frac{1}{d} \|\overline{W}_{t_1}\|_F + \frac{\log d}{\sqrt{d}}.
\end{align*}

Combining with Equation~\ref{eq:mathcalE},
\begin{align*}
    \frac{1}{N} \sum_{n=1}^N \left(\frac{1}{d} \|\overline{W}_{t_1}\|_F + \frac{\log d}{\sqrt{d}}\right) \geq  \frac{1}{N} \sum_{n=1}^N y_L^n h_{t_1}(X_1^n) \gtrsim \log(1/\mathcal{E}).
\end{align*}
Therefore, we obtain $\|\overline{W}_{t_1}\|_F \gtrsim d\log(1/\mathcal{E})-\sqrt{d}\log d,$ where $\mathcal{E} \triangleq \frac{1}{\text{Poly}(d)} +\frac{(\log d)^2}{\sqrt{d}}+  \left(\frac{\log d}{N}\right)^{1/4}$ is of $\frac{(\log d)^2}{\sqrt{d}}$ order, i.e., 
$$
\|\overline{W}_{t_1}\|_F\gtrsim d\log d - d \log \log d - \sqrt{d}\log d.
$$
With $\lim_{d \to \infty} \frac{d \log d - d \log \log d - \sqrt{d} \log d}{d \log d} = 1 - \lim_{d \to \infty} \frac{\log \log d}{\log d} - \lim_{d \to \infty} \frac{1}{\sqrt{d}} =1$, it implies that $\|\overline{W}_{t_1}\|_F$ asymptotically converges toward the optimal signal weight $W^\star$.
\end{proof}

\subsection{Proof for the Specialized Stage: Proof of Theorem \ref{the:stage2-Q}}\label{proof:stage2-Q}

\begin{proof}
We begin by the following definition.
\begin{definition}\label{def:H/epsilon-V-t1}
	For time $t_1$, input $X \in \mathbb{R}^{d \times L}$ with query $x_L\in\{z-\zeta, z, z+\zeta\} \in \mathbb{R}^d$, define
	\begin{align*}
		&\mathcal{H}_1 \triangleq \{ i \in [L] \mid [X^\top]_i V_{t_1} (z - \zeta) \geq 0, [X^\top]_i V_{t_1} z \geq 0, [X^\top]_i V_{t_1} (z + \zeta) < 0 \}, \\
		&\mathcal{H}_2 \triangleq \{ i \in [L] \mid [X^\top]_i V_{t_1} (z - \zeta) \geq 0, [X^\top]_i V_{t_1} z < 0, [X^\top]_i V_{t_1} (z + \zeta) < 0 \}, \\
		&\mathcal{H}_3 \triangleq \{ i \in [L] \mid [X^\top]_i V_{t_1} (z - \zeta) < 0, [X^\top]_i V_{t_1} z < 0, [X^\top]_i V_{t_1} (z + \zeta) \geq 0 \}, \\
		&\mathcal{H}_4 \triangleq \{ i \in [L] \mid [X^\top]_i V_{t_1} (z - \zeta) < 0, [X^\top]_i V_{t_1} z \geq 0, [X^\top]_i V_{t_1} (z + \zeta) \geq 0 \}.
	\end{align*}
\end{definition}
Similar to Definition \ref{def:epsilon-tV}, note that $X$ aligns with $X_2$ and $x_L$ aligns with $x_{L,2}$.

We first analyze the expected cardinality of these sets. With Assumption \ref{ass:x-norm}, we can compute the cosine similarity of $z-\zeta$ and $z$, with the angle $\theta$. Then, $\cos \theta = \frac{u^2-ur\cos \theta_0}{u  \sqrt{u^2+ r^2-2ur\cos \theta_0}},$ $\sin \theta = \frac{r \sin \theta_0}{\sqrt{u^2+ r^2-2ur\cos \theta_0}}$.
For $r=\Theta(1/\text{Poly}(d))$, $u=\Theta(1)$, employing the Taylor expansion of $\arcsin \theta$, the angle of $z-\zeta$ and $z$ is approximated by $\theta = \frac{r}{u}+\mathcal{O}(r^2)$, similarly, the angle of $z+\zeta$ and $z$ is also $\theta = \frac{r}{u}+\mathcal{O}(r^2)$.

For $\mathcal{H}_1$, when $[X^\top]_i V_{t_1}$ falls into the middle of $z-\zeta$ and $z$, as well as not in the positive half space of  $z + \zeta$, its probability is approximately the proportion of the spherical surface area corresponding to the angle \(\frac{r}{u} + \mathcal{O}(r^2)\). 
If we use $\widetilde{V}_{t_1}$ in $\mathcal{H}_i$, its geometric probability of falling into the region (restrained by the subset) is isotropic, yielding an expected subset size of $L \cdot \frac{r}{2\pi u}$. 
According to Corollary~\ref{coro:lemma-ac-V}, the number of tokens altering their ReLU activation states between $V_{t_1}$ and $\widetilde{V}_{t_1}$ is strictly by $\epsilon_V$. The expected number of flipped tokens falling into $\mathcal{H}_1$ is proportionally scaled by the angle fraction $\frac{r}{2\pi u}$. Thus,
\begin{align*}
    \frac{rL}{2\pi u} -  \frac{r\epsilon_V}{2\pi u} \leq \mathbb{E}\left[|\mathcal{H}_1|\right] \leq \frac{rL}{2\pi u} + \frac{r\epsilon_V}{2\pi u}.
\end{align*}
Using Hoeffding's inequality in Lemma \ref{lemma:Hoeffding}, let $X_i = \mathbf{1}\{ i \in \mathcal{H}_1 \}$ and $|\mathcal{H}_1| = \sum_{i=1}^L X_i$. With probability at least $1-1/d$ and similarly for $\mathcal{H}_2\sim\mathcal{H}_4$,
\begin{align*}
    \frac{r L}{2\pi u} - \frac{r \epsilon_V}{2\pi u} - \sqrt{L\log d} \lesssim |\mathcal{H}_1|, |\mathcal{H}_2|, |\mathcal{H}_3|, |\mathcal{H}_4| \lesssim  \frac{rL}{2\pi u} + \frac{r\epsilon_V}{2\pi u} + \sqrt{L\log d}.
\end{align*}

\begin{definition}\label{def:stage2-Ustar}
	In the specialized stage, denote optimal weight as $U^\star_2 = \begin{bmatrix}
		\overline{W}_{t_1}+\Delta W & 0 \\ 0 & \overline{V}_{t_1} + V^\star
	\end{bmatrix}$ $= \begin{bmatrix}
		\overline{W}_{t_1+t} & 0 \\ 0 & \overline{V}_{t_1} + V^\star
	\end{bmatrix}$, $V^\star \in \mathbb{R}^{d \times d}$ satisfies
	\begin{align}
		[X^\top_2 V^\star]_i = \left\{ \begin{array}{ll}
			\frac{\log(1/\epsilon_{V,1})}{r} z^\top & \mbox{if $i \in \mathcal{H}_1$};\\
			-\frac{2\log(1/\epsilon_{V,1})}{ r}z^\top & \mbox{if $i \in \mathcal{H}_2 $};
			\\
            -\frac{2\log(1/\epsilon_{V,1})}{ r}z^\top &\mbox{if $i \in \mathcal{H}_3$};
			\\
			\frac{\log(1/\epsilon_{V,1})}{ r}z^\top &\mbox{if $i \in \mathcal{H}_4$};
			\\
			0 & \mbox{otherwise.}\end{array} \right.
	\end{align}
\end{definition}
In this section, we primarily focus on the process of optimizing from $\overline{V}_{t_1}$ to $\overline{V}_{t_1} + V^\star$. By Definition~\ref{def:stage2-Ustar},
\begin{align*}
	\|X_2^\top V^\star\|_2^2 
	\lesssim u^2  \left(\sum_{k=1}^4 |\mathcal{H}_k|\right) \left(\frac{ \log(1/\epsilon_{V,1})}{r}\right)^2 
	\lesssim u^2 \left(\frac{rL}{2\pi u} + \frac{r\epsilon_V}{2\pi u}  + \sqrt{L\log d}\right) \frac{\log^2(1/\epsilon_{V,1})}{r^2}.
\end{align*}
With $\epsilon_V=\Theta(d^4L^{2/3})$, we have $L\gg\epsilon_V$ when with a large polynomial degree of $L$. Then with $\|X_2\|_F=\Theta(\sqrt{L})$,
\begin{align*}
    \|V^\star\|_F \lesssim \frac{1}{\sqrt{L}} \sqrt{u^2 \frac{rL}{u} \frac{\log^2(1/\epsilon_{V,1})}{r^2}}= \frac{u^{1/2}\log(1/\epsilon_{V,1})}{r^{1/2}}.
\end{align*}
Thus $\|V^\star\|_F$ is of $\mathcal{O}\left(\text{Poly}(d)\log d\right)$ order.

\textbf{We first evaluate the surrogate loss at the optimal signal weight $\overline{V}_{t_1} + V^\star$.} In this evaluation, the network's non-linear activation states are fixed by the true trajectory state at iteration $t_1+t$, i.e., $V_{t_1+t} = \overline{V}_{t_1+t} + \widetilde{V}_{t_1+t}$.

We then consider $g_{t_1+t}(X_2) = N_{V_{t_1+t}}(\overline{V}_{t_1} + V^\star + \widetilde{V}_{t_1+t}; X_2, Y)$:
\begin{align*}
	y_L N_{V_{t_1+t}}(\overline{V}_{t_1} + V^\star + \widetilde{V}_{t_1+t}; X_2, Y) 
	=& y_L N_{V_{t_1+t}}(V^\star ; X_2, Y) + y_L N_{V_{t_1+t}}(\overline{V}_{t_1} ; X_2, Y) + y_L N_{V_{t_1+t}}(\widetilde{V}_{t_1+t} ; X_2, Y) \\
	\geq& \underbrace{y_L N_{V_{t_1}}(V^\star ; X_2, Y)}_{A} - \underbrace{\left|y_L N_{V_{t_1+t}}(V^\star; X_2,Y) - y_L N_{V_{t_1}}(V^\star; X_2,Y) \right|}_{B} \\
	&- \underbrace{\left|y_L N_{V_{t_1+t}}(\overline{V}_{t_1} ; X_2, Y)\right|}_{C} - \underbrace{\left|y_L N_{V_{t_1+t}}(\widetilde{V}_{t_1+t} ; X_2, Y)\right|}_{D}.
\end{align*}

\textbf{Deal with term $A$.}\quad For $x_{L,2} = z - \zeta$, we have that 
\begin{align*}
	N_{V_{t_1}}(V^\star; X_2, Y, x_{L,2} = z - \zeta) 
    \leq \left( \frac{|\mathcal{H}_1|}{L} - \frac{2|\mathcal{H}_2|}{L} \right) \frac{ \log(1/\epsilon_{V,1})}{r} z^\top(z-\zeta) 
    \lesssim -\log(1/\epsilon_{V,1}),
\end{align*}
where the last step comes from $-\frac{r}{2\pi u}$ dominates other terms with $u-r=\Theta(1), L=\text{Poly}(d)\gg\epsilon_V$.

For $x_{L,2} = z + \zeta$, we have that
\begin{align*}
	N_{V_{t_1}}(V^\star; X_2, Y, x_{L,2} = z + \zeta) 
    \leq \left( -\frac{2|\mathcal{H}_3|}{L} + \frac{|\mathcal{H}_4|}{L} \right) \frac{ \log(1/\epsilon_{V,1})}{r}  z^\top(z+\zeta)  
	\lesssim -\log(1/\epsilon_{V,1}).
\end{align*}
For $x_{L,2} = z $, we have that 
\begin{align*}
	N_{V_{t_1}}(V^\star; X_2, Y, x_{L,2} = z ) 
    \geq \left( \frac{|\mathcal{H}_1|}{L} + \frac{|\mathcal{H}_4|}{L} \right) \frac{ \log(1/\epsilon_{V,1})}{r} z^\top z 
    \gtrsim \log(1/\epsilon_{V,1}).
\end{align*}
Finally, since $y_L=-1$ if $x_L \in \{z-\zeta,z+\zeta\}$ and $y_L=1$ if $x_L=z$,
\begin{align*} 
	y_L N_{V_{t_1}}(V^\star; X_2, Y) \gtrsim \log(1/\epsilon_{V,1}).
\end{align*}

\textbf{Deal with term $B$.}\quad 
By the definition of $V^\star$, for any inactive token $i \notin \bigcup_{k=1}^4 \mathcal{H}_k$, we have $[X_2^\top V^\star]_i = 0$. For any active token $i \in \bigcup_{k=1}^4 \mathcal{H}_k$, with $\left|[X_2^\top]_i V^\star x_{L,2}\right| \lesssim u^{1/2}r^{-1/2}(u+r)^2\log(1/\epsilon_{V,1})$ and Corollary~\ref{coro:lemma-ac-V-t1+t-t},
\begin{align*}
	\left|y_L N_{V_{t_1 + t }}(V^\star; X_2,Y) - y_L N_{V_{t_1}}(V^\star; X_2,Y) \right|
	\lesssim \left( \frac{\epsilon_V}{L} + \sqrt{\frac{\eta_2}{\eta_1}} + \sqrt{\frac{\log d}{L}} \right) \frac{u^{1/2}(u+r)^2\log(1/\epsilon_{V,1})}{r^{1/2}}.
\end{align*}
Under our asymptotic framework where the prompt length is large ($L = \text{Poly}(d) \gg \epsilon_V$), $\frac{\epsilon_V}{L}$ and $\sqrt{\frac{\log d}{L}}$ are vanishing terms. With small learning rate $\eta_2 = \eta_1 \lambda^2 r^2$,
\begin{align*}
	\left|y_L N_{V_{t_1 + t }}(V^\star; X_2,Y) - y_L N_{V_{t_1}}(V^\star; X_2,Y) \right| \lesssim \lambda \sqrt{r}\log(1/\epsilon_{V,1}).
\end{align*}

\textbf{Deal with term $C$.}\quad 
With Equation~\ref{eq:g-bound} and Corollary \ref{coro:N-V-tV}, we have
\begin{align*}
	|N_{V_{t_1}}(\overline{V}_{t_1}; X_2,Y) | \leq |g_{t_1} (X_2) | + |N_{V_{t_1}}(\overline{V}_{t_1}; X_2,Y)  - N_{V_{t_1}}(V_{t_1}; X_2,Y)|
	\lesssim \sqrt{\frac{1}{\text{Poly(d)}} + \sqrt{\frac{\log d}{N}}} + \epsilon_{V,1}.
\end{align*}
With Corollary \ref{coro:lemma-ac-V-t1+t-t} and $\|\overline{V}_{t_1}\|_F \lesssim \frac{1}{\text{Poly}(d)}$,
\begin{align*}
	\left|y_L N_{V_{t_1 + t }}(\overline{V}_{t_1}; X_2,Y) - y_L N_{V_{t_1}}(\overline{V}_{t_1}; X_2,Y)\right|  &\lesssim \left(\epsilon_V+L\sqrt{\frac{\eta_2}{\eta_1}}+ \sqrt{L \log d}\right)\frac{1}{L}\cdot \|\overline{V}_{t_1}\|_F
	\lesssim \frac{1}{\text{Poly}(d)}\sqrt{\frac{\eta_2}{\eta_1}}.
\end{align*}
Finally, with $\epsilon_{V,1}=\Theta(1/\text{Poly}(d))$ in Corollary~\ref{coro:N-V-tV} and $\eta_2 = \eta_1 \lambda^2 r^2$, we get
\begin{align*}
	\left|y_L N_{V_{t_1 + t }}(\overline{V}_{t_1}; X_2,Y)\right| &\lesssim  \sqrt{\frac{1}{\text{Poly(d)}} + \sqrt{\frac{\log d}{N}}} + \epsilon_{V,1} + \frac{1}{\text{Poly}(d)}\sqrt{\frac{\eta_2}{\eta_1}}
    =\mathcal{O}\left(\frac{1}{\text{Poly}(d)}\right).
\end{align*}

\textbf{Deal with term $D$.}\quad
With Corollary \ref{coro:N-V-tV} and Assumption~\ref{ass:choice-hyperparam},
\begin{align*}
    \left|y_L N_{V_{t_1+t}}(\widetilde{V}_{t_1+t} ; X_2, Y)\right| \lesssim \epsilon_{V,1},
\end{align*}
where $\epsilon_{V,1} =\Theta\big( (\log d) d^{3/2}L^{-2/3} \big)=\Theta(\frac{1}{\text{Poly}(d)})$.

\textbf{Combine term $A, B$, $C$ and $D$.}\quad
\begin{align}\label{eq:yNV_tgt}
	y_L N_{V_{t_1+t}}(\overline{V}_{t_1} + V^\star+\widetilde{V}_{t_1+t}; X_2, Y)
	\gtrsim \log(1/\epsilon_{V,1}) - \lambda \sqrt{r}\log(1/\epsilon_{V,1}) - \frac{1}{\text{Poly}(d)} - \epsilon_{V,1}.
\end{align}
Consider the surrogate loss $K_{t_1+t}$ evaluated at the target signal $\overline{V}_{t_1}+V^\star$ and noise weight $\widetilde{V}_{t_1+t}$ at time $t_1+t$:
\begin{align}
	K^2_{t_1+t}(\overline{V}_{t_1} + V^\star)
    &\lesssim \log \left(1+\exp\left(-\log(1/\epsilon_{V,1}) + \lambda \sqrt{r}\log(1/\epsilon_{V,1}) + \frac{1}{\text{Poly}(d)} + \epsilon_{V,1}\right)\right) \nonumber\\
    &\lesssim \exp\left(-\log(1/\epsilon_{V,1})\right) = \epsilon_{V,1} \label{eq:K_V_star_bound},
\end{align}
which comes from $\log(1/\epsilon_{V,1}) = \Theta(\log d)$, $\lambda \sqrt{r} \log(1/\epsilon_{V,1}) = \Theta(\frac{\log d}{d^4})$, $\epsilon_{V,1} = \Theta(\frac{1}{\text{Poly}(d)})$. This establishes that our constructed target weight $V^\star$ yields a small loss, serving as a valid optimum for the specialized knowledge $\mathcal{Q}$.

\textbf{Deal with gradient descent to find $\overline{V}_{t_1}+V^\star$.}\quad
Consider the gradient descent of signal $\overline{V}$ and apply learning rate $\eta_2$,
\begin{align*}
	\overline{V}_{t+1} & = (1-\eta_2\lambda)\overline{V}_t -\eta_2 \nabla_{\overline{V}} K_t(\overline{U}_t).
\end{align*}
Similar to the gradient descent analysis of $\overline{W}$ in Theorem~\ref{the:stage1-P}, let $\| V^\star\|_F = \mathcal{O}\left(\frac{u^{1/2}\log(1/\epsilon_{V,1})}{r^{1/2}}\right)\triangleq B$ and $\|\nabla_{\overline{V}} K_t(\overline{U}_t)\|_F \leq \frac{1}{2}K_l(u+r)^2 \triangleq L_K$ from Proposition~\ref{prop:derivative-L-norm}. Assume that $\|\overline{V}_t - (\overline{V}_{t_1}+V^\star)\|_F \leq \|\overline{V}_{t_1} - (\overline{V}_{t_1}+V^\star)\|_F \leq B$. Then
\begin{align*}
	\left\|\overline{V}_{t+1}-(\overline{V}_{t_1} + V^\star )\right\|_F^2 
    =&\left\|(1-\eta_2\lambda)\overline{V}_t - \eta_2\nabla_{\overline{V}} K_t(\overline{U}_t) - (\overline{V}_{t_1} + V^\star)\right\|_F^2 \\  
	\leq& \left\|(1-\eta_2\lambda)(\overline{V}_t-(\overline{V}_{t_1} + V^\star))\right\|_F^2  + 2\eta_2^2(\lambda^2 B^2 + L_K^2) + 2\eta_2\lambda(1-\eta_2\lambda)B^2   \\
	&- 2\eta_2 (1-\eta_2\lambda) \left( K_t(\overline{W}_t,\overline{V}_t) - K_t(\overline{W}_t, \overline{V}_{t_1}+V^\star) \right),
\end{align*}
where we apply the Cauchy-Schwarz inequality and the convexity of the proxy $K_t$.

For the sake of contradiction, assume that the global loss has not entered the target neighborhood basin, i.e., for $t_1 \leq t < t_1+t_2$, $K_t(\overline{W}_t,\overline{V}_t) - K_t(\overline{W}_t, \overline{V}_{t_1}+V^\star) \geq C$. Let $\lambda B^2 \leq \frac{1}{8}C$ and $\eta_2 \leq \frac{C}{16(\lambda^2 B^2 + L_K^2)}$ (which implies $\eta_2\lambda \leq 1/4$), this ensures the descent term dominates the positive perturbation terms:
\begin{align*}
	\left\|\overline{V}_{t+1}-(\overline{V}_{t_1}+V^\star)\right\|_F^2 
	&\leq \left\|\overline{V}_t-(\overline{V}_{t_1}+V^\star)\right\|_F^2  + 2\eta_2^2(\lambda^2 B^2 + L_K^2) + 2\eta_2\lambda B^2 - 2\eta_2(1-\eta_2\lambda)C \\
	&\leq \left\|\overline{V}_t-(\overline{V}_{t_1}+V^\star)\right\|_F^2 + \frac{1}{8}\eta_2C - \frac{3}{2}\eta_2C + \frac{1}{4}\eta_2C \\
    &\leq \|\overline{V}_t - (\overline{V}_{t_1}+V^\star)\|_F^2 - \eta_2C.
\end{align*}
Thus, evaluating the trajectory over the specialized stage for $t_2 \triangleq \frac{B^2}{\eta_2C}$ iterations, we recursively obtain:
\begin{align*}
	\left\|\overline{V}_{t_1+t_2}-(\overline{V}_{t_1}+V^\star)\right\|_F^2 
	\leq \left\|\overline{V}_{t_1}-(\overline{V}_{t_1}+V^\star)\right\|_F^2  - t_2 \eta_2 C \leq B^2 - t_2 \eta_2 C = 0,
\end{align*}
which is a contradiction. This implies that the trajectory enters the target convergence basin at some time:
\begin{align}\label{eq:K-Kstar}
    K_{t_1+t}(\overline{W}_{t_1+t},\overline{V}_{t_1+t}) - K_{t_1+t}(\overline{W}_{t_1+t}, \overline{V}_{t_1}+V^\star) \leq C, 
\end{align}
where $C=\Theta(\frac{\lambda u \log^2(1/\epsilon_{V,1})}{r})=\Theta(\frac{\log^2 d}{d^{1/4}})$ given $r=\Theta(d^{-9/4}), \lambda=\Theta(d^{-5/2})$, $t_2 = \Theta(\frac{1}{\eta_2 \lambda})$, $\eta_2 = \eta_1 \lambda^2 r^2$.

From the contradiction proof, for any time step $t \leq t_2$ before entering the target basin, the parameter distance contracts:
\begin{align*}
	\left\|\overline{V}_{t_1+t}-(\overline{V}_{t_1}+V^\star)\right\|_F^2 \leq \left\|\overline{V}_{t_1}-(\overline{V}_{t_1}+V^\star)\right\|_F^2 - t \eta_2 C.
\end{align*}
We define the parameter movement (norm growth) from the beginning of the specialized stage as $D_t \triangleq \overline{V}_{t_1+t} - \overline{V}_{t_1}$. Substituting this into the contraction inequality, we obtain $\left\|D_t - V^\star \right\|_F^2 \leq \left\| V^\star \right\|_F^2 - t \eta_2 C$. Since $\|D_t\|_F^2 \geq 0$, it holds that $2\langle D_t, V^\star \rangle \geq t \eta_2 C$. Applying the Cauchy-Schwarz inequality, we have
\begin{align*}
	2\|D_t\|_F \|V^\star\|_F \geq t \eta_2 C \implies \|D_t\|_F \geq \frac{t \eta_2 C}{2 \|V^\star\|_F}.
\end{align*}
With $B = \|V^\star\|_F$ and $t_2=\frac{B^2}{\eta_2 C}$, we evaluate the norm growth over $t = t_2$ iterations:
\begin{align}\label{eq:deltaV_lower_bound}
	\|D_{t_2}\|_F \geq \frac{t_2 \eta_2 C}{2 B} = \frac{\left(\frac{B^2}{\eta_2 C}\right) \eta_2 C}{2 B} = \frac{B}{2}.
\end{align}
This implies that, the gradient descent trajectory must travel a minimum absolute distance of $B/2$ to satisfy the contraction inequality. 
Finally, applying the triangle inequality $\|\overline{V}_{t_1+t_2}\|_F \geq \|D_{t_2}\|_F - \|\overline{V}_{t_1}\|_F$, we measure the final signal norm: $\left\|\overline{V}_{t_1+t_2}\right\|_F \geq \frac{B}{2} - \left\|\overline{V}_{t_1}\right\|_F.$
From Theorem~\ref{the:stage1-Q}, $\|\overline{V}_{t_1}\|_F \lesssim \frac{1}{\text{Poly}(d)}$. We have $B = \|V^\star\|_F = \Theta\left(\frac{u^{1/2}\log(1/\epsilon_{V,1})}{r^{1/2}}\right) = \Theta(d^{9/8} \log d)$ with $r = d^{-9/4}$.

Therefore, we have
\begin{align*}
    \left\|\overline{V}_{t_1+t_2}\right\|_F 
    \gtrsim \frac{u^{1/2}\log(1/\epsilon_{V,1})}{r^{1/2}} - \frac{1}{\text{Poly}(d)}
    \gtrsim d^{9/8} \log d - \frac{1}{\text{Poly}(d)}.
\end{align*}
With $\lim_{d \to \infty}\frac{\|\overline{V}_{t_1+t_2}\|_F}{\|\overline{V}_{t_1} + V^\star\|_F} \gtrsim \lim_{d \to \infty} \frac{d^{9/8} \log d - \frac{1}{\text{Poly}(d)}}{d^{9/8} \log d + \frac{1}{\text{Poly}(d)}} =1$, it implies that $\|\overline{V}_{t_1+t_2}\|_F$ asymptotically converges toward the optimal signal weight $\overline{V}_{t_1} + V^\star$.

\textbf{Finally, we discuss the accumulated gradient $\Delta V_{t}$ during the specialized stage and further establish surrogate loss $K^2_{t_1+t_2}(\overline{V}_{t_1+t_2})$.} For $t$ iterations from $t_1$ to $t_1+t$,
\begin{align*}
    \overline{V}_{t_1+t} = (1-\eta_2\lambda)^{t}\overline{V}_{t_1} - \eta_2 \sum_{s=0}^{t-1} (1-\eta_2\lambda)^{t-1-s} \nabla_{\overline{V}} K_{t_1+s}(\overline{U}_{t_1+s})
    \triangleq \alpha_t \overline{V}_{t_1} + \Delta \overline{V}_t,
\end{align*}
where $\alpha_t = (1-\eta_2\lambda)^{t} \in (0,1)$ is the scaling factor.
We discuss the accumulated gradient $\Delta \overline{V}_t$:
\begin{align}
    \Delta \overline{V}_{t} &= -\eta_2 \sum_{s=0}^{t-1}(1-\eta_2\lambda)^{t-1-s} \nabla_{\overline{V}} K_{t_1+s}(\overline{U}_{t_1+s}) \nonumber\\
    &= -\eta_2 \sum_{s=0}^{t-1}(1-\eta_2\lambda)^{t-1-s} \frac{1}{N}\sum_{n=1}^N \left[\frac{1}{2L}l^\prime\left(f_{t_1+s}^n\right) y_L^n \sum_{j=1}^L y_j^n \mathbbm{1}\left([(X_2^n)^\top]_j V_{t_1+s} x_{L,2}^n\right) ([(X_2^n)^\top]_j)^\top (x_{L,2}^n)^\top\right].\label{eq:Delta_V}
\end{align}
Let $c_{s,n,j}^V \triangleq -\frac{\eta_2}{2NL}(1-\eta_2\lambda)^{t-1-s}l^\prime\left(f_{t_1+s}^n\right)$ denote the non-negative scalar multiplier. We have $0 \leq c_{s,n,j}^V \lesssim \frac{\eta_2}{NL}$.
We then analyze the remaining term $y_L^n y_j^n \mathbbm{1}\left([(X_2^n)^\top]_j V_{t_1+s} x_{L,2}^n\right) ([(X_2^n)^\top]_j)^\top (x_{L,2}^n)^\top$, and decompose it into
\begin{align}
    &y_L^n y_j^n x_{j,2}^n (x_{L,2}^n)^\top \mathbbm{1}\left([(X_2^n)^\top]_j V_{t_1+s} x_{L,2}^n\right) \nonumber\\
    =&\underbrace{\mathbb{E}_{x_{L,2}^n, y_L^n} \left[ y_L^n y_j^n x_{j,2}^n (x_{L,2}^n)^\top \mathbbm{1}\left([(X_2^n)^\top]_j V_{t_1} x_{L,2}^n\right) \mid x_{j,2}^n \right]}_{M_{n,j}^\star} \nonumber\\
    &+ \underbrace{y_L^n y_j^n x_{j,2}^n (x_{L,2}^n)^\top \mathbbm{1}\left([(X_2^n)^\top]_j V_{t_1} x_{L,2}^n\right) - \mathbb{E}_{x_{L,2}^n, y_L^n} \left[ y_L^n y_j^n x_{j,2}^n (x_{L,2}^n)^\top \mathbbm{1}\left([(X_2^n)^\top]_j V_{t_1} x_{L,2}^n\right) \mid x_{j,2}^n \right]}_{Z_{n,j}^V} \nonumber\\
    &+ \underbrace{y_L^n y_j^n x_{j,2}^n (x_{L,2}^n)^\top \left(\mathbbm{1}\left([(X_2^n)^\top]_j V_{t_1+s} x_{L,2}^n\right) -\mathbbm{1}\left([(X_2^n)^\top]_j V_{t_1} x_{L,2}^n\right)\right)}_{\Delta_{s,n,j}}, \label{eq:V-decomposition}
\end{align}
where we define the expected structural component over the query distribution $(x_L, y_L)$, conditioned on a given context token $x_{j,2}^n$, as $M_{n,j}^\star$; we define the zero-mean noise as the deviation from this expectation as $Z_{n,j}^V$ and the activation drift as $\Delta_{s,n,j}$.

\textbf{For term $M_{n,j}^\star$}: Let $\mathcal{H} \triangleq \bigcup_{k=1}^4 \mathcal{H}_k$ denote the structured tokens defined in Definition~\ref{def:H/epsilon-V-t1}, and $\overline{\mathcal{H}} \triangleq [L] \setminus \mathcal{H}$ denote the residual tokens. 
For any token $j \in \mathcal{H}$, let $k(j) \in \{1,2,3,4\}$ denote its subset index. 
We define 
$$
v_{k(j)}^\top \triangleq \mathbb{E}_{x_{L,2}^n, y_L^n} \left[ y_L^n (x_{L,2}^n)^\top \mathbbm{1}\left([(X_2^n)^\top]_j V_{t_1} x_{L,2}^n\right) \right].$$ 

By Definition~\ref{def:H/epsilon-V-t1}, $j \in \mathcal{H}_1, v_1^\top = \frac{1}{4}z^\top + \frac{1}{4}\zeta^\top$; $j \in \mathcal{H}_2, v_2^\top = -\frac{1}{4}z^\top + \frac{1}{4}\zeta^\top$; $j \in \mathcal{H}_3$, $v_3^\top = -\frac{1}{4}z^\top - \frac{1}{4}\zeta^\top$, $j \in \mathcal{H}_4, v_4^\top = \frac{1}{4}z^\top - \frac{1}{4}\zeta^\top$.
Beyond $\mathcal{H}_1\sim\mathcal{H}_4$, the remaining four logical sign configurations for $\{[X^\top]_i V_{t_1} x_L \}, {x_L \in \{z-\zeta, z, z+\zeta\}}$ contribute zero to $v^\top$. Specifically, due to $z = \frac{1}{2}((z-\zeta) + (z+\zeta))$, it is impossible for the pre-activation at $z$ to have an opposite sign to both $z-\zeta$ and $z+\zeta$, making those two corresponding subsets empty. For the other two subsets where all three signs are negative or non-negative, $v^\top$ trivially vanishes. 
In total, for any token $j \in \overline{\mathcal{H}}$, we have $M_{n,j}^\star \equiv 0$.
For any $j \in \mathcal{H}$, we separate $v_{k(j)}^\top$ into its principal signal and orthogonal residual (since $\langle z, \zeta \rangle = 0$): $v_{k(j)}^\top = c_{k(j)} z^\top + d_{k(j)} \zeta^\top,$ where $|c_{k(j)}|=\frac{1}{4}, |d_{k(j)}|=\frac{1}{4}$. 
Thus, for $j \in \mathcal{H}$, we have:
\begin{align*}
    M_{n,j}^\star = y_j^n x_{j,2}^n v_{k(j)}^\top = c_{k(j)} y_j^n x_{j,2}^n z^\top + d_{k(j)} y_j^n x_{j,2}^n \zeta^\top.
\end{align*}

Equation~\ref{eq:Delta_V} is translated into
\begin{align}
    \Delta \overline{V}_t &= \sum_{s=0}^{t-1} \sum_{n=1}^N \sum_{j \in \mathcal{H}} c_{s,n,j}^V \left( c_{k(j)} y_j^n x_{j,2}^n z^\top + d_{k(j)} y_j^n x_{j,2}^n \zeta^\top \right)  + \sum_{s=0}^{t-1} \sum_{n=1}^N \sum_{j=1}^L c_{s,n,j}^V Z_{n,j}^V + \sum_{s=0}^{t-1} \sum_{n=1}^N \sum_{j=1}^L c_{s,n,j}^V \Delta_{s,n,j} \nonumber\\
    &\triangleq \boldsymbol{\mu}^V_{t_1+t} z^\top + \boldsymbol{\nu}^V_{t_1+t} \zeta^\top + E_{t_1+t}^V + \Delta_{t_1+t}^V,\label{eq:Delta_V2}
\end{align}
where ${\boldsymbol{\mu}}^V_{t_1+t} \triangleq \sum_{s=0}^{t-1} \sum_{n=1}^N \sum_{j \in \mathcal{H}} c_{s,n,j}^V c_{k(j)} y_j^n x_{j,2}^n$; $\boldsymbol{\nu}^V_{t_1+t} \triangleq \sum_{s=0}^{t-1} \sum_{n=1}^N \sum_{j \in \mathcal{H}} c_{s,n,j}^V d_{k(j)} y_j^n x_{j,2}^n$; the zero-mean noise is $E_{t_1+t}^V \triangleq \sum_{s=0}^{t-1} \sum_{n=1}^N \sum_{j=1}^L c_{s,n,j}^V Z_{n,j}^V $; the activation drift is $\Delta_{t_1+t}^V \triangleq \sum_{s=0}^{t-1} \sum_{n=1}^N \sum_{j=1}^L c_{s,n,j}^V \Delta_{s,n,j}$.

For the noise matrix $E_{t_1+t}^V$, we have $Z_{n,j}^{V,(u,v)} \lesssim 1/d$ for any coordinate $(u,v) \in [d] \times [d]$. Given $0 \leq c_{s,n,j}^V \lesssim \frac{\eta_2}{NL}$, applying Hoeffding's inequality and employing a union bound over all $d^2$ coordinates with probability $1-1/d$, $ \|E_{t_1+t}^V\|_F \lesssim \eta_2 t \sqrt{\frac{\log d}{NL}}.$
For $\Delta_{t_1+t}^V$, by Corollary~\ref{coro:lemma-ac-V-t1+t-t} and $0 \leq c_{s,n,j}^V \lesssim \frac{\eta_2}{NL}$, we have $\|\Delta_{t_1+t}^V\|_F \lesssim \frac{\eta_2 t}{L} \left( \epsilon_V + L\sqrt{\frac{\eta_2}{\eta_1}} + \sqrt{L \log d} \right).$
Evaluating at the end of the specialized stage $t_2 = \Theta(\frac{1}{\eta_2 \lambda})$, with $L, N = \Theta(\text{Poly}(d))$ and $\eta_2=\eta_1\lambda^2r^2$, we obtain:
\begin{align*}
    \|E_{t_1+t_2}^V\|_F \lesssim \frac{1}{\lambda \sqrt{NL}}, \quad
    \|\Delta_{t_1+t_2}^V\|_F \lesssim \frac{\epsilon_V}{\lambda L},
\end{align*}
where these terms satisfy $\mathcal{O}(1/\text{Poly}(d))$ given $\lambda =\Theta(1/d^{5/2})$, $\epsilon_V = \Theta(d^4L^{2/3})$ and large polynomial $L, N = \Theta(\text{Poly}(d))$.

From Equation~\ref{eq:deltaV_lower_bound} established in the trajectory contraction, we have: $\|D_{t_2}\|_F \geq \frac{1}{2}\|V^\star\|_F = \Theta(d^{9/8}\log d)$ and $\|D_{t_2}\|_F \leq \|V^\star\|_F = \Theta(d^{9/8}\log d)$, i.e., $\|D_{t_2}\|_F=\Theta(d^{9/8}\log d)$.
Since $D_{t_2} = \overline{V}_{t_1+t_2} - \overline{V}_{t_1}= \Delta \overline{V}_{t_2} - (1-\alpha)\overline{V}_{t_1}$, and $\|\overline{V}_{t_1}\|_F = \mathcal{O}(1/\text{Poly}(d))$, we have $\|\Delta \overline{V}_{t_2}\|_F = \Theta(\|D_{t_2}\|_F) = \Theta(d^{9/8}\log d).$
Applying the triangle inequality to $\Delta \overline{V}_{t_2}$:
\begin{align*}
    &\|\boldsymbol{\mu}^V_{t_1+t_2} z^\top + \boldsymbol{\nu}^V_{t_1+t_2} \zeta^\top \|_F \geq \|\Delta \overline{V}_{t_2}\|_F  - \|E_{t_1+t_2}^V\|_F - \|\Delta_{t_1+t_2}^V\|_F = \Theta(d^{9/8}\log d), \\
    &\|\boldsymbol{\mu}^V_{t_1+t_2} z^\top + \boldsymbol{\nu}^V_{t_1+t_2} \zeta^\top \|_F \leq \|\Delta \overline{V}_{t_2}\|_F  + \|E_{t_1+t_2}^V\|_F + \|\Delta_{t_1+t_2}^V\|_F = \Theta(d^{9/8}\log d).
\end{align*}
where $\|E_{t_1+t_2}^V\|_F, \|\Delta_{t_1+t_2}^V\|_F =\mathcal{O}(1/\text{Poly}(d))$ are asymptotically absorbed into the dominant scale.
Since $\langle z, \zeta \rangle = 0$,
\begin{align*}
    \|\boldsymbol{\mu}^V_{t_1+t_2} z^\top + \boldsymbol{\nu}^V_{t_1+t_2} \zeta^\top\|_F^2 = \|\boldsymbol{\mu}^V_{t_1+t_2}\|_2^2 \|z\|_2^2 + \|\boldsymbol{\nu}^V_{t_1+t_2}\|_2^2 \|\zeta\|_2^2 = \Theta(d^{9/4}\log^2 d),
\end{align*}
With $|c_{k(j)}| = |d_{k(j)}| = \frac{1}{4}$, $\boldsymbol{\mu}^V_{t_1+t_2}$ and $\boldsymbol{\nu}^V_{t_1+t_2}$ scale proportionally in magnitude. Given that $\|z\|_2 = \Theta(1)$ and $\|\zeta\|_2 = r=\Theta(1/d^{9/4})$ (in fact, this implies $\boldsymbol{\mu}^V_{t_1+t} z^\top$ dominates $\boldsymbol{\nu}^V_{t_1+t} \zeta^\top$), we obtain
\begin{align*}
    \|\boldsymbol{\mu}^V_{t_1+t_2}\|_2 = \Theta(d^{9/8}\log d), \quad \text{and} \quad \|\boldsymbol{\nu}^V_{t_1+t_2}\|_2 =\Theta(d^{9/8}\log d).
\end{align*}

By the definition of surrogate loss $K^2_{t_1+t_2}(\overline{V}_{t_1+t_2})$, we next consider the margin $y^n_L N_{V_{t_1+t_2}}(\overline{V}_{t_1+t_2} +\widetilde{V}_{t_1+t_2}; X_2^n,Y^n)$. We decompose the total weight into the state at iteration $t_1$, signal shift $D_{t_2} = \overline{V}_{t_1+t_2} - \overline{V}_{t_1}$, and noise drift $ \widetilde{D}_{t_2} = \widetilde{V}_{t_1+t_2} - \widetilde{V}_{t_1}$, i.e., $V_{t_1+t_2} = V_{t_1} + D_{t_2} + \widetilde{D}_{t_2}$.
\begin{align}
    &y_L^n N_{V_{t_1+t_2}}(\overline{V}_{t_1+t_2} +\widetilde{V}_{t_1+t}; X_2^n,Y^n) \nonumber\\
    =&\frac{1}{L} \sum_{j=1}^L y_L^n y_j^n \left(\mathbbm{1}\left([(X_2^n)^\top]_j V_{t_1+t_2}x_{L,2}^n\right) - \mathbbm{1}\left([(X_2^n)^\top]_j V_{t_1}x_{L,2}^n\right)\right) \left( [(X_2^n)^\top]_j V_{t_1+t_2} x_{L,2}^n \right) \nonumber\\
    &+ \frac{1}{L} \sum_{j=1}^L y_L^n y_j^n \mathbbm{1}\left([(X_2^n)^\top]_j V_{t_1}x_{L,2}^n\right)\left([(X_2^n)^\top]_jV_{t_1}x_{L,2}^n \right)
    + \frac{1}{L} \sum_{j=1}^L y_L^n y_j^n \mathbbm{1}\left([(X_2^n)^\top]_j V_{t_1}x_{L,2}^n\right)\left([(X_2^n)^\top]_j (D_{t_2} + \widetilde{D}_{t_2}) x_{L,2}^n\right).\label{eq:yNv_t1t2_decompose}
\end{align}

For the first term in Equation~\ref{eq:yNv_t1t2_decompose}, using Corollary~\ref{coro:lemma-ac-V-t1+t-t},
\begin{align*}
    \left|\text{LHS} \right|
    \leq& \frac{1}{L} \left\| \mathbbm{1}\left((X_2^n)^\top V_{t_1 + t_2}x_{L,2}^n\right) - \mathbbm{1}\left((X_2^n)^\top V_{t_1}x_{L,2}^n\right) \right\|_1 \cdot \max_{j} \left| [(X_2^n)^\top]_j V_{t_1+t_2} x_{L,2}^n \right| \\
    \leq& \frac{1}{L}\left(\epsilon_V + L\sqrt{\frac{\eta_2}{\eta_1}} + \sqrt{L\log d}\right) (u+r)^2 \|V_{t_1+t_2}\|_F,
\end{align*}
where $\eta_2=\eta_1\lambda^2 r^2$. We established that $\|\overline{V}_{t_1+t_2}\|_F = \Omega(d^{9/8}\log d)$, target signal $\|\overline{V}_{t_1}+V^\star\|_F = \Theta(d^{9/8}\log d)$, noise weight $\|\widetilde{V}_{t_1+t_2}\|_F \lesssim \tau_0 d = \Theta(\sqrt{d})$, thus $(u+r)^2 \|{V}_{t_1+t_2}\|_F$ scales at most polynomially with $d$. Together with $\epsilon_{V,1}=\Theta(d^4L^{2/3})$ and a large polynomial prompt length $L = \Theta(\text{Poly}(d))$, the upper bound is $\mathcal{O}(1/\text{Poly}(d))$.

For the second term in Equation~\ref{eq:yNv_t1t2_decompose}: we have $\|\overline{V}_{t_1}\|_F \lesssim 1/\text{Poly}(d)$ (Theorem~\ref{the:stage1-Q}) and $\|\widetilde{V}_{t_1}\|_F \lesssim \tau_0 d=\Theta(\sqrt{d})$ with $\widetilde{V}_{t_1} \sim \mathcal{N}(0,\tau_0^2 I)$ and $\tau_0 =\Theta(1/\sqrt{d})$. Then $\|V_{t_1}\|_F \lesssim 1/\text{Poly}(d) + \sqrt{d}$. Thus,
\begin{align*}
    \left| \text{LHS} \right| \leq \max_j \left| [(X_2^n)^\top]_j V_{t_1} x_{L,2}^n \right| \lesssim (u+r)^2 \|V_{t_1}\|_F \lesssim \sqrt{d}.
\end{align*}

For the third term in Equation~\ref{eq:yNv_t1t2_decompose}: recall that $D_{t_2} = \Delta \overline{V}_{t_2} - (1-\alpha)\overline{V}_{t_1}$, together with $\Delta \overline{V}_{t_2} = \boldsymbol{\mu}^V_{t_1+t_2} z^\top + \boldsymbol{\nu}^V_{t_1+t_2} \zeta^\top + E_{t_1+t_2}^V + \Delta_{t_1+t_2}^V$, we have $D_{t_2} + \widetilde{D}_{t_2} = \boldsymbol{\mu}^V_{t_1+t_2} z^\top + \boldsymbol{\nu}^V_{t_1+t_2} \zeta^\top + R_{t_2}$, where the residual matrix is $R_{t_2} \triangleq E_{t_1+t_2}^V + \Delta_{t_1+t_2}^V - (1-\alpha)\overline{V}_{t_1} + \widetilde{D}_{t_2}$.
\begin{align*}
    \text{LHS}=\frac{1}{L} \sum_{j=1}^L y_L^n y_j^n \mathbbm{1}\left([(X_2^n)^\top]_j V_{t_1}x_{L,2}^n\right) [(X_2^n)^\top]_j (\boldsymbol{\mu}^V_{t_1+t_2} z^\top + \boldsymbol{\nu}^V_{t_1+t_2} \zeta^\top + R_{t_2}) x_{L,2}^n 
\end{align*}
Among the components of $R_{t_2}$, we established $\|E_{t_1+t_2}^V\|_F, \|\Delta_{t_1+t_2}^V\|_F, \|\overline{V}_{t_1}\|_F = \mathcal{O}(1/\text{Poly}(d))$. The noise drift satisfies $\|\widetilde{D}_{t_2}\|_F \leq \|\widetilde{V}_{t_1+t_2}\|_F + \|\widetilde{V}_{t_1}\|_F \lesssim \tau_0 d = \Theta(\sqrt{d})$. Therefore, the residual is bounded by $\|R_{t_2}\|_F \lesssim \sqrt{d}$. Then:
\begin{align*}
    \left| \frac{1}{L} \sum_{j=1}^L y_L^n y_j^n \mathbbm{1}\left([(X_2^n)^\top]_j V_{t_1}x_{L,2}^n\right) [(X_2^n)^\top]_j R_{t_2} x_{L,2}^n \right|
    \leq \max_j \left\|[(X_2^n)^\top]_j\right\|_2 \left\|R_{t_2}\right\|_F \|x_{L,2}^n\|_2 \lesssim \sqrt{d}.
\end{align*}
Next, for the components regarding $\boldsymbol{\mu}_{t_1+t_2}^V$ and $\boldsymbol{\nu}_{t_1+t_2}^V$:
\begin{align*}
    \frac{1}{L} \sum_{j=1}^L y_L^n y_j^n \mathbbm{1}\left([(X_2^n)^\top]_j V_{t_1}x_{L,2}^n\right) (x_{j,2}^n)^\top \boldsymbol{\mu}_{t_1+t_2}^V (z^\top x_{L,2}^n)
    + \frac{1}{L} \sum_{j=1}^L y_L^n y_j^n \mathbbm{1}\left([(X_2^n)^\top]_j V_{t_1}x_{L,2}^n\right)(x_{j,2}^n)^\top \boldsymbol{\nu}_{t_1+t_2}^V (\zeta^\top x_{L,2}^n).
\end{align*}
The term regarding $\boldsymbol{\nu}_{t_1+t_2}^V$ satisfies $\frac{r^2}{L} \sum_{j=1}^L y_j^n \mathbbm{1}\left([(X_2^n)^\top]_j V_{t_1}x_{L,2}^n\right) (x_{j,2}^n)^\top \boldsymbol{\nu}_{t_1+t_2}^V= \Theta(r^2\|\boldsymbol{\nu}_{t_1+t_2}^V\|_2) = \Theta(d^{-81/16}\log d).$
The term regarding $\boldsymbol{\mu}_{t_1+t_2}^V$ satisfies $\frac{u^2}{L} \sum_{j=1}^L y_j^n \mathbbm{1}\left([(X_2^n)^\top]_j V_{t_1}x_{L,2}^n\right) (x_{j,2}^n)^\top \boldsymbol{\mu}_{t_1+t_2}^V = \Theta\left( \|\boldsymbol{\mu}_{t_1+t_2}^V\|_2 \right) = \Theta(d^{9/8}\log d)$. Thus the term $\boldsymbol{\mu}_{t_1+t_2}^V$ is dominant. The third term in Equation~\ref{eq:yNv_t1t2_decompose} satisfies:
\begin{align*}
    \frac{1}{L} \sum_{j=1}^L y_L^n y_j^n \mathbbm{1}\left([(X_2^n)^\top]_j V_{t_1}x_{L,2}^n\right) \left( [(X_2^n)^\top]_j (D_{t_2} + \widetilde{D}_{t_2}) x_{L,2}^n \right)
    \gtrsim d^{9/8}\log d - \sqrt{d}.
\end{align*}

Combining above, by the inequality of $a + b + c \ge a - |b| - |c|$, Equation~\ref{eq:yNv_t1t2_decompose} satisfies:
\begin{align*}
    y_L^n N_{V_{t_1+t_2}} \gtrsim d^{9/8}\log d - \sqrt{d} - \frac{1}{\text{Poly}(d)}.
\end{align*}

Therefore, the surrogate loss $K^2_{t_1+t_2}(\overline{V}_{t_1+t_2})$ satisfies
\begin{align*}
	K^2_{t_1+t_2}(\overline{V}_{t_1+t_2}) 
	&= \frac{1}{N} \sum_{n=1}^N \log \left(1+\exp\left(-y^n_L N_{V_{t_1+t_2}}(\overline{V}_{t_1+t_2} +\widetilde{V}_{t_1+t_2}; X_2^n,Y^n) \right)\right)\\
    &\lesssim \frac{1}{N} \sum_{n=1}^N \log \left(1+\exp\left(-d^{9/8}\log d  + \sqrt{d} + \frac{1}{\text{Poly}(d)}\right)\right)\\
    &\lesssim \exp\left(-d^{9/8}\log d\right),
\end{align*}
which converges to $0$ with large $d$. 
We conclude that the loss on nonlinear separable component $\mathcal{Q}$ is small and the model learns specialized knowledge during the specialized stage.
\end{proof}

\subsection{Proof for the Specialized Stage: Proof of Theorem \ref{the:stage2-P}}\label{proof:stage2-P}

\begin{proof}
\textbf{During the specialized stage, we consider the gradient descent of $\overline{W}$.} For any $t \in [0, t_2]$,
\begin{align*}
    \overline{W}_{t_1+t} &= (1-\eta_2\lambda)\overline{W}_{t_1+t-1} - \eta_2 \nabla_{\overline{W}} K_{t_1+t-1}(\overline{U}_{t_1+t-1}).
\end{align*}
For $t$ iterations from $t_1$ to $t_1+t$,
\begin{align*}
    \overline{W}_{t_1+t} = (1-\eta_2\lambda)^{t}\overline{W}_{t_1} - \eta_2 \sum_{s=0}^{t-1} (1-\eta_2\lambda)^{t-1-s} \nabla_{\overline{W}} K_{t_1+s}(\overline{U}_{t_1+s}) 
    \triangleq \alpha_t \overline{W}_{t_1} + \Delta \overline{W}_t,
\end{align*}
where $\alpha_t = (1-\eta_2\lambda)^{t} \in (0,1)$ is the scaling factor, and $\Delta \overline{W}_t$ is the accumulated gradient perturbation.

From Appendix~\ref{proof:stage1-P}, we derive that $\overline{W}_{t_1}=\mu_{t_1}w^\star (w^\star)^\top + E_{t_1}$ where $\mu_{t_1} \gg \|E_{t_1}\|_F$ and $\mu_{t_1}\approx \|\overline{W}_{t_1}\|_F = \Omega(d\log d-d\log\log d-\sqrt{d}\log d)$.
Similarly to Appendix~\ref{proof:stage1-P} (analyze the signal weight), we discuss the accumulated gradient $\Delta \overline{W}_t$:
\begin{align*}
    \Delta \overline{W}_t &= -\eta_2 \sum_{s=0}^{t-1}(1-\eta_2\lambda)^{t-1-s} \nabla_{\overline{W}} K_{t_1+s}(\overline{U}_{t_1+s}) \\
    &= -\eta_2 \sum_{s=0}^{t-1}(1-\eta_2\lambda)^{t-1-s} \frac{1}{N}\sum_{n=1}^N \left[\frac{1}{2L}l^\prime\left(f_{t_1+s}^n\right) y_L^n \sum_{j=1}^L y_j^n \mathbbm{1}\left([(X_1^n)^\top]_j W_{t_1+s} x_{L,1}^n\right) ([(X_1^n)^\top]_j)^\top (x_{L,1}^n)^\top\right].
\end{align*}
Let $c_{s,n,j}^\prime \triangleq -\frac{\eta_2}{2NL}(1-\eta_2\lambda)^{t-1-s}l^\prime\left(f_{t_1+s}^n\right)\mathbbm{1}\left([(X_1^n)^\top]_j W_{t_1+t} x_{L,1}^n\right) \geq 0$ denote the non-negative scalar multiplier (the derivative $l^\prime<0$). Substituting the data generation process $([(X_1^n)^\top]_j)^\top = y_j^n \gamma_0 w^\star + e_j^n$ and $x_{L,1}^n = y_L^n \gamma_0 w^\star + e_L^n$, the outer product of the tokens expands as:
\begin{align*}
    y_L^n y_j^n ([(X_1^n)^\top]_j)^\top (x_{L,1}^n)^\top &= y_L^n y_j^n \left( y_j^n \gamma_0 w^\star + e_j^n \right) \left( y_L^n \gamma_0 w^\star + e_L^n \right)^\top \\
    &= \gamma_0^2 w^\star (w^\star)^\top + \gamma_0 y_L^n w^\star (e_L^n)^\top + \gamma_0 y_j^n e_j^n (w^\star)^\top + y_L^n y_j^n e_j^n (e_L^n)^\top.
\end{align*}
Summing these expanded terms, denote $\Delta \overline{W}_t = \mu_{t_1+t}^\prime w^\star (w^\star)^\top + E_{t_1+t}^\prime$, where:
\begin{align*}
    \mu_{t_1+t}^\prime = \gamma_0^2 \sum_{s=0}^{t-1} \sum_{n=1}^N \sum_{j=1}^L c_{s,n,j}^\prime, \quad
    E_{t_1+t}^\prime = \sum_{s=0}^{t-1} \sum_{n=1}^N \sum_{j=1}^L c_{s,n,j}^\prime \Big( \gamma_0 y_L^n w^\star (e_L^n)^\top + \gamma_0 y_j^n e_j^n (w^\star)^\top + y_L^n y_j^n e_j^n (e_L^n)^\top \Big).
\end{align*}
With $|l^\prime(z)| = \frac{e^{-z}}{1+e^{-z}} \le \log(1+e^{-z}) = l(z)$, then
\begin{align*}
    \sum_{n=1}^N c_{s,n,j}^\prime &\le \frac{\eta_2}{2NL} (1-\eta_2\lambda)^{t-1-s} \sum_{n=1}^N |l^\prime(f_{t_1+s}^n)|
    \le \frac{\eta_2}{2L} (1-\eta_2\lambda)^{t-1-s} \left( \frac{1}{N} \sum_{n=1}^N l(f_{t_1+s}^n) \right) \\
    &= \frac{\eta_2}{2L} (1-\eta_2\lambda)^{t-1-s} K_{t_1+s} \lesssim \frac{\eta_2}{L} \mathcal{E},
\end{align*}
We then have $0 \leq \mu_{t_1+t}^\prime \lesssim \gamma_0^2\eta_2t \mathcal{E}$; $\sum_{s=0}^{t-1} c_{s,n,j}^\prime \lesssim \frac{\eta_2 t}{NL}\mathcal{E}$.
Similarly to the technique of analyzing $E_{t_1}$ in Appendix~\ref{proof:stage1-P}, we analyze $E_{t_1+t}^\prime$ coordinate-wise. Let $X_{n,j}^{\prime(u,v)} \triangleq \left( \sum c_{s,n,j}^\prime \right) Z_{n,j}^{(u,v)}$. Since the coordinates of $Z_{n,j}$ are sub-exponential random variables with $\psi_1$-norm bounded by $\mathcal{O}(1/d)$, the independent zero-mean variables $X_{n,j}^{\prime(u,v)}$ satisfy $\sum_{n=1}^N \sum_{j=1}^L \left\|X_{n,j}^{\prime(u,v)}\right\|_{\psi_1}^2 \lesssim \frac{\eta_2^2 t^2 \mathcal{E}^2}{NLd^2}$ and $\max_{n,j} \left\|X_{n,j}^{\prime(u,v)}\right\|_{\psi_1} \lesssim\frac{\eta_2 t \mathcal{E}}{NL d}$.
With Bernstein's inequality in Lemma~\ref{lemma:Bernstein-sub-exp} and taking the union bound over all $d^2$ coordinates, with probability at least $1-1/d$,
\begin{align*}
    \|E_{t_1+t}^\prime\|_F \lesssim \sqrt{ d^2 \left(\frac{\eta_2 t \mathcal{E}}{d} \sqrt{\frac{\log d}{NL}}\right)^2 } = \eta_2 t \mathcal{E} \sqrt{\frac{\log d}{NL}}.
\end{align*}

\textbf{Evaluating the accumulation at $t = t_2$}, we denote $\alpha \triangleq \alpha_{t_2} = (1-\eta_2\lambda)^{t_2} \approx 1/e$. With $\gamma_0^2 = 1/d$ and $t_2=\Theta(\frac{1}{\eta_2\lambda})$, the target signal reinforcement $\mu_{t_1+t_2}^\prime$ is definitively bounded by:
\begin{align*}
    0 \leq \mu_{t_1+t_2}^\prime \lesssim \gamma_0^2 \eta_2 t_2 \mathcal{E} = \Theta\left(\frac{\mathcal{E}}{d\lambda}\right).
\end{align*}
Substituting $\mathcal{E} = \Theta(\frac{\log^2 d}{\sqrt{d}})$ and $\lambda=\Theta(1/d^{5/2})$, we obtain $\mu_{t_1+t_2}^\prime \lesssim d\log^2 d$.
With polynomial datasize $N = \Theta(\text{Poly}(d))$, prompt length $L=\Theta(\text{Poly}(d))$,
\begin{align*}
    \|E_{t_1+t_2}^\prime\|_F \lesssim \eta_2 t_2 \mathcal{E} \sqrt{\frac{\log d}{NL}} = \Theta\left(\frac{\mathcal{E}}{\lambda \sqrt{NL}}\right),
\end{align*}
which is of $1/\text{Poly}(d)$ order. We have that $\mu_{t_1+t_2}^\prime \gg \|E_{t_1+t_1}^\prime\|_F$.

In total, for $\Delta \overline{W}_{t_2} = \mu_{t_1+t_2}^\prime w^\star (w^\star)^\top + E_{t_1+t_2}^\prime$, $\mu_{t_1+t_2}^\prime \geq 0$ and $\mu_{t_1+t_2}^\prime \gg \|E_{t_1+t_1}^\prime\|_F$  reveal that the gradient accumulated during the specialized stage structurally reinforces the true target signal direction. Furthermore,
\begin{align*}
    \|\overline{W}_{t_1+t_2} - \alpha \overline{W}_{t_1}\|_F \lesssim \|\Delta \overline{W}_{t_2}\|_F \lesssim \Theta(d\log^2 d).
\end{align*}
This implies that the signal reinforce is up to $d\log^2 d$.
    
\textbf{We bound the surrogate loss $K^1_{t_1+t_2}(\overline{W}_{t_1+t_2})$ at iteration $t_1+t_2$.}\quad
From Corollary~\ref{coro:N-W-tW}, we have $|N_{W_{t_1+t_2}}(\widetilde{W}_{t_1+t_2}; X_1^n,Y^n)|\lesssim \epsilon_{W,1}$. Leveraging $\overline{W}_{t_1+t_2} = \alpha \overline{W}_{t_1} + \Delta \overline{W}_{t_2}$ ($\alpha = \alpha_{t_2} = (1-\eta_2\lambda)^{t_2}$),
\begin{align}
    \frac{1}{2}y_L^n N_{W_{t_1+t_2}}(W_{t_1+t_2}; X_1^n,Y^n)
    \geq& \frac{1}{2}y_L^n N_{W_{t_1+t_2}}(\overline{W}_{t_1+t_2}; X_1^n,Y^n)- \left|\frac{1}{2}y_L^n N_{W_{t_1+t_2}}(\widetilde{W}_{t_1+t_2}; X_1^n,Y^n)\right| \nonumber\\
    \geq& \underbrace{ \frac{\alpha}{2}y_L^n N_{W_{t_1}}(\overline{W}_{t_1}; X_1^n,Y^n) }_{\text{Term 1}} - \underbrace{ \frac{\alpha}{2}\left| y_L^n \left(N_{W_{t_1+t_2}}(\overline{W}_{t_1}; X_1^n,Y^n) - N_{W_{t_1}}(\overline{W}_{t_1}; X_1^n,Y^n)\right) \right| }_{\text{Term 2}} \nonumber\\
    &+ \underbrace{ \frac{1}{2}y_L^n N_{W_{t_1+t_2}}(\Delta \overline{W}_{t_2}; X_1^n,Y^n) }_{\text{Term 3}} -  \frac{1}{2}\epsilon_{W,1}, \label{eq:yN_t1t2}
\end{align}

\textbf{For Term 2}, from Corollary~\ref{coro:lemma-ac-W-t1+t-t}, denote $\mathcal{E}^\prime \triangleq (\epsilon_W+L\sqrt{\frac{\eta_2}{\eta_1}}+ \sqrt{L \log d})^{1/2}K_l (1+\gamma_0)^4 (\log d)^{1/2} \lambda^{-1} L^{-1},$ we then have 
\begin{align}\label{eq:KW_t1t2_term2}
    \frac{\alpha}{2}\left| y_L^n \left(N_{W_{t_1+t_2}}(\overline{W}_{t_1}; X_1^n,Y^n)  - N_{W_{t_1}}(\overline{W}_{t_1}; X_1^n,Y^n)\right) \right| \lesssim \frac{\alpha}{2} \mathcal{E}^\prime.
\end{align}

\textbf{For Term 1}, let $M_n \triangleq y_L^n N_{W_{t_1}}(\overline{W}_{t_1}; X_1^n, Y^n)$ and $\mathcal{A}_n = \{j \in [L] \mid y_L^n y_j^n = 1\}$, with $\overline{W}_{t_1} \approx \mu_{t_1} w^\star (w^\star)^\top$ (Appendix~\ref{proof:stage1-P}),
\begin{align*}
    M_n &\approx \frac{|\mathcal{A}_n|}{L} \underbrace{ \mu_{t_1} \gamma_0^2 }_{\Delta_1} + \underbrace{ \frac{|\mathcal{A}_n|}{L} \mu_{t_1} \gamma_0 y_L^n (w^\star)^\top e_L^n }_{\Delta_2} + \underbrace{ \frac{\mu_{t_1}}{L} \sum_{j \in \mathcal{A}_n} \left( \gamma_0 y_j^n (e_j^n)^\top w^\star + y_L^n y_j^n (e_j^n)^\top w^\star (w^\star)^\top e_L^n \right) }_{\Delta_3}.
\end{align*}
Since $|\mathcal{A}_n| \sim \text{Binomial}(L, 1/2)$, $e_L^n \sim \mathcal{N}(0, I/d)$, $\Delta_1$ is gaussian and $\Delta_2$ is sub-exponential, we apply standard concentration inequalities (Lemmas~\ref{lemma:Hoeffding}, \ref{lemma:chernoff-gaussian} and \ref{lemma:Bernstein-sub-exp}) respectively. Taking the union bound, $M_n$ concentrates around its expectation with high probability $1-\delta$:
\begin{align*}
    \left| M_n - \mathbb{E}[M_n] \right| \lesssim \mu_{t_1} \gamma_0^2 \sqrt{\frac{\log(1/\delta)}{L}} + \mu_{t_1} \gamma_0 \sqrt{\frac{\log(1/\delta)}{d}} + \frac{\mu_{t_1}}{d} \frac{\log(1/\delta)}{L} \triangleq t(\delta).
\end{align*}
Given $\mu_{t_1} = \Omega(d\log d)$, $\gamma_0=1/d$, and $L=\Theta(\text{Poly}(d))$, we have $t(\delta) = \mathcal{O}(1/\text{Poly}(d))$.

From Equation~\ref{eq:yh-second-term}, we have $\max_n \left|y_L^n N_{W_{t_1}}(\widetilde{W}_{t_1}; X_1^n, Y^n)\right| \lesssim \frac{\log d}{\sqrt{d}}$. Together with large $N=\Theta(\text{Poly}(d))$ and Equation~\ref{eq:mathcalE},
\begin{align*}
    \mathbb{E}[M_n] &\approx \frac{1}{N} \sum_{n=1}^N M_n
    = \frac{1}{N} \sum_{n=1}^N y_L^n h_{t_1}(X_1^n) - \frac{1}{N} \sum_{n=1}^N y_L^n N_{W_{t_1}}(\widetilde{W}_{t_1}; X_1^n,Y^n) \\
    &\geq \frac{1}{N} \sum_{n=1}^N y_L^n h_{t_1}(X_1^n) - \max_n \left| y_L^n N_{W_{t_1}}(\widetilde{W}_{t_1}; X_1^n,Y^n) \right| \\
    &\gtrsim \log(1/\mathcal{E}) - \frac{\log d}{\sqrt{d}} \approx \log(1/\mathcal{E}).
\end{align*}
Thus, from $\left| M_n - \mathbb{E}[M_n] \right| \le t(\delta)$, we have $M_n \ge \mathbb{E}[M_n] - t(\delta) \gtrsim \log(1/\mathcal{E}) - t(\delta).$
With $\mathcal{E}= \frac{1}{\text{Poly}(d)} +\frac{(\log d)^2}{\sqrt{d}}+  \left(\frac{\log d}{N}\right)^{1/4} = \Theta(\frac{(\log d)^2}{\sqrt{d}})$, $\log(1/\mathcal{E}) = \Theta(\log d)$ and $t(\delta) = \Theta(1/\text{Poly}(d))$,
\begin{align}\label{eq:KW_t1t2_term1}
    \frac{\alpha}{2}y_L^n N_{W_{t_1}}(\overline{W}_{t_1}; X_1^n, Y^n) \gtrsim \frac{\alpha}{2}\log(1/\mathcal{E}).
\end{align}

\textbf{For Term 3,} from $\Delta \overline{W}_{t_2} = \mu^\prime_{t_1+t_2} w^\star (w^\star)^\top + E^\prime_{t_1+t_2}$, we have $\mu^\prime_{t_1+t_2} \ge 0$. Then,
\begin{align*}
    \frac{1}{2}y_L^n N_{W_{t_1+t_2}}(\Delta \overline{W}_{t_2}; X_1^n,Y^n) &= \frac{1}{2} y_L^n  N_{W_{t_1+t_2}}(\mu^\prime_{t_1+t_2} w^\star (w^\star)^\top; X_1^n,Y^n) + \frac{1}{2}y_L^n N_{W_{t_1+t_2}}(E^\prime_{t_1+t_2}; X_1^n,Y^n).
\end{align*}
Similarly to $M_n$ analyzed in Term 1, let $M_n^\prime \triangleq y_L^n N_{W_{t_1+t_2}}(\mu^\prime_{t_1+t_2} w^\star (w^\star)^\top; X_1^n, Y^n)$ and $\mathcal{A}_n^\prime = \{j \in [L] \mid y_L^n y_j^n = 1\}$,
\begin{align*}
    M_n^\prime
    &= \frac{1}{L} \sum_{j \in \mathcal{A}_n^\prime} \mu^\prime_{t_1+t_2}\left( y_j^n \gamma_0 + (e_j^n)^\top w^\star \right) \left( y_L^n \gamma_0 + (w^\star)^\top e_L^n \right) \\
    &= \frac{|\mathcal{A}_n^\prime|}{L} \mu^\prime_{t_1+t_2}\gamma_0^2 + \frac{|\mathcal{A}_n^\prime|}{L} \mu^\prime_{t_1+t_2}\gamma_0 y_L^n (w^\star)^\top e_L^n + \frac{\mu^\prime_{t_1+t_2}}{L} \sum_{j \in \mathcal{A}_n^\prime} \left( \gamma_0 y_j^n (e_j^n)^\top w^\star + y_L^n y_j^n (e_j^n)^\top w^\star (w^\star)^\top e_L^n \right).
\end{align*}
We similarly get $\left| M_n^\prime - \mathbb{E}[M_n^\prime] \right| \le t^\prime(\delta)$ with $t^\prime(\delta)=\Theta(1/\text{Poly}(d))$ and high probability at least $1-\delta$. 
Since $e_j^n$ and $e_L^n$ are independent zero-mean Gaussian noises $\mathcal{N}(0, I/d)$, the expectations of the second and third cross-terms are zero. Since the expected proportion of active tokens $\mathbb{E}[|\mathcal{A}_n^\prime|/L] \approx 1/2$, we have $\mathbb{E}[M_n^\prime] \approx \frac{1}{2} \mu^\prime_{t_1+t_2} \gamma_0^2 \geq 0$.
Therefore,
\begin{align*}
    M_n^\prime \ge \mathbb{E}[M_n^\prime] - t^\prime(\delta) \gtrsim - \frac{1}{\text{Poly}(d)}.
\end{align*}
For $\frac{1}{2}y_L^n N_{W_{t_1+t_2}}(E^\prime_{t_1+t_2}; X_1^n,Y^n)$, we have
\begin{align*}
    \left| y_L^n N_{W_{t_1+t_2}}(E^\prime_{t_1+t_2}; X_1^n,Y^n) \right| 
    &\leq \frac{1}{L} \sum_{j=1}^L \left| ([X_1^n]_j)^\top E^\prime_{t_1+t_2} x_{L,1}^n \right| \lesssim (1+\gamma_0)^2\|E^\prime_{t_1+t_2}\|_F,
\end{align*}
which is of $\Theta(1/\text{Poly}(d))$ order.
Thus, Term 3 satisfies:
\begin{align}\label{eq:KW_t1t2_term3}
    \frac{1}{2}y_L^n N_{W_{t_1+t_2}}(\Delta \overline{W}_{t_2}; X_1^n,Y^n) 
    = \frac{1}{2} M_n^\prime + \frac{1}{2}y_L^n N_{W_{t_1+t_2}}(E^\prime_{t_1+t_2}; X_1^n,Y^n)
    \gtrsim - \frac{1}{\text{Poly}(d)}.
\end{align}

Substituting Equation~\ref{eq:KW_t1t2_term2}$\sim$\ref{eq:KW_t1t2_term3}
into Equation~\ref{eq:yN_t1t2}, we have
\begin{align*}
    \frac{1}{2}y_L^n N_{W_{t_1+t_2}}(W_{t_1+t_2}; X_1^n,Y^n) \gtrsim \frac{\alpha}{2} \log(1/\mathcal{E})-\frac{\alpha}{2}\mathcal{E}^\prime - \frac{1}{\text{Poly}(d)} - \frac{1}{2}\epsilon_{W,1},
\end{align*}
where $\alpha=(1-\eta_2\lambda)^{t_2}$, $\mathcal{E}=\Theta(\frac{(\log d)^2}{\sqrt{d}})$ (from $K^1_{t_1}(\overline{W}_{t_1}) \lesssim \mathcal{E}$ in Theorem~\ref{the:stage1-P}), $\mathcal{E}^\prime=\Theta((\log d)^{1/2}d^{9/2}L^{-2/3}) = \Theta(1/\text{Poly}(d))$ with $L=\Omega(d^{6.75})$, $\epsilon_{W,1}=\Theta(1/\text{Poly}(d))$.

Finally, the surrogate loss for $\overline{W}_{t_1+t_2}$ satisfies
\begin{align*}
    K^1_{t_1+t_2}(\overline{W}_{t_1+t_2})
    \lesssim& \log\left(1 + \exp\left(-\frac{\alpha}{2} \log(1/\mathcal{E}) + \frac{\alpha}{2}\mathcal{E}^\prime + \frac{1}{\text{Poly}(d)} + \frac{1}{2}\epsilon_{W,1}\right)\right)\\
    \lesssim& \exp\left(-\frac{\alpha}{2} \log(1/\mathcal{E})\right)
    =\mathcal{E}^{\alpha/2},
\end{align*}
where $K^1_{t_1}(\overline{W}_{t_1}) \lesssim \mathcal{E} = \Theta(\frac{(\log d)^2}{\sqrt{d}})$ from Theorem~\ref{the:stage1-P}, $\alpha=(1-\eta_2\lambda)^{t_2} \in (0,1)$. We conclude that the loss on linear separable component $\mathcal{P}$ remains very low. Thus, the model successfully avoids catastrophic forgetting during the specialized stage.
\end{proof}

\subsection{Proof for Spectral Characteristics: Proof of Corollary \ref{coro:spectral-characteristics}}\label{sec:appendix-spectral}

\begin{proof}

We compare the spectral trace between $W$ and $V$ directly from our main theorems.

\textbf{In the Elementary Stage:}
At iteration $t_1$, according to Theorem~\ref{the:stage1-Q} (a.1), we have
\begin{align*}
    \text{Tr}(\overline{V}_{t_1}) \leq \sqrt{d} \|\overline{V}_{t_1}\|_F \lesssim \frac{1}{\text{Poly}(d)}.
\end{align*}
According to Theorem~\ref{the:stage1-P} (b.1), the network $h$ successfully learns the elementary knowledge. Its signal weight satisfies $\overline{W}_{t_1} = \mu_{t_1} w^\star (w^\star)^\top + E_{t_1}$ with $\mu_{t_1} \gg \|E_{t_1}\|_F$ from Appendix~\ref{proof:stage1-P}. Given $\|w^\star\|_2 = 1$, we have $\mu_{t_1} \approx \|\overline{W}_{t_1}\|_F$.
By the trace property, $\text{Tr}(\overline{W}_{t_1}) \approx \mu_{t_1} \text{Tr}(w^\star (w^\star)^\top) = \mu_{t_1} \|w^\star\|_2^2$. Thus
\begin{align*}
    \text{Tr}(\overline{W}_{t_1}) = \Theta\left(\|\overline{W}_{t_1}\|_F\right) \gtrsim d \log d.
\end{align*}

For weight $W_{t_1} = \overline{W}_{t_1} + \widetilde{W}_{t_1}$ and $V_{t_1} = \overline{V}_{t_1} + \widetilde{V}_{t_1}$, the linearity of the trace operator gives $\text{Tr}(W_{t_1}) = \text{Tr}(\overline{W}_{t_1}) + \text{Tr}(\widetilde{W}_{t_1})$. Since the noise weight $[\widetilde{W}_{t_1}]_{ij} \sim \mathcal{N}(0, \tau_0^2)$ with $\tau_0 = \Theta(1/\sqrt{d})$, the trace is a sum of $d$ independent zero-mean Gaussian variables, i.e., $\text{Tr}(\widetilde{W}_{t_1}) \sim \mathcal{N}(0, d\tau_0^2)$. By the standard Gaussian tail bound, with probability of $1-\delta$ (set $\delta=1/d$), we have $|\text{Tr}(\widetilde{W}_{t_1})| \lesssim \sqrt{\log d}, |\text{Tr}(\widetilde{V}_{t_1})| \lesssim \sqrt{\log d}$.

Finally, in the elementary stage, $\text{Tr}({V}_{t_1}) \lesssim \frac{1}{\text{Poly}(d)} + \sqrt{\log d}$ and $\text{Tr}({W}_{t_1}) \gtrsim d\log d - \sqrt{\log d}$, then
\begin{align*}
    \text{Tr}(W_{t_1}) > \text{Tr}(V_{t_1}).
\end{align*}

\textbf{In the Specialized Stage:}
At iteration $t_1+t_2$, according to Theorem~\ref{the:stage2-P} (c.1), the signal weight of network $h$ experiences a structured evolution governed by a decay factor $\alpha \in (0, 1)$ and a bounded reinforcement multiplier $\mu \lesssim d\log^2 d$, i.e.,
\begin{align*}
    \overline{W}_{t_1 + t_2} = \alpha \overline{W}_{t_1} + \mu^\prime w^\star (w^\star)^\top + E^\prime.
\end{align*}
Given $\|w^\star\|_2 = 1$ and the negligible noise trace $\text{Tr}(E^\prime) = \mathcal{O}(1/\text{Poly}(d))$, the spectral trace cleanly decouples:
\begin{align*}
    \text{Tr}(\overline{W}_{t_1+t_2}) \approx \alpha \text{Tr}(\overline{W}_{t_1}) + \mu^\prime  \lesssim \alpha d\log d + d\log^2 d.
\end{align*}
Together with noise weight, $\text{Tr}(\widetilde{W}_{t_1+t_2}) \lesssim \sqrt{\log d}$, we have
\begin{align*}
    \text{Tr}(W_{t_1+t_2}) \lesssim d \log^2 d + \sqrt{\log d}.
\end{align*}

For network $g$, according to Theorem~\ref{the:stage2-Q} (c.1), the trajectory successfully enters the target convergence basin to learn the specialized knowledge. As derived in Appendix~\ref{proof:stage2-Q}, we have $\Delta \overline{V}_{t_2} \approx \boldsymbol{\mu}^V_{t_1+t_2} z^\top + \boldsymbol{\nu}^V_{t_1+t_2} \zeta^\top$. Then,
\begin{align*}
    \text{Tr}(\Delta \overline{V}_{t_2}) \approx \text{Tr}(\boldsymbol{\mu}^V_{t_1+t_2} z^\top) + \text{Tr}(\boldsymbol{\nu}^V_{t_1+t_2} \zeta^\top) = z^\top \boldsymbol{\mu}^V_{t_1+t_2} + \zeta^\top \boldsymbol{\nu}^V_{t_1+t_2} = \Theta(u^2 \|\boldsymbol{\mu}^V_{t_1+t_2}\|_2 + r^2 \|\boldsymbol{\mu}^V_{t_1+t_2}\|_2) = \Theta(d^{9/8}\log d).
\end{align*}
Together with $\text{Tr}(\overline{V}_{t_1}) \lesssim \frac{1}{\text{Poly}(d)}$, $\text{Tr}(\widetilde{V}_{t_1+t_2}) \lesssim \sqrt{\log d}$
Thus,
\begin{align*}
    \text{Tr}(V_{t_1+t_2})
    =\text{Tr}(\overline{V}_{t_1}) + \text{Tr}(\Delta \overline{V}_{t_2}) +\text{Tr}(\widetilde{V}_{t_1+t_2}) \lesssim d^{9/8}\log d+\frac{1}{\text{Poly}(d)} + \sqrt{\log d}.
\end{align*}
We conclude that at the end of specialized stage,
\begin{align*}
    \text{Tr}(W_{t_1+t_2}) < \text{Tr}(V_{t_1+t_2}).
\end{align*}
\end{proof}

\stopcontents[section]
\bibliographystyle{IEEEtran}
\bibliography{IEEEabrv,reference}

\end{document}